\documentclass[preprint,review,12pt]{elsarticle}

\pdfoutput=1

\usepackage{amssymb}

\usepackage{epsfig}
\usepackage[ruled]{algorithm2e}
\usepackage{algorithmic}
\usepackage{subfigure}
\usepackage{xspace}
\usepackage{rotating}
\usepackage{fancyhdr}
\usepackage{url}
\usepackage{color}
\usepackage{booktabs}
\usepackage{amsmath}  
\usepackage{amssymb}
\DeclareFontFamily{OT1}{pzc}{}
\DeclareFontShape{OT1}{pzc}{m}{it}%
              {<-> s * [0.900] pzcmi7t}{}
\DeclareMathAlphabet{\mathpzc}{OT1}{pzc}%
                                 {m}{it}
                                 
\urldef{\mailsa}\path|{lwei, eribeiro}@fit.edu|
\newcommand{\argmin}{\operatornamewithlimits{arg\,\,min}}

\newcommand{\etal}{{et al.}}
\newcommand{\ie}{{i.e., }}          
\newcommand{\eg}{{e.g., }}

\def\Transpose{^\mathsf{T}}

\newcommand{\VChamferE}{E^{\text{d}}(\bfu)}
\newcommand{\ApproxJ}{\mathbf{\widetilde{J}}}
\newcommand{\supp}{\text{supp}}
\newcommand{\tvec}{\text{vec}}
\newcommand{\BallSet}{\mathbb{P}}

\newcommand{\ProximityFunction}{N_{\delta}}

\newcommand{\ShapeS}{\mathpzc{S}}
\newcommand{\ShapeD}{\mathpzc{D}}
\newcommand{\SimilarityMeasure}{{F}}
\newcommand{\RegionS}{{R_{S}}}
\newcommand{\ImageDomain}{\Omega}
\newcommand{\DistanceTransform}{\Pi}

\newcommand{\BasisVector}{{\boldsymbol{\phi}}}

\newcommand{\LocalDeform}{\mathbf{u}}
\newcommand{\PosVector}{\mathbf{x}}

\newcommand{\Ec}{E^c}

\newcommand{\Ecpq}{\Ec_{p,q}}
\newcommand{\bfp}{\mathbf{p}}
\newcommand{\bfq}{\mathbf{q}}
\newcommand{\bfd}{\mathbf{d}}
\newcommand{\bfD}{\mathbf{D}}

\newcommand{\wqp}{w_q(\|\mathbf{p-q}\|)}
\newcommand{\bfx}{\mathbf{x}}
\newcommand{\bft}{\mathbf{x}}
\newcommand{\bfu}{\mathbf{u}}

\newcommand{\morder}{m} 
\usepackage{soul}

\newtheorem{theorem}{Theorem}[section]
\newtheorem{lemma}[theorem]{Lemma}
\newenvironment{proof}[1][Proof]{\begin{trivlist}
\item[\hskip \labelsep {\bfseries #1}]}{\end{trivlist}}

\journal{Computer Vision and Image Understanding}

\begin{document}

\include{notationsymbols}

\begin{frontmatter}



\title{A Meshless Method for Variational Nonrigid 2-D Shape Registration}

\author{Wei Liu and Eraldo Ribeiro}
\address{Computer Vision and Bio-Inspired Computing Laboratory\\
	Department of Computer Sciences \\
	Florida Institute of Technology \\
	Melbourne, FL 32901, U.S.A.\\
	{\tt eribeiro@cs.fit.edu}
}

\begin{abstract}
We present a method for nonrigid registration of 2-D geometric shapes. Our contribution is twofold. First, we extend the classic chamfer-matching energy to a variational functional. Secondly, we introduce a meshless deformation model that can handle significant high-curvature deformations. We represent 2-D shapes implicitly using distance transforms, and registration error is defined based on the shape contours' mutual distances. In addition, we model global shape deformation as an approximation blended from local deformation fields using partition-of-unity. The global deformation field is regularized by penalizing inconsistencies between local fields. The representation can be made adaptive to shape's contour, leading to registration that is both flexible and efficient. Finally, registration is achieved by minimizing a variational chamfer-energy functional combined with the consistency regularizer. We demonstrate the effectiveness of our method on a number of experiments.
\end{abstract}

\begin{keyword}
meshless models \sep shape registration \sep shape correspondence \sep variational chamfer matching \sep distance transform



\end{keyword}

\end{frontmatter}


\section{Introduction}

Registering 2-D shapes that have been deformed by nonlinear mappings is a challenging problem that has applications in many areas including medical imaging~\cite{chen_hui_2006} and shape recognition~\cite{borgefors1988hierarchical,thayananthan2003shape}. Existing shape registration methods differ in three main components~\cite{paragios2003non,huang2006shape}: shape representation, deformation model, and registration criterion. A number of recent works represent shape contours implicitly as zero level-sets of distance transforms~\cite{paragios2003non,huang2006shape,rousson2008prior,el2007shape}. This shape representation has two important advantages. First, distance transform is a generic representation that can handle arbitrary shapes in arbitrary dimensions~\cite{huang2006shape}. Secondly, distance transform provides a 2-D embedding of 1-D shape contours, so contours can be registered by aligning their 2-D distance transforms as done in image registration~\cite{paragios2003non,huang2006shape,rousson2008prior}. Here, the shape-registration criterion can be simply the sum-of-squared-differences (SSD) of the shapes' distance transforms~\cite{paragios2003non} or mutual information (MI)~\cite{huang2006shape}. Furthermore, nonlinear shape deformations can be represented using parametric models such as free-form deformation (FFD)~\cite{huang2006shape} and radial basis functions (RBF)~\cite{chen_hui_2006}.

Despite these advances, using distance transforms for shape representation in registration methods present some common drawbacks. One issue is that distance transforms are redundant representations of shape contours which brings extra computational cost when performing registration. There is also the issue of stability of the representation when shapes undergo high-curvature deformation. Existing methods~\cite{paragios2003non,huang2006shape,rousson2008prior} use a narrow-band function to confine computation to be around the shape contours, but it only partially addresses the problem and complicates the registration framework. An additional issue is the use of regular mesh of control points to represent deformations. FFD-based deformation models~\cite{huang2006shape,rousson2008prior} rely on a mesh of regularly distributed control points that makes it difficult to adapt to the shape contours, and the registration accuracy may suffer from folding effect of the control-point mesh under large deformation. The limitation of control-point meshes is shared by many areas including computer graphics~\cite{ohtake2003multi}, engineering~\cite{Melenk1996289,liu_gr_meshfree2009}, and image registration~\cite{makram510non,rohde2003adaptive}. The lack of flexibility of regular mesh models can be addressed by replacing the mesh-based deformation models with so-called meshless models that do not rely on explicit connections between control points. Examples of meshless models include thin-plate splines using radial basis functions (RBFs) that has long been used for image and shape registration~\cite{rohde2003adaptive,chen_hui_2006,belongie_pami_2002,Chui2003114}. However, RBFs are less accurate than mesh models~\cite{liu_gr_meshfree2009}, and can be computationally expensive~\cite{rohde2003adaptive}. A number of recent works resorted to more accurate meshless models such as partition-of-unity (PU) method~\cite{liu_gr_meshfree2009,makram510non,Melenk1996289, ohtake2003multi}. The PU-meshless model represents a continuous function by blending together local polynomial models, and the local models can be arbitrarily distributed and overlaid without relying on connected control points~\cite{Melenk1996289}. However, existing PU methods rely on complicated functionals~\cite{makram510non,Nagai_sgp2009} for regularizing shape deformation to avoid degenerating solutions and to produce smooth results.

In this paper, we follow the work in~\cite{paragios2003non,huang2006shape,rousson2008prior}, and represent shape contours using distance transforms. In Section~\ref{sec:distance_functions}, we briefly introduce two representative registration methods~\cite{huang2006shape,paragios2003non} using distance-transform representations, and elaborate on their limitations.
Then, we address these limitations, and improve the state-of-the-art in two aspects. First, we improve the registration criterion, by proposing a variational extension to the classic chamfer-matching method~\cite{borgefors1988hierarchical} (Section~\ref{sec:vchamfer}). The proposed functional does not rely on a narrow-band function, and can better handle high-curvature shape deformations, leading to improvements in the registration accuracy. Chamfer-matching has been commonly used for detecting objects under affine transformations~\cite{borgefors1988hierarchical,thayananthan2003shape,liu2010fast}. Here, we adopt it for variational nonrigid contour registration. In addition, we derive a simplified gradient function of the variational chamfer-matching functional to facilitate the registration process. 

Secondly, we adopt the PU-meshless model to represent nonlinear shape deformations (Section~\ref{sec:meshless}). The meshless model can better represent large shape deformations, and can be easily adapted along shape contours to reduce computational cost. Figure~\ref{fig:overall} shows the overall scheme of our method. To regularize the PU model, we propose a novel regularizer that penalizes inconsistencies between neighboring local models. The inconsistency regularizer is much simpler than existing methods~\cite{makram510non,Nagai_sgp2009}, and provides a theoretical upper-bound to some regularizers~\cite{makram510non}. It is worth pointing out that a similar consistency regularizer has been used for smoothing B-splines~\cite{eilers1996flexible} that penalizes differences of adjacent B-splines' coefficients. However, this regularizer cannot be directly applied to PU-meshless model as it ignores the differences of local models' coefficients caused by shifting coordinate systems. Our method compensates for this  shifting effect using a linear operator, so that neighboring polynomial models can be compared on a common ground. 

In Section~\ref{sec:gradient}, we introduce our registration algorithm that combines the variational chamfer-matching gradients, the PU-meshless deformation model, and the consistency regularizer into a functional-minimization framework. We verify our method by comparing it to two recent registration methods~\cite{huang2006shape,chen_hui_2006}. The first method by Huang~\etal~\cite{huang2006shape} uses distance-transform representation and FFD-based deformation model, while the second method by Chen~\etal~\cite{chen_hui_2006} represents shapes using shape-context~\cite{belongie_pami_2002} and RBF-based deformation model. Experiments show that our method outperforms both methods in terms of registration accuracy. Additionally, our method can handle shapes with large high-curvature deformations  (Section~\ref{sec:experiments}).

\begin{figure}[t]
	\begin{center}
		{\includegraphics[width=.99\linewidth]
		  		         {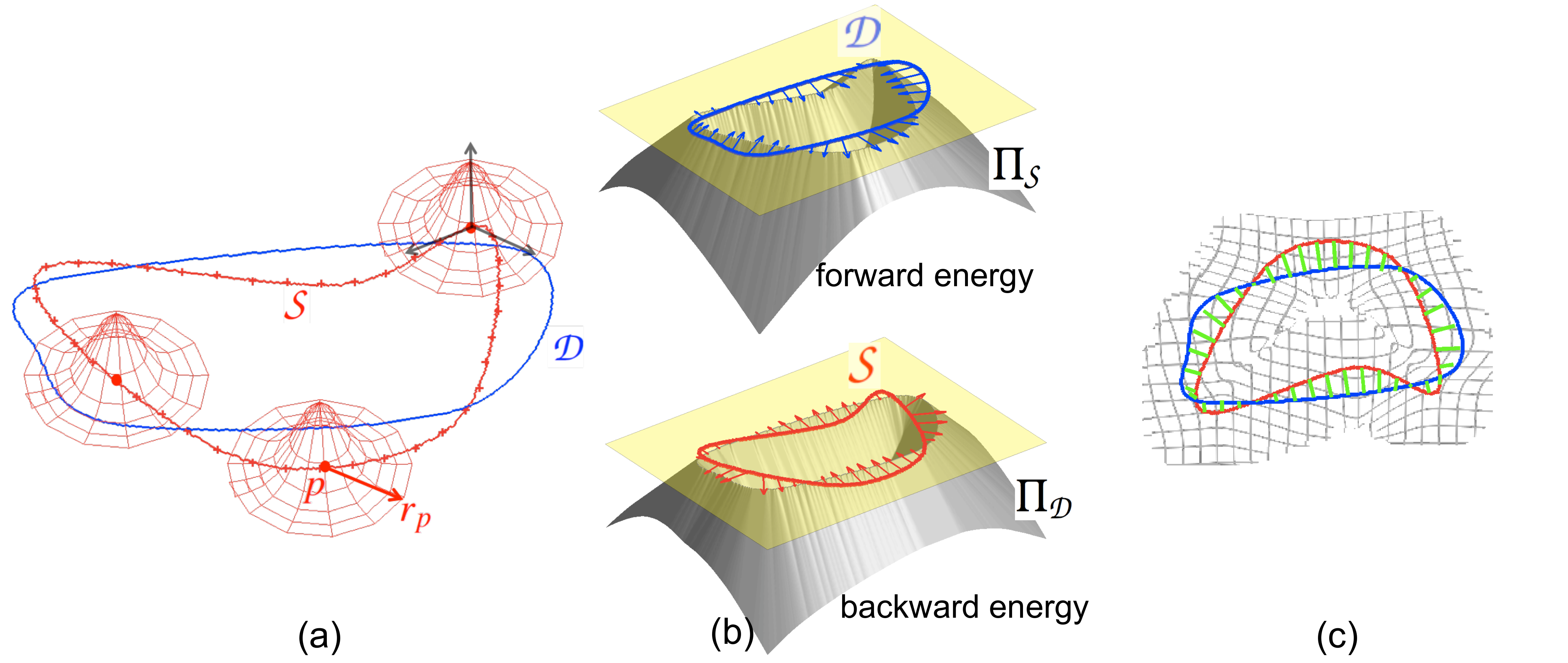}}					
	\label{fig:overall}
	\caption{Meshless shape registration. 
			 (a) Source (red curve) and target (blue curve) 
			 shapes. Local polynomial deformation models 
			 are defined on overlapping patches along 
			 the contour. Patches are weighted by locally 
			 supported radial functions. Three patches are 
			 shown with their corresponding weighting functions.  
			 (b) Forward and backward registration gradients.  
			 (c) Blended global deformation map and 
			 correspondence after registration.}
	\end{center}
\end{figure}

In summary, our main contribution is twofold. First, we improve the registration criterion used by previous shape-registration methods~\cite{paragios2003non,huang2006shape,rousson2008prior} that use  distance-transform representation, and formulate the registration problem as the one of minimizing a variational chamfer-matching functional. Secondly, we introduce a meshless model for representing nonlinear shape deformations, and design a simple consistency regularizer to both produce smooth shape deformation and avoid degenerated results. Finally, we would like to comment that this work extends our previous works in~\cite{WeiRibeiroShapeISVC2010,WeiRibeiroMeshlessISVC2010}.

\subsection{Our assumptions}
We make an important assumption that shape contours are already ``roughly aligned'' before applying our registration algorithm. This is assumption is common to many existing shape and image registration methods that employ a global-to-local scheme where the shapes are registered in two stages~\cite{chen_hui_2006,huang2006shape,paragios2003non}. In these methods, shapes are first roughly aligned using affine transformations to compensate for global shape deformations such as rotation, shifting, and scaling. Then, a second step seeks for local nonlinear deformations that align shapes as close as possible. Global shape alignment is not our main goal, since it can be achieved using off-the-shelf methods such as shape context~\cite{chen_hui_2006,belongie_pami_2002}, mutual information~\cite{huang2006shape}, and chamfer matching~\cite{borgefors1988hierarchical}. As a result, we assume that shapes are aligned beforehand using a rigid transformation, and focus ourselves on the nonrigid registration part.

\section{Related work}


\subsection{Shape matching and registration based on distance transform}
The distance transform (distance map) is an implicit shape representation whose value at a point in map indicates the minimum distance of that point to the shape. This implicit representation has a number of advantages. First, the calculation of distance transforms is highly efficient~\cite{felzenszwalb2004distance}. Secondly, the distance information helps efficient shape alignment by allowing the searching step to be adjusted according to the shapes' mutual distance~\cite{borgefors1988hierarchical}. Indeed, chamfer-matching are frequently used for shape contour matching and detection~\cite{gavrila_eccv2000,thayananthan2003shape,shotton2008,liu2010fast}. Earlier chamfer-matching methods~\cite{gavrila_eccv2000,thayananthan2003shape} focus on detecting shapes under affine transformations with little intra-class variations, while recent works include the use of orientation information~\cite{opelt2006boundary,liu2010fast}, hierarchical template matching~\cite{gavrila_eccv2000}, and statistical learning~\cite{zhang_2004,shotton2008,ferrari_2010} that significantly improve the handling of shape variations and occlusion. The majority of these works apply distance transforms to shape contours, with exception of a recent work by Bai~\etal~\cite{bai2009active} that uses distance transforms for matching articulated shape skeletons undergoing deformations. 

Distance transforms are also used for registering contours~\cite{put2009contour,paragios2003non,huang2006shape,rousson2008prior,huseiny_2010,tsechpenakis2008novel}, and point sets~\cite{parra2009rigid}. Similarity can be drawn between the problems of shape registration and shape detection, since both problems involve shape matching, but important differences exist. The goal of shape detection is to locate a given shape in an image, while registration assumes the shape's existence and aims at accurately recovering its deformation. Paragios~\etal~\cite{paragios2003non} introduced the idea of using distance transform as an implicit embedding of shape contours, and registered distance maps  essentially as normal 2-D images. Huang~\etal~\cite{huang2006shape} extended the method using FFD-based parametric deformation model. Munim~\etal~\cite{el2007shape} extended the representation in~\cite{huang2006shape} to a vectorized one for registering open shape contours. Rousson~\etal~\cite{rousson2008prior} and Taron~\etal~\cite{taron2009registration} investigated the integration of statistical priors into the registration framework. In this paper, we focus on improving the registration criterion and parametric deformation model based on the work in~\cite{paragios2003non,huang2006shape}.

Finally, distance transform is a level-set representation that can naturally represent shapes of arbitrary topology~\cite{huang2006shape}. One may find similarities between our work and level-set methods~\cite{chenyang_xu,Li_TIP08} in three aspects. First, both methods use level sets of certain potential functions to represent a deforming contour. However, the potential functions in level-set methods are not necessarily distance maps, unless a regularizer is applied to the potential function to encourage it to maintain the shape of a distance transform~\cite{Li_TIP08}. In addition, many level-set methods aim at image segmentation (\ie to locate a shape contour), and they deform a potential function according to an external force that favors image features (\eg edges). In contrast, the shape registration methods in~\cite{paragios2003non,huang2006shape,rousson2008prior} and our method assume that two shape contours (\ie the source and the target shapes) are given beforehand, and try to find \emph{how} the source contour can be deformed into the target contour.  Even though level-set methods are also used for object tracking~\cite{paragios2003level,paragios2000geodesic}, their goal is to locate a moving object and seldom provide a dense correspondence between shape contours. Indeed, level-set methods can provide input data to our method by extracting the contours.

A second similarity between ours and level-set methods~\cite{chenyang_xu,Li_TIP08,paragios2003level,paragios2000geodesic} is the use of penalty functions for regularization. In level-set methods~\cite{chenyang_xu,Li_TIP08}, the regularizers penalize the geometric \emph{properties} (\eg curvature) of shape contours or their potential functions, to produce smooth shape boundaries. In contrast, registration methods~\cite{paragios2003non,huang2006shape} penalize the \emph{behavior} of shape contours to produce smooth shape deformation, and usually do not care about the shapes' curvature. Finally, both level-set methods and our work minimize variational functionals by iteratively solving partial differential equations (PDEs). Please refer to~\cite{paragios2003non,huang2006shape} for more details of the relationship between level-set methods and shape registration based on distance transforms.

\subsection{Modeling and regularizing nonlinear deformations}

Distance transforms provide an implicit representation for shape contours themselves, but not of their deformations. It has been shown in~\cite{huang2006shape} that FFD models can significantly improve registration accuracy over the nonparametric deformation model~\cite{paragios2003non}. The FFD model belongs to a group of so-called mesh models~\cite{liu_gr_meshfree2009} that rely on a mesh of explicitly connected control points, that makes it difficult to adapt mesh models to different shapes~\cite{hansen2008adaptive}, especially for shapes undergoing topology changes, where remeshing is often required and can be prone to errors~\cite{ohtake_2003}. Another drawback of mesh models is that control points may collapse together (\ie the folding effect) and change spatial relationship under large deformation, leading to numerical inaccuracies~\cite{makram510non,liu_gr_meshfree2009}. For mesh models, the folding effect can be addressed by using extra regularizers to produce diffeomorphic deformations~\cite{rueckert_diffeomorphic_2006}, but these regularizers increase the complexity of the registration framework. 

To address these two issues, researchers in different areas seek for parametric models, called meshless models~\cite{liu_gr_meshfree2009} that do not rely on explicit connections between control points. Meshless models can be easily adapted to shape contours and greatly alleviate the folding effect of large deformations. Please refer to~\cite{LiuRibeiroSurvey2011} for an extensive review of meshless models. Here, we focus on the works that are more relevant to nonrigid registration. An example of meshless deformation models are the {\it Thin-Plate Splines} (TPS). These models represent a deformation field by a linear combination of radial basis functions~\cite{rohde2003adaptive} centered at scattered points, without explicit connections between them. However, it has been shown that RBFs can be numerically less accurate than the partition-of-unity method~\cite{liu_gr_meshfree2009}, due to the fact that RBFs cannot exactly represent a polynomial function (lack of reproducibility). In addition, it has been acknowledged in image registration~\cite{rohde2003adaptive} and computer graphics~\cite{ohtake_2003}  RBFs have higher computational cost than the PU method.

In this paper, we adopt PU to represent nonrigid shape deformations. Additionally, we propose a simple regularizer to produce smooth deformation fields and to avoid degenerated results. Existing registration methods regularize the deformation model using a  conformity functional~\cite{makram510non} or by imposing mechanical energy constraints~\cite{liu2003meshfree}. There are other meshless methods proposed for image registration~\cite{chen2008fast,chen2010object,sykora2009rigid,wang2008meshless} and segmentation~\cite{ho2005point}. Many of these rely on predefined landmarks and use them as boundary conditions to solve mechanical PDEs. Finally, there are registration methods that represent shapes using PU models~\cite{lee20083d,walder2009markerless}. These should not be confused with our method that use PU models to represent the shapes' deformations.

\section{Distance functions and nonrigid registration}
\label{sec:distance_functions}

The goal of shape registration is to deform a source shape onto a target shape. This is achieved by searching for the best deformation field that minimizes a dissimilarity measure between the shapes. Formally, if $\ShapeS$ and $\ShapeD$ represent source and target shapes, respectively, and $\SimilarityMeasure$ is a dissimilarity measure between the two shapes, we seek for a warping field $\mathbf{u}(\bfx)$ that satisfies the following equation:
\begin{align}
	\text{arg} \min_{\bfx'} F(\ShapeD(\bfx),\ShapeS(\bfx'),\bfx'),\,\,\,\,\, \bfx'=\bfx+\mathbf{u}(\bfx),
	\label{eq:registration}
\end{align}
where $\bfx$ is a coordinate vector. The dissimilarity measure $F$ usually depends on the shape model.  In this paper, we implicitly represent a shape ${\ShapeS}$ as the zero level set of its distance transform $\Pi_{\ShapeS}$~\cite{paragios2003non,huang2006shape}, where  ${\ShapeS}$ defines a partition of the image domain $\ImageDomain$. The model is given by:
\begin{align}
	\DistanceTransform_{\ShapeS} =
		\begin{cases} 
		  \hfill                  0, 	&\bfx \in \ShapeS\\ 
		 +D_{\ShapeS}(\bfx)>0, 	&\bfx \in \RegionS\\
		 -D_{\ShapeS}(\bfx)<0, 	&\bfx \in [\Omega-\RegionS]
	\end{cases},
	\label{eq:signed_distance}
\end{align}
where $D_{\ShapeS}$ is the minimum Euclidean distance between location $\bf x$ and shape $\ShapeS$, and $\RegionS$ is the region inside $\ShapeS$. Distance transforms are essentially solutions of a Eikonal equation given by~\cite{felzenszwalb2004distance}:
\begin{align}
\|\nabla \DistanceTransform_{\ShapeS}(\bfx) \| =1, \hspace{.2in} \bfx \in \mathbb{R}^2, \hspace{.3in} \text{with} \hspace{.1in}
\DistanceTransform_{\ShapeS}(\bfx)=0, \hspace{.2in} \bfx \in \ShapeS,
\label{eq:distance_transform_constraint}
\end{align}
where shape contour $\ShapeS$ defines the boundary condition. Distance transforms are 2-D functions themselves, and rigorously $\DistanceTransform_{\ShapeS}(\bfx)$ should be written as $\Pi_{\ShapeS(\bft)}(\bfx)$, but to simplify the notation, we write as $\Pi_{\ShapeS}(\bfx)$, $\Pi_{\ShapeS(\bft)}$ and even $\Pi_{\ShapeS}$.  This simplification should not cause confusion given the context of its usage.
The Eikonal equation in (\ref{eq:distance_transform_constraint}) is well-posed, meaning its solutions (distance transforms) depend continuously on the initial conditions (shape contours). As a result, similarity of distance transforms indicates the similarity of shape contours, and one can solve the shape registration problem indirectly by aligning shapes' distance transformations. Formally, the registration problem defined in (\ref{eq:registration})  is converted to the following one given by~\cite{paragios2003non}:
\begin{align}
	\text{arg} \min_{\bfx'} F(\DistanceTransform_{\ShapeD}(\bfx),\DistanceTransform_{\ShapeS}(\bfx'),\bfx'),\,\,\,\,\, \bfx'=\bfx+\mathbf{u}(\bfx).
	\label{eq:registration_distance}
\end{align}
In this approach, registering two 1-D contours is converted to the problem of registering their 2-D embeddings. The later problem can be easily solved using established 2-D image registration methods. In the following text, we will call this approach the \emph{shape-embedding method}.
Indeed, exactly as registering 2-D images, $F$ can be simply defined as the squared-sum of differences, and registration is achieved by minimizing following functional~\cite{huang2006shape,paragios2003non}:
\begin{align}
	E(\bfu) = \underbrace{\int_{\Omega} 
	          \ProximityFunction(\DistanceTransform_{\ShapeD} - 
	          \DistanceTransform_{\ShapeS})^2 d\bfx}_{\text{data term}} + 
	          \alpha \underbrace{\int_{\Omega} \ProximityFunction\left( \|\nabla u_x\|^2 + 
	          \|\nabla u_y\|^2 \right) d\bfx}_{\text{smoothness regularizer}}.
	\label{eq:functional_paragios}
\end{align}
Here, $\DistanceTransform_{\ShapeS}$ and $\DistanceTransform_{\ShapeD}$ are distance transforms of the source and target shapes, respectively. $u_x$ and $u_y$ are the components of the deformation field, \ie $\bfu(\bfx)=\left(u_x(\bfx), u_y(\bfx)\right)$. First-order derivatives of $u_x$ and $u_y$ of large magnitude are penalized by a regularizer to produce a smooth deformation field and to avoid degenerated solutions. Finally, a proximity function $\ProximityFunction$ limits the data-term evaluation to be near the shape's boundary. The use of a proximity function is a limitation of the shape-embedding methods, and the reason will be clarified in Section~\ref{subsec:limitations}. 

Similarly to nonrigid 2-D image registration, Equation~\ref{eq:functional_paragios} has been previously extended in several ways. First, the image deformation field $\bfu$ can be modeled parametrically using B-splines~\cite{huang2006shape}, or thin-plate splines~\cite{chen_hui_2006}. Secondly, statistical priors can be leveraged to address uncertainties in the registration process~\cite{taron2009registration}. Despite these developments, shape-embedding methods are still limited in some aspects. In the following section, we will elaborate on these limitations that motivated us for this work.

\subsection{Limitations of the shape-embedding approach}
\label{subsec:limitations}
The shape-embedding approach in Equation~\ref{eq:functional_paragios} facilitates the use of existing nonrigid 2-D registration techniques to solve 1-D shape registration problem. However, several issues need to be considered. First, as pointed out by El Munim and Farag~\cite{el2007shape}, the distance map defined  in Equation~\ref{eq:signed_distance} cannot model open shape contours or shapes   that do not define closed image domains (\ie ``inside'' and ``outside'' parts of a shape). This problem was addressed in~\cite{el2007shape} using vector distance transform. However, for the purpose of contour registration, we may simply use an unsigned distance transform. In this paper, we assume the shapes to be 0-1 encoded contours, and represented using unsigned distance transform as: 
\begin{align}
	\DistanceTransform_{\ShapeS} =
		\begin{cases} 
		  \hfill                  0, 	&\bfx \in \ShapeS\\ 
			D_{\ShapeS}(\bfx)>0, 	&\bfx \in \Omega -\ShapeS
	\end{cases},
	\label{eq:unsigned_distance}
\end{align}
where $D_{\ShapeS}(\bfx)$ is still the minimum Euclidean distance of position $\bfx$ to the shape contour $\ShapeS$. In the following text, we always assume $\Pi_\ShapeS \geq 0$. 

Secondly, minimizing Equation~\ref{eq:functional_paragios} essentially deforms the distance map $\DistanceTransform_\ShapeS(\bfx+\bfu)$ to align with $\DistanceTransform_\ShapeD(\bfx)$. However, deformation in the distance transform only approximates the shape deformation. Formally, $\DistanceTransform_\ShapeS(\bfx+\bfu) \neq \DistanceTransform_{\ShapeS(\bft+\bfu)}$, where $\DistanceTransform_\ShapeS(\bfx+\bfu)$ is a direct deformation of the original shape's distance map, while $\DistanceTransform_{\ShapeS(\bft+\bfu)}$ is the distance map calculated from the deformed shape. For example, it has been noticed by Paragios~\etal~\cite{paragios2003non} that scaling a shape is not equivalent to scaling its distance transform. 
Here, we would like to further point out that modest deformations in the shape contour (\ie $\ShapeS(\bft+\bfu)$) may sometimes cause significant deformations in its distance transform (\ie $\Pi_{\ShapeS(\bft+u)}$), since $\ShapeS(\bft+\bfu)$ is the boundary condition for the Eikonal PDE (Equation~\ref{eq:distance_transform_constraint}), its deformation will be propagated into the PDE's solution. As a result, registering  distance maps may be unnecessarily complicated. Figure~\ref{fig:bending_curve} shows the distance maps of a bending curve, \ie $\Pi_{\ShapeS(\bft)}$ and $\Pi_{\ShapeS(\bft+\bfu)}$, together with iso-curves indicating points of same distance values. The deformation of $\Pi_{\ShapeS(\bft+\bfu)}$ has much higher curvature in certain places than the deformation of $\ShapeS(\bft+\bfu)$ itself. As a result, when directly registering the distance maps, much higher bending energy is needed to cope with these high curvature regions, leading to registration errors or even degenerated results.
\begin{figure}[t]
\begin{center}
\subfigure[Source shape $\Pi_{\ShapeS(\bft)}$]
{
	\includegraphics[width=.42\linewidth]
	{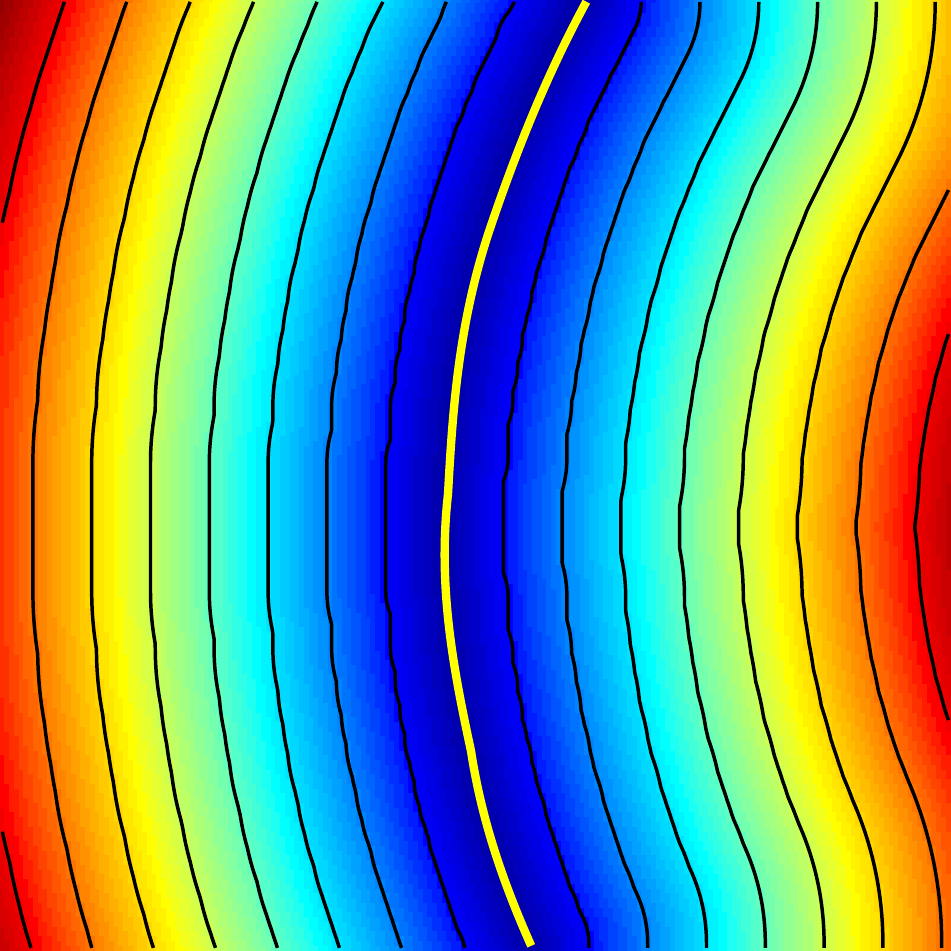}
}
\subfigure[Deformed shape $\Pi_{\ShapeS(\bft+\bfu)}$]
{
	\includegraphics[width=.42\linewidth]
	{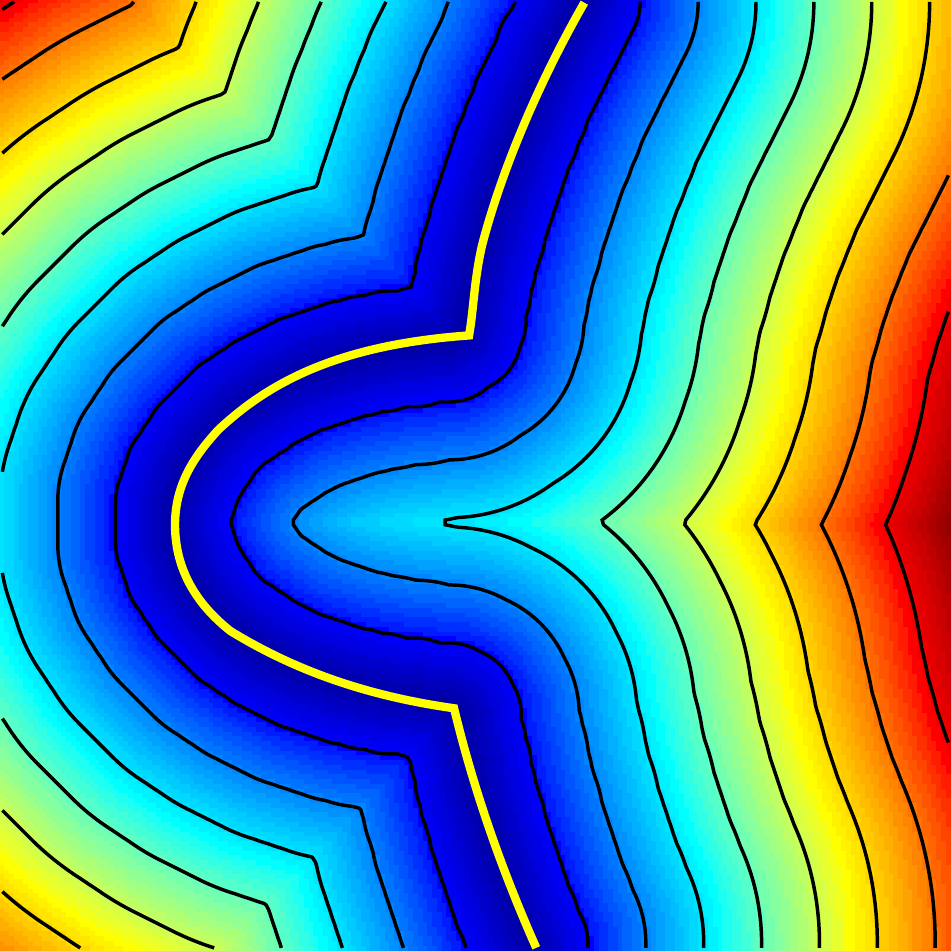}
	\label{subfig:narrow_band}
}

\end{center}
\caption{The (unsigned) distance maps of a bending curve (center in yellow). The distance maps are color-coded, with red and blue indicating maximum and minimum values, respectively. Black iso-curves are plotted to indicate points of the same distance values. (a) Distance map of the source shape. (b) A modest bending in the shape contour may cause high-curvature distortions in the distance map.}
\label{fig:bending_curve}
\end{figure}

This problem is only partially addressed in~\cite{paragios2003non,huang2006shape} by using a proximity function $\ProximityFunction$ (Equation~\ref{eq:functional_paragios}), that limits similarity measure in the proximity of shape contours with distance value less than $\delta$, where the propagated distortion is minimal. Intuitively, $\ProximityFunction$ defines a ``narrow band'' along the shape contour to exclude the high-curvature part. Nevertheless, care must be taken to choose the narrow-band width parameter $\delta$. On one hand, if $\delta$ is too small, there may be little overlapping areas between the two shapes, and the registration algorithm can easily get trapped in local minima. On the other hand, if $\delta$ is too large, the proximity function may fail to exclude distorted regions. Taron~\etal~\cite{taron2009registration} proposed to use an iteratively decreasing $\delta$, but this approach complicates the optimization algorithm and still relies on this extra parameter.



Finally, registering 2-D distance functions results into extra computation as the method needs to register the whole image plane instead of 1-D shape contours. Parametric representations~\cite{huang2006shape} such as B-Splines can be used to improve  registration efficiency, but B-Spline model cannot be easily adapted to the shape contours as it relies on a regular grid (mesh) of control points. Consequently, a lot of computation is still wasted in modeling and calculating deformation fields at  irrelevant regions. The use of a proximity function~\cite{paragios2003non,huang2006shape,taron2009registration} further reduces the computational cost, but the formulation becomes more complicated. More importantly, mesh-based models may suffer from the folding effect when the deformation is large. Figure~\ref{fig:folding} shows a result obtained using the method in~\cite{huang2006shape}, with the deformation field represented using a control-point mesh of B-splines~\footnote{Obviously, the shape deformation here can be modeled as affine global deformation, but we treat it as large local deformation for the sake of demonstration.}. Large deformations in parts of the image cause  some control points to be ``squeezed'' together and even flip their positions (folding). As a result, the deformation field no longer equals to the interpolation of control-point displacements, leading to less accurate or even failed registration. 
\begin{figure}[t]
\begin{center}
\subfigure[Correspondence]
{
	\includegraphics[width=.47\linewidth]
	{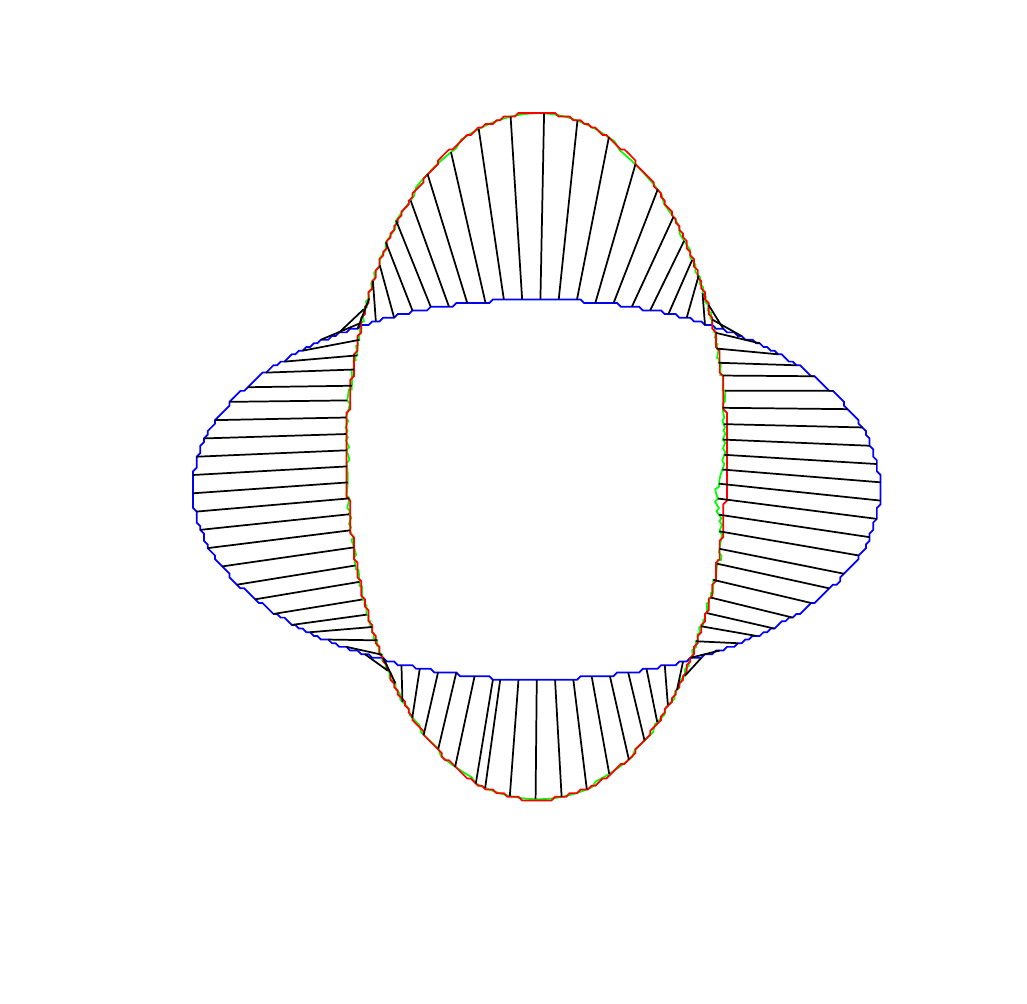}
}
\subfigure[Deformation field]
{
	\includegraphics[width=.47\linewidth]
	{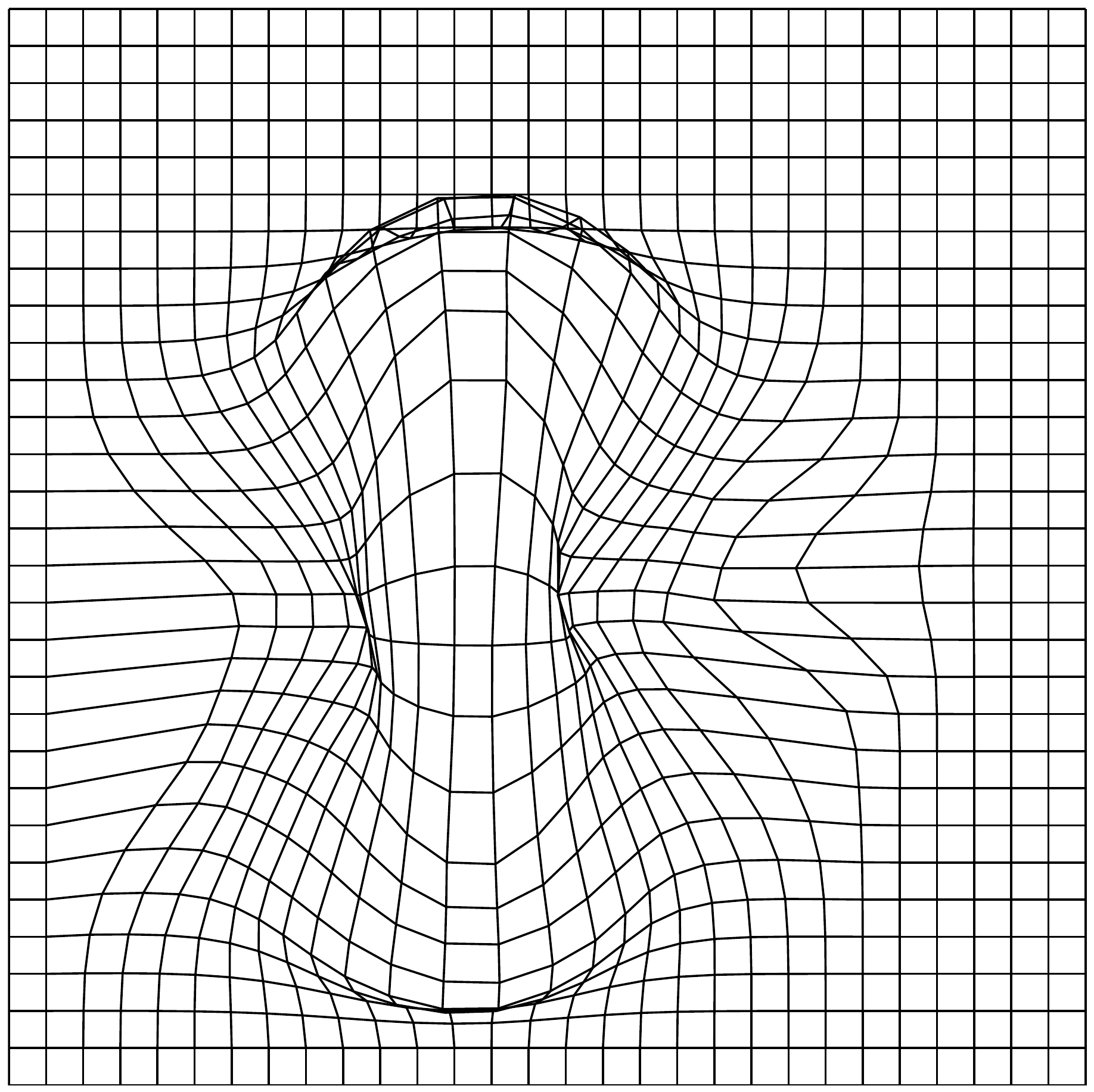}
}
\end{center}
\caption{Folding of mesh model under large deformation. The images are obtained using method in~\cite{huang2006shape}. a) Correspondence of the source shape (in blue) deformed to the target image (in red). The deformed source shape is in green. b) Under large deformation, the B-spline control points collapse together and may even flip their positions.}
\label{fig:folding}
\end{figure}

Existing registration works address the folding issue in two ways.  First, foldings in mesh models can be removed with a stronger regularizer to produce diffeomorphic deformation fields that are not only smooth but also invertible~\cite{rueckert_diffeomorphic_2006}. For example, the folding areas shown in Figure~\ref{subfig:narrow_band} violate one-to-one correspondence between the original and target images (thus not invertible), and will be penalized by a diffeomorphic regularizer. The main drawback of diffeomorphic registration methods is that they are relatively more difficult to optimize. 

Alternatively, the folding problem can also be alleviated by replacing the mesh models with so-called ``meshless models'' that do not rely on explicit connections between control points, and thus are less sensitive to the positioning of control points. Meshless models were previously used in mechanical engineering~\cite{Melenk1996289,liu_gr_meshfree2009} for solving partial differential equations (PDEs), and have been recently introduced in computer graphics~\cite{ohtake2003multi} and image registration~\cite{makram510non}.  It has been shown that meshless models can improve computational accuracy under large deformations~\cite{liu_gr_meshfree2009,makram510non}. In this paper, we extend a meshless model called partition-of-unity method~\cite{Melenk1996289, ohtake2003multi, makram510non} and adopt it for shape registration. 

In summary, the registration methods in~\cite{paragios2003non,huang2006shape,taron2009registration} share two common drawbacks: the reliance on a narrow-band function that complicates the registration framework, and the use of a mesh-based deformation model with limited adaptation to the shape contour, in addition to causing a folding effect under large deformations.
Next, we propose a novel dissimilarity measure based on variational formulation of the classic chamfer-matching energy that does not rely on a proximity function, and in Section~\ref{sec:meshless}, we integrate this energy term with a meshless parametric-deformation model that can be naturally adapted to the shape contours, and also reduces the folding effect. 

\section{Variational chamfer-matching energy}
\label{sec:vchamfer}
The limitations of shape-embedding methods are partially due to the fact that their distance transforms are registered as normal 2-D images. In this case, the dissimilarity measure takes into consideration the entire embedding 2-D space, leading to a redundancy that unnecessarily distorts the contour-matching quality. The narrow-band function only partially reduces this problem. To completely remove this redundancy, we need a dissimilarity function that directly measures the shape alignment ``along'' the source and target contours without taking into consideration areas around the shapes. Meanwhile, we still want to use distance transform to represent the shapes due to its many advantages.

We start by observing that, when the source shape $\ShapeS$ is aligned with the target $\ShapeD$, the deformed shape $\ShapeS(\bft+\bfu)$ will coincide with the zero level set of $\Pi_{\ShapeD}$, \ie $\ShapeS(\bft+\bfu)\,\, \Pi_{\ShapeD}(\bfx)=0$. Consequently, we may enforce alignment between shapes  by minimizing the squared sum $\int_{\Omega} \left|\ShapeS(\bfx+\bfu)\,\, \Pi_{\ShapeD}\right|^2 d\bfx$, which corresponds to the classic chamfer-matching energy function~\cite{borgefors1988hierarchical}, used for affine shape registration. However, this functional can be ill-posed. 
For example, the energy function will vanish for any deformation field that shrinks the source shape to a single point on shape contour $\ShapeD$.

We can address this problem by including a symmetric term that measures the distance-error between target and source shapes, similarly to the classic \emph{symmetric chamfer-matching energy}~\cite{thayananthan2003shape}. In addition, we compensate for  scaling by dividing the distance-error by the contours' lengths, and minimize the following functional given by:
\begin{align}
	\VChamferE &= \frac{1}{A_\ShapeS}\int_{\Omega} |\ShapeS(\bfx+\bfu)\,\, \Pi_{\ShapeD}|^2 d\bfx + \frac{1}{A_\ShapeD} \int_{\Omega}|\ShapeD(\bfx)\,\, \Pi_{\ShapeS(\bfx+\bfu)}|^2 d\bfx \notag\\
	&= \frac{1}{A_\ShapeS}\underbrace{\int_{\Omega} \ShapeS(\bfx+\bfu)\,\, \Pi_{\ShapeD}^2 d\bfx}_{\text{forward energy $E^{\text{f}}$}} +\frac{1}{A_\ShapeD} \underbrace{\int_{\Omega}\ShapeD(\bfx)\,\, \Pi_{\ShapeS(\bfx+\bfu)}^2 d\bfx}_{\text{backward energy $E^\text{b}$}},
	\label{eq:vchamfer_data_term}
\end{align}
where $A_\ShapeS=\int_{\Omega} \ShapeS(\bfx+\bfu) d\bfx $ and $A_\ShapeD= \int_{\Omega} \ShapeD(\bfx) d\bfx$ are contour lengths as normalizing factors, and the second equation in (\ref{eq:vchamfer_data_term}) holds because $\ShapeS$ is a binary map and $ \ShapeS(\bfx)^2=\ShapeS(\bfx)$ . Additionally, $\VChamferE$ is independent of the sign of $\DistanceTransform$, and we have assumed that $\DistanceTransform_{\ShapeS}\geq 0$ and $\DistanceTransform_{\ShapeD}\geq 0$ (unsigned distance transform). 
The registration error is now directly measured using the shape contours without resorting to a proximity function as in~\cite{paragios2003non,huang2006shape,taron2009registration}. As in the classic chamfer matching, the use of distance transform allows for efficient optimization~\cite{paragios2003non,huang2006shape} by providing a distance-adaptive gradient function. Figure~\ref{fig:overall}(b) shows the two energy terms in our shape-registration functional. We now examine the minimization of our variational chamfer-matching energy.

\subsection{Functional minimization using Euler-Lagrange equation}
A common approach to minimizing a variational functional is to calculate its Euler-Lagrange equation~\cite{sapiro2001geometric}. For example, if the functional is given by:
\begin{align}
	E(\bfu)=\int L \left(\bfx,\bfu(\bfx), \dot{\bfu}(\bfx) \right) d\bfx,
\end{align}
then its Euler-Lagrange equation is calculated as follows:
\begin{align}
\mathbf{J}=\frac{\partial L}{\partial \bfu}- \frac{d}{d\bfx} \left(\frac{\partial L}{\partial \dot{\bfu}(\bfx)}\right)=0.
\end{align}
The Euler-Lagrange equation can rarely be solved directly, and it is often solved iteratively by converting the original problem to the  equivalent PDE:
\begin{align}
	\frac{\partial \bfu}{\partial t}=\mathbf{J(\bfu)}.
	\label{eq:euler_lagrange_pde}
\end{align}
In solving Equation~\ref{eq:euler_lagrange_pde}, $\mathbf{J}$ is used as a gradient function to facilitate iterative gradient-descent methods. However, calculating the gradients $\mathbf{J}$ from the variational functional in Equation~\ref{eq:vchamfer_data_term} is not always straightforward. In these cases~\cite{chenyang_xu}, we need to look for equivalent or approximating constraints to its exact Euler-Lagrange equation. 
In next section, we examine the forward and backward energy terms separately, and derive a simple gradient function for minimizing Equation~\ref{eq:vchamfer_data_term}.  We begin with an intuitive and informal derivation, and then provide a more rigorous prove in Section~\ref{subsec:simplified_gradient_minimization} that the derived function can indeed minimize the variational chamfer-matching energy.

\subsection{Simplified chamfer-matching gradient}
Starting with our forward energy $E^f(\bfu)$, its Euler-Lagrange equation can be written as follows:
\begin{align}
	\frac{\partial L^\text{f}}{\partial \bfu}=\Pi_{\ShapeD(\bfx)} ^2
	\frac{\partial \ShapeS(\bfx+\bfu)}{\partial \bfx} =0
	\label{eq:forward_euler}
\end{align}
where $\ShapeS$ is the source shape contour represented as a binary edge map. However, $\frac{\partial \ShapeS(\bfx+\bfu)}{\partial \bfx}$ does not exist in the classic limiting sense and is numerically sensitive, since edge maps  $\ShapeS$ and $\ShapeD$ are not even continuous as 2-D functions, let alone differentiable. This is the price we pay for our varitationl chamfer-matching formulation. Fortunately, under integration, the derivative of $\ShapeS(\bfx+\bfu)$ can be exchanged to that of $\Pi_{\ShapeD}^2$, according to the generalized derivative based on distribution theory~\cite{gel1964generalized} as:
\begin{align}
\int \frac{\partial L^\text{f}}{\partial \bfu}\, d \bfx &=\int \Pi_{\ShapeD} ^2(\bfx) \frac{\partial \ShapeS(\bfx+\bfu)}{\partial \bfx} \,d\bfx \notag \\
 &= \int -\frac{\partial \Pi^2_{\ShapeD}(\bfx)}{\partial \bfx} \ShapeS(\bfx+\bfu) \, d\bfx \notag\\
&= \int -2 \Pi_{\ShapeD}(\bfx) \frac{\partial \Pi_{\ShapeD}(\bfx)}{\partial \bfx} \ShapeS(\bfx+\bfu) \, d\bfx .
\label{eq:forward_gradient_simplify}
\end{align}

The rationale behind this conversion is illustrated in Figure~\ref{fig:derivative} by a simple numerical example. Figure~\ref{subfig:shape_contour_map} shows a 0--1 encoded binary map of a shape contour, with its finite differences along x and y directions shown in Figure~\ref{subfig:shape_contour_dx} and Figure~\ref{subfig:shape_contour_dy}, respectively. It is intuitive  that these difference maps reproduce the negative derivative operator along the shape contour. After integration with the target shape's distance map $\Pi_\ShapeD(\bfx)$, the derivative operators are essentially transfered over to $\Pi_\ShapeD(\bfx)$.
In this way, we may replace the original gradient operator with a ``better-behaved'' one on $\Pi_{\ShapeD(\bfx)}$, since $\Pi_{\ShapeD(\bfx)}$ is differentiable. Equation~\ref{eq:forward_gradient_simplify} suggests that we may approximate the original Euler-Lagrange equation in (\ref{eq:forward_euler}) by a simpler constraint as:
\begin{align}
	-2 \Pi_{\ShapeD} \frac{\partial \Pi_{\ShapeD}}
	{\partial \bfx} \ShapeS(\bfx+\bfu) =0.
	\label{eq:forward_linear_constraint}
\end{align}
\begin{figure}[t]
\begin{center}
\subfigure[$\ShapeS(\bfx)$]
{
	\includegraphics[width=.30\linewidth]
	{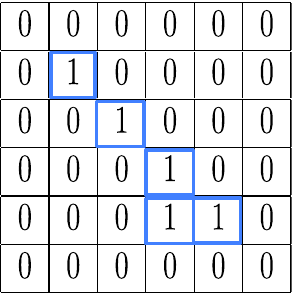}
	\label{subfig:shape_contour_map}
}
\subfigure[$\partial \ShapeS(\bfx) / \partial x $]
{
	\includegraphics[width=.30\linewidth]
	{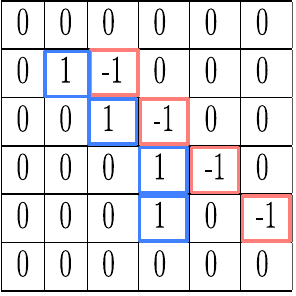}
	\label{subfig:shape_contour_dx}
}
\subfigure[$\partial \ShapeS(\bfx) / \partial y $]
{
	\includegraphics[width=.30\linewidth]
	{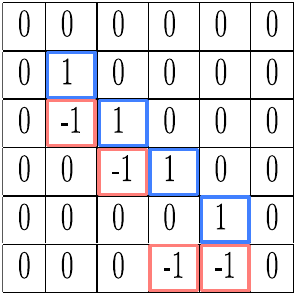}
	\label{subfig:shape_contour_dy}
}
\end{center}
\caption{Gradients of a 0--1 edge map reproduce the negative derivative operators. When integrated with distance map $\Pi_\ShapeD$, the operators are essentially transfered to $\Pi_\ShapeD$.}
\label{fig:derivative}
\end{figure}

Let us now examine the Equation~\ref{eq:forward_linear_constraint} in detail. From the definition of distance transform, we know that $\left\|\frac{\partial \Pi_{\ShapeD(\bfx)}}{\partial \bfx}\right\|=1$, \ie $\frac{\partial \Pi_{\ShapeD(\bfx)}}{\partial \bfx}$ is a \emph{direction vector} that points to the increasing value of $\Pi_{\ShapeD(\bfx)}$. Taking the negative sign into consideration, the left-hand side of Equation~\ref{eq:forward_linear_constraint} is a vector pointing in the direction that minimizes the distance of $\ShapeS(\bfx+\bfu)$ to $\ShapeD(\bfx)$, and the magnitude of this vector varies along the shape contour, proportionally to the distance of $\ShapeS(\bfx+\bfu)$ towards the target shape, \ie $\Pi_{\ShapeD(\bfx)} \ShapeS(\bfx+\bfu)$.  Figure~\ref{subfig:forward_force} shows the distribution of these vectors. They resemble a force that ``pulls'' the source shape to align with the target, hence we call it the \emph{forward force}.

We now turn our attention to the backward-energy term $E^\text{b}$. From the definition of $E^\text{b}$ in (\ref{eq:vchamfer_data_term}), we obtain:
\begin{align}
	\frac{\partial L^\text{b}}{\partial \bfu}=2\Pi_{\ShapeS(\bfx+\bfu)} \frac{\partial \Pi_{\ShapeS(\bfx+\bfu)}}{\partial \bfu} \ShapeD(\bfx) =0,
	\label{eq:backward_euler}
\end{align}
where the distance map $\Pi_{\ShapeS(\bfx+\bfu)}$ is a solution of the Eikonal equation in (\ref{eq:distance_transform_constraint}), given $\ShapeS(\bfx+\bfu)$ as its boundary condition, so the dependence of $\Pi_{\ShapeS(\bfx+\bfu)}$ on variations in $\bfu$ (\ie $\frac{\partial \Pi_{\ShapeS(\bfx+\bfu)}}{\partial \bfu}$) is difficult to analyze. This dependence can be greatly simplified if we assume that the deformation field is a simple shifting, \ie $\bfu(\bfx)=\bfu_0$, $\forall \bfx \in \Omega$ being a constant function. With this assumption, the shape deformation equals to the deformation of its distance map~\cite{paragios2003non}: $\Pi_{\ShapeS(\bfx+\bfu)}=\Pi_{\ShapeS}(\bfx+\bfu)$. Substituting this into (\ref{eq:backward_euler}), we obtain:
\begin{align}
\frac{\partial L^\text{b}}{\partial \bfu}&=2\Pi_{\ShapeS(\bfx+\bfu)} \frac{\partial \Pi_{\ShapeS}(\bfx+\bfu)}{\partial \bfu} \ShapeD(\bfx)\notag \\
&=2 \Pi_{\ShapeS(\bfx+\bfu)} \frac{\partial \Pi_{\ShapeS}(\bfx+\bfu)}{\partial \bfx} \ShapeD(\bfx) \notag\\
&=2 \Pi_{\ShapeS(\bfx+\bfu)} \frac{\partial \Pi_{\ShapeS(\bfx+\bfu)}}{\partial \bfx} \ShapeD(\bfx).
\label{eq:backward_linear_constraint}
\end{align}
We would like to clarify that $\frac{\partial \Pi_{\ShapeS(\bfx+\bfu)}}{\partial \bfx}$ in (\ref{eq:backward_linear_constraint}) is simply the gradient of $\Pi_{\ShapeS(\bfx+\bfu)}$ without further dependence on the deformed shape $\ShapeS(\bfx+\bfu)$. Using a more complex notation, this gradient term should be written as $\frac{\partial \Pi_{\ShapeS'}(\bfx)}{\partial \bfx}$ where $\ShapeS'=\ShapeS(\mathbf{y}+\bfu)$ is the deformed shape. Here, we keep its notation simpler and consistent with Equation~\ref{eq:forward_linear_constraint}. It is also worth pointing out that the above derivation is valid only under the assumption that $\bfu(\bfx)$ is constant for every $\bfx$, but we relax this assumption to arbitrary continuous $\bfu(\bfx)$, and use (\ref{eq:backward_linear_constraint}) as the gradient function for minimizing the backward energy. This is of course only an approximated function as we have ignored variations of the deformation function $\bfu(\bfx)$, but we show in the following section that this simplification is not arbitrary and the approximated gradient function will lead to the minimization of our chamfer-matching functional.

The gradient function in Equation~\ref{eq:backward_linear_constraint} has similar intuitive motivation as the forward force. By negating the left-hand side of Equation~\ref{eq:backward_linear_constraint}, we obtain $-\frac{\partial L^\text{b}}{\partial \bfu}$ that acts like a force ``pulling'' the target shape towards the source as shown in Figure~\ref{subfig:backward_force}. In actual minimization, the target shape is kept unchanged (static), thus the force is applied in reversed direction to the source shape, and we call it the \emph{backward force}.

Combining (\ref{eq:forward_linear_constraint}) with (\ref{eq:backward_linear_constraint}) and ignoring the constants, we may now approximate the Euler-Lagrange equation of the original variational functional (Equation~\ref{eq:vchamfer_data_term}) as follows:
\begin{align}
\mathbf{J}\approx \ApproxJ = \underbrace{-\Pi_{\ShapeD(\bfx)} \frac{\partial \Pi_{\ShapeD(\bfx)}}{\partial \bfx} \ShapeS(\bfx+\bfu)}_{\text{forward force}}+
\underbrace{\Pi_{\ShapeS(\bfx+\bfu)} \frac{\partial \Pi_{\ShapeS(\bfx+\bfu)}}{\partial \bfx} \ShapeD(\bfx)}_{\text{backward force}}.
\label{eq:vchamfer_gradient}
\end{align}
$ \ApproxJ$ closely resemble the gradient of classic chamfer-matching method~\cite{borgefors1988hierarchical}, and we call $ \ApproxJ$ the \emph{variational chamfer-matching gradient}. $ \ApproxJ$ can also be interpretated as a force field deforming the source contour, in a similar analogy used in the deformation of active contours~\cite{chenyang_xu}.
Finally, the strength of the combined force increases according to their mutual distances. 

Note that, in the chamfer-matching energy functional (Equation~\ref{eq:vchamfer_data_term}), we may also use the $\mathbf{L}^1$ norm instead of the squared-sum (\ie $\mathbf{L}^2$ norm). In this case, the forward and backward forces degenerate to direction vectors with uniform magnitude regardless to the mutual distance between the shapes. Our experiments showed that $\mathbf{L}^1$ norm is more sensitive to local minima, and leads to slower minimization convergence. This observation echoes a similar finding in the classic chamfer-matching method~\cite{borgefors1988hierarchical}.

\begin{figure}[t]
\begin{center}
\subfigure[]
{
	\includegraphics[width=.30\linewidth]
	{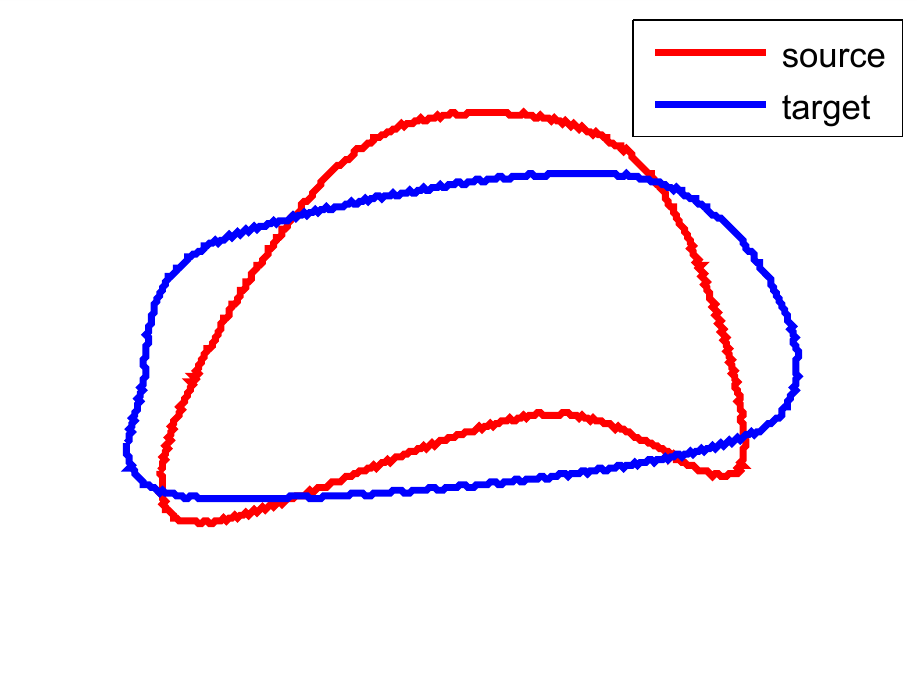}

}
\subfigure[]
{
	\includegraphics[width=.30\linewidth]
	{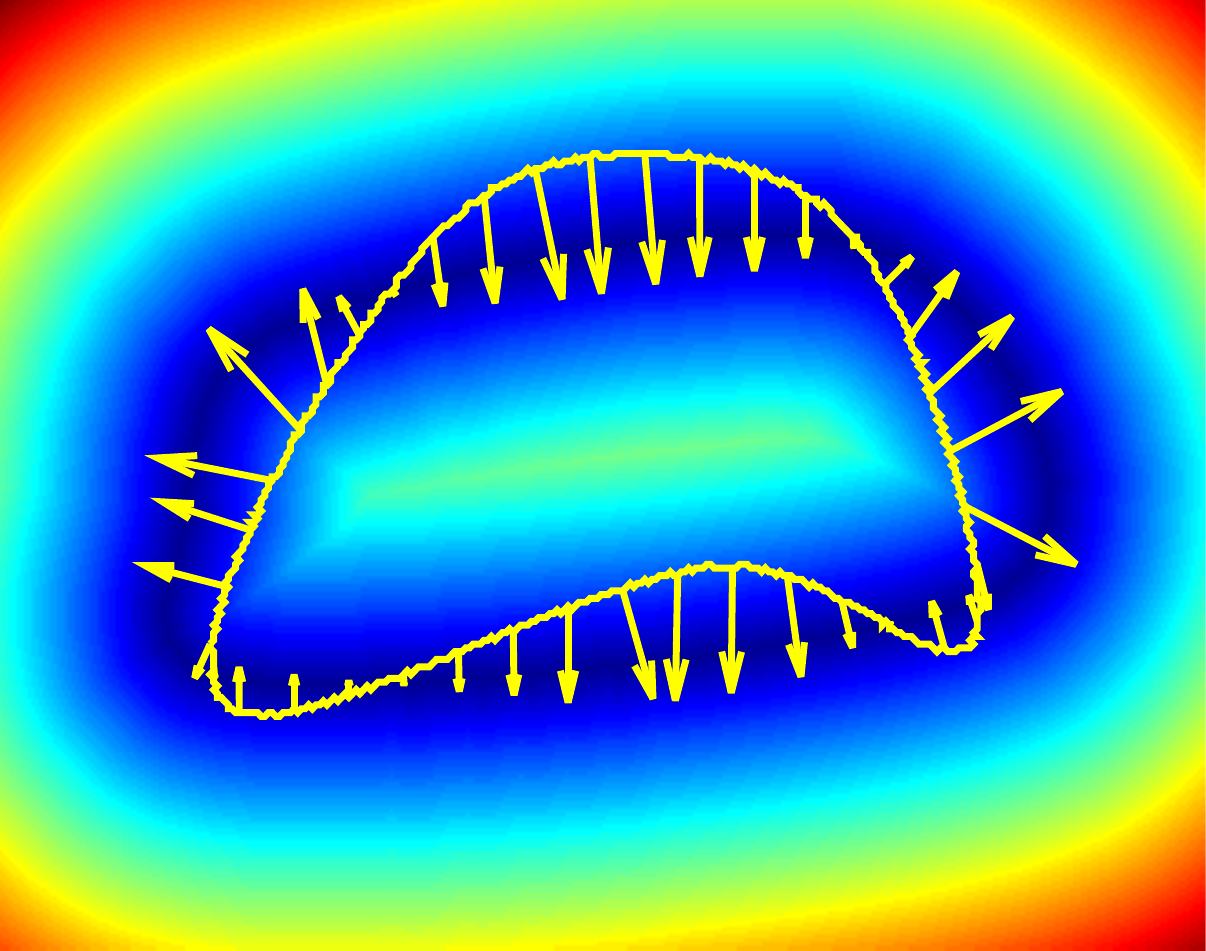}
	\label{subfig:forward_force}
}
\subfigure[]
{
	\includegraphics[width=.30\linewidth]
	{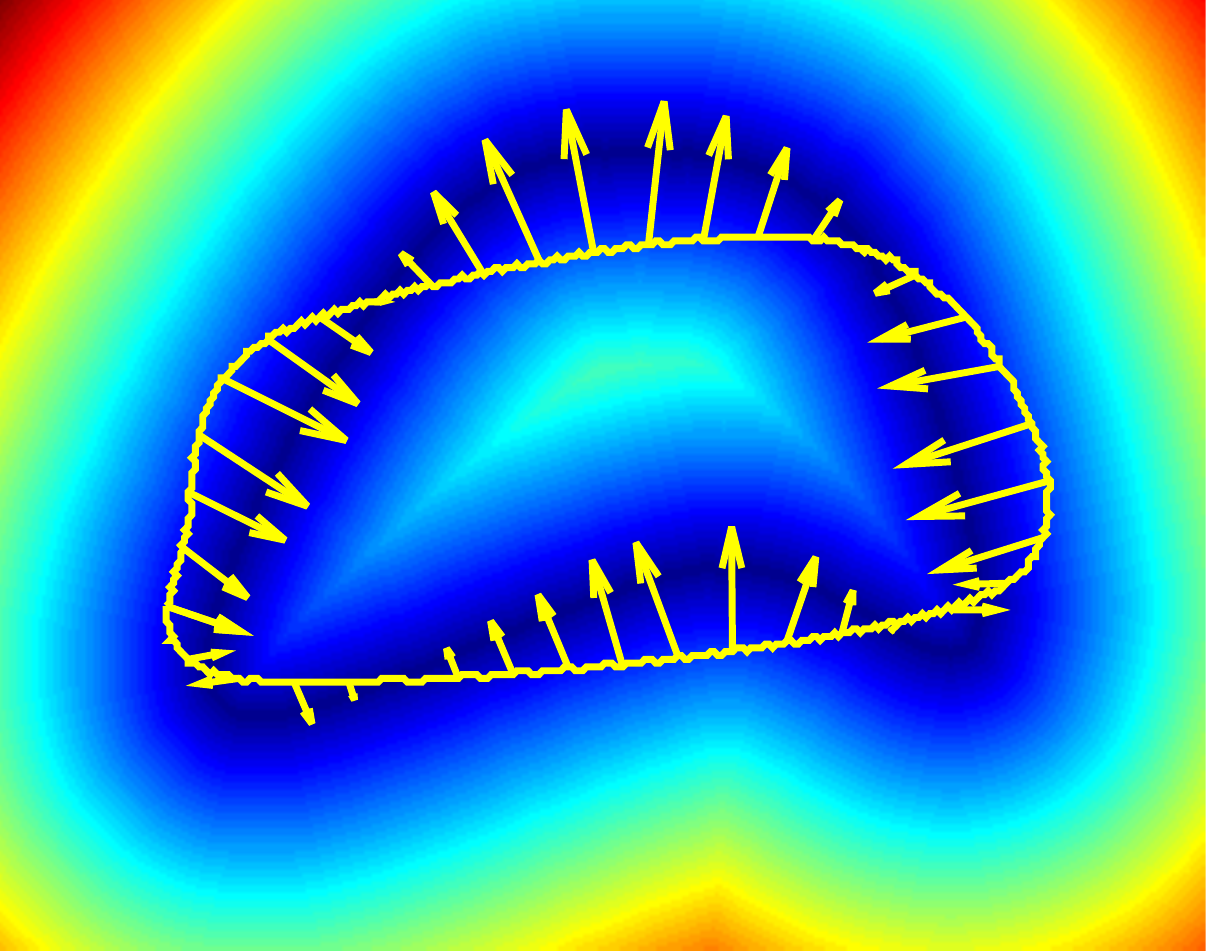}
	\label{subfig:backward_force}
}
\end{center}
\caption{Variational chamfer-matching gradient. a) Source and target shape contours. b) Forward force is distributed along the source shape (curve in yellow). The distance transforms are color-coded, with red and blue indicating the maximum and minimum values respectively.  
The forward force pulls the source towards the target. c) Backward force in reversed direction that pulls the target to the source. The magnitude of these forces are proportional to the distances between two shape contours.}

\label{fig:vchamfer_gradient_example}
\end{figure}

\subsection{Variational chamfer-matching based on the simplified gradient}
\label{subsec:simplified_gradient_minimization}
We have derived an approximate but simple gradient function $\ApproxJ$ in (\ref{eq:vchamfer_gradient}) to minimize the variational chamfer-matching functional $\VChamferE$. Now, we show that $\ApproxJ$  can indeed lead to the minimization of $\VChamferE$. From variational calculus, we know that a necessary condition for minimizing $\VChamferE$ is its Euler-Lagrange equation $\mathbf{J}=0$. However, this condition is not sufficient as the Euler-Lagrange equation may also hold when $\VChamferE$ is at local minima. Nevertheless, the Euler-Lagrange equation is widely used in computer vision to minimize variational functionals. Here, we wish to establish a similar relationship between the minimization of $\VChamferE$ and our approximate gradent function $\ApproxJ$, then the use of $\ApproxJ$ is theoretically justifiable. It turns out that $\ApproxJ=0$ is indeed a necessary and sufficient condition for the variational-matching energy to vanish, \ie $\VChamferE=0$. 

Before elaborating on the proof, we want to clarify that our discussion is limited to continuous deformation fields, \ie $\bfu \in C(\Omega)$ with $C(\Omega)$ as the set of continuous functions defined on domain $\Omega$. Vanishing of the chamfer-matching energy, \ie $\min\limits_{\bfu \in C(\Omega)} \VChamferE =0 $ can only be reached when two shapes $\ShapeS(\bfx+\bfu)$ and $\ShapeD(\bfx)$ can be exactly aligned, meaning that there exists a \emph{continuous} deformation $\bfu(\bfx)$, such that $\ShapeS(\bfx+\bfu) = \ShapeD(\bfx)$, $\forall \bfx \in \Omega$~\footnote{Rigorously, exceptions may exist for singular point set of zero measure. As far as discrete image data is concerned, this assertion is valid.}. A mismatch of the two shapes at any image pixel $\mathbf{p} \in \Omega$ would lead to $\VChamferE>0$, due to the non-negativeness of unsigned distance transform. As a result, if the source and target shapes have different topologies, then $\VChamferE>0$, $\forall \bfu(\bfx) \in C(\Omega)$ since there is no continuous mapping between shapes of different topologies. By focusing on $\VChamferE=0$, we have excluded the cases of aligning shapes with different topologies. Fortunately, these cases are not very interesting for many practical applications. When the two shapes have the same topology, then $\VChamferE=0$ is indeed the global minimum of $\VChamferE$. 
In summary, we can prove the following lemma:
\begin{lemma}
If $\min\limits_{\bfu \in C(\Omega)} \VChamferE =0 $, then $ \VChamferE =0 \Leftrightarrow \ApproxJ(\bfu)=0$.
\label{lem:the_one}
\end{lemma}
\begin{proof}
On one hand, $\VChamferE=0$ means that a continuous deformation field $\bfu(\bfx)$ exists such that $\ShapeS(\bfx+\bfu)=\ShapeD(\bfx)$, and $\Pi_{\ShapeD(\bfx)}=\Pi_{\ShapeS(\bfx+\bfu)}$.
By the definition of distance transform,  we obtain $\Pi_{\ShapeS(\bfx+\bfu)}\,\ShapeD(\bfx)=0$, $\Pi_{\ShapeD(\bfx)}\,\ShapeS(\bfx+\bfu)=0$ and $\ApproxJ(\bfx)=0$. Consequently, $\ApproxJ(\bfx)=0$ is indeed necessary for $\VChamferE=0$.  

On the other hand, let us assume $\ApproxJ(\bfx)=0$,  $\forall \bfx \in \Omega$, then we can show that $\ShapeS(\bfx+\bfu)=\ShapeD(\bfx)$ must be true. 
For example, if $\ShapeS(\bfx+\bfu)\neq \ShapeD(\bfx)$ for some $\bfx \in \Omega$, we can assume $ \ShapeD(\bfx)=1$ and $\ShapeS(\bfx+\bfu)=0$ without loss of generality. As a result, the forward force vanishes at $\bfx$, from the definition of $\ApproxJ(\bfx)$ (Equation~\ref{eq:vchamfer_gradient}). Since $\ApproxJ(\bfx)=0$, the backward force also become zero at $\bfx$, \ie $-\Pi_{\ShapeD(\bfx)}\frac{\partial \Pi_{\ShapeD(\bfx)}}{\partial \bfx}=0$, but we have assumed that $ \ShapeD(\bfx)=1$, so $\Pi_{\ShapeD}(\bfx) >0$ and $\left|\frac{\partial \Pi_{\ShapeD(\bfx)}}{\partial \bfx}\right|=1$, leading to a contradiction. Consequently, $\ApproxJ(\bfx)=0$ ensures that $\ShapeS(\bfx+\bfu)=\ShapeD(\bfx)$ holds for all $\bfx\in \Omega$. This in turn proves $\VChamferE=0$, and hence the sufficiency of $\ApproxJ=0$ for $\VChamferE=0$\qed.
\end{proof}

We have shown that $\ApproxJ(\bfu)$ can be used to minimize the variational chamfer-matching energy. Now, we can combine the chamfer-matching energy with different regularizers and deformation models for nonrigid shape registration such as  the derivative-based regularizer of  Equation~\ref{eq:functional_paragios} and B-Spline representation in~\cite{huang2006shape}. However, B-Spline models rely on  an explicitly connected control-point grid (\ie mesh), making it difficult to adapt the model to shapes of challenging topologies (holes, splits). Instead, the whole image domain is taken into consideration even though a large part of it can be irrelevant. Additionally, B-Spline models may suffer from the folding effect that control points collapse together under large deformation, causing numerical instability.
Next, we study how to remove the redundancy from parametric deformation models, by adopting a meshless representation that approximates the shape's deformation field. 

\section{Meshless deformation model}
\label{sec:meshless}
Our model is derived from recent developments in computer graphics~\cite{ohtake2003multi} and mechanical engineering~\cite{liu_gr_meshfree2009}, on building shape functions of arbitrary topology. Specifically, our model belongs to a group of so-called \emph{partition-of-unity} meshless methods. Although there are other meshless shape-deformation models based on thin-plate splines and radial basis functions (RBFs)~\cite{chen_hui_2006}, those methods are less accurate than polynomial-based representations as radial basis functions cannot exactly represent polynomial deformations (lack of reproducibility)~\cite{liu_gr_meshfree2009}. In the partition-of-unity model, one can ensure polynomial reproducibility by locally modeling shape deformation as polynomials. These local deformation models are then blended together into a global deformation field based on a set of weighting functions. We illustrate the partition-of-unity blending concept using the 1-D example shown in Figure~\ref{fig:pu_1d_example}. 
The figure shows 1-D patches represented by their weighting functions partitioning the data domain into three regions. The local deformation models are represented by the colors pink, blue, and red. The global deformation model is represented by the green polynomial curve obtained by blending the local models.  Here, we would like to remind the readers that the meshless model in our method is used for representing the incremental nonrigid deformation (\ie function $\bfu$ in Equation~\ref{eq:vchamfer_data_term}), but not the shape contours themselves (the shapes are represented by distance functions and binary maps). Our approach largely follows the partition-of-unity used in computer graphics~\cite{ohtake2003multi}. In the following subsections, we first introduce the local deformation model, and then explain how to blend them into a global model. 
\begin{figure}[t]
\begin{center}
	\includegraphics[width=.95\linewidth]
	{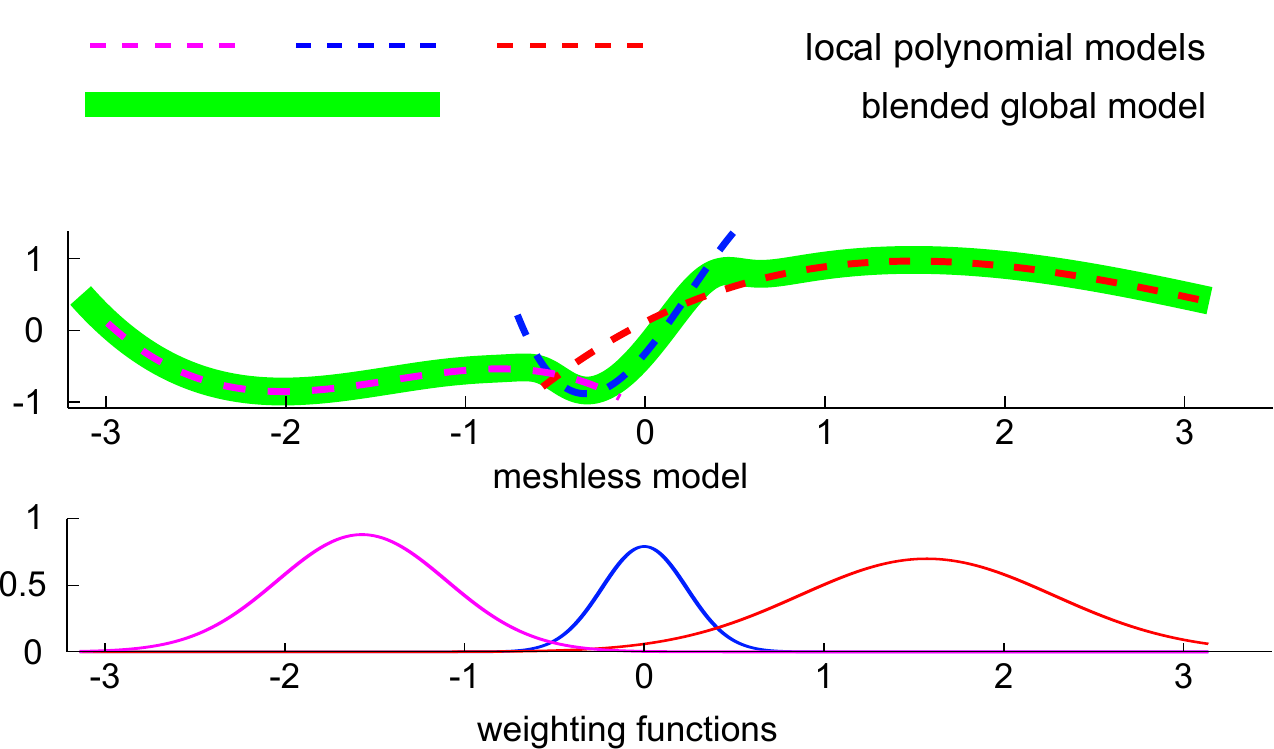}
\caption{A 1-D meshless model based on partition-of-unity. The local polynomial models are blended together (top), using predefined weighting functions (bottom).}
\end{center}
\label{fig:pu_1d_example}
\end{figure}

\subsection{Local polynomial model}
We commence by modeling shape deformation locally in scattered local domains called patches~\cite{Melenk1996289}. 
As in~\cite{ohtake2003multi}, we define these patches as the support of some weighting functions.  Various types of weighting functions exist~\cite{liu_gr_meshfree2009}. In this work, we use the 2nd-order B-Spline $w_p(\PosVector)=\alpha_p \,w_b\left(\frac{\|\PosVector-\bfp\|}{r_p}\right)$ for its computational efficiency, where $\alpha_p \in (0,1]$ is the patch's predefined influence factor in the final global blending. Formally, a 2nd-order B-Spline function is defined as a piecewise polynomial:
\begin{align}
	w(r) =
		\begin{cases} 
		\frac{3}{4}-r^2 &, r <\frac{1}{2}
		\\
		\frac{1}{2}(\frac{3}{2}-r)^2
		&, \frac{3}{2}\geq r \geq \frac{1}{2} \\ 
		0    
		&, r>\frac{3}{2}
	\end{cases},
	\label{eq:weight_function}
\end{align}
This weighting function's support is a disk with radius $r_p$ centered at $\bfp$ (ball in 3-D\footnote{A 3-D extension is straightforward.}). To summarize, a disk-like patch $p$ is defined by three parameters $(\mathbf{p},r_p,\alpha_p)$, \ie its center position, radius, and influence factor. The union of these patches covers the  image region of interest, $\Omega$, \ie $\Omega \subset \cup_p \supp \left(w_p\right)$, $p \in \BallSet$, with $\BallSet$ is the set of all patches. Note that although the weighting function in (\ref{eq:weight_function}) is a radial function, its usage is different from previous RBF models such as thin-plate splines~\cite{chen_hui_2006}. Here, RBFs are used for blending  local polynomial models instead of directly representing the shape deformation. 

It is also important to point out that the patches simply define a partition of the computational domain, and remain static during shape deformation, unlike B-spline control points, whose displacements directly control the deformation field. Instead, the local deformation $\bfu_p=(u,v)$ within each patch is defined by the parameters of a polynomial model centered at $\bfp$ as follows:
\begin{align}
	u(\PosVector)&=\sum_{s,t=0}^{s,t = \morder} a_{s,t} x^sy^t
	\hspace{.3in}\text{and}\hspace{.3in}
	v(\PosVector)=\sum_{s,t=0}^{s,t=\morder} b_{s,t} x^s y^t,
	\label{eq:local_model_slack}
\end{align}
Alternatively, $\LocalDeform_p(\PosVector)$ is a linear combination of monomial basis functions $\BasisVector{\Transpose}(\PosVector) = (1, x, y, xy, $ $x^2, y^2, \ldots,  x^\morder y^\morder)$, with a two-column coefficient matrix $\bfd_p$ as:
\begin{align}
	\bfu_p(\bfx)=\bfd_p{\Transpose} \BasisVector_p(\PosVector),
	\label{eq:local_model_compact}
\end{align}
where $\tvec\left(\mathbf{d}_p\right) = (a_{0,0} ,b_{0,0},\cdots,a_{\morder,\morder},b_{\morder,\morder} ){\Transpose}$ and $\BasisVector_p(\PosVector)=\BasisVector(\PosVector-\bfp)$ is the basis vector re-centered at $\bfp$. 
The sequence of monomials in $\BasisVector$ can be represented as a one-to-one mapping from the natural number to nonnegative 2-D monomial indices  $\xi : \mathbb{N} \leftrightarrow \mathbb{N}_0\times \mathbb{N}_0$. In this paper, the monomials are arranged in a Pascal-triangle manner~\cite{liu_gr_meshfree2009}, where low-order monomials are arranged before high-order ones. For monomials of the same order, the ones with more balanced indices in $x$ and $y$ are arranged first. This special arrangement increases numerical stability according to existing literature~\cite{liu_gr_meshfree2009}.

\subsection{Blending local models into a global deformation field}
\label{sec:global_deformation}
Once the local deformation models are at hand, the deformation at any point $\bfx$ is obtained by blending local fields of patches that contain $\bfx$. These patches are denoted by $N_{\bfx}=\{p\,\,|\,\,\bfx\in \supp\left(w_p\right)\}$. The blended global-deformation field is given by:
\begin{align}
	\bfu(\bfx)=\sum_{p\in N_{\bfx}} r_p(\bfx) \LocalDeform_p(\bfx), 
	\,\,\,\,\,\,\,\,\text{with}\,\,\,\,\,\,\,\,
	r_p\left(\bfx\right)=\frac{w_p(\bfx)}{\sum_{{p'}\in N_\bfx} w_{{p'}}(\bfx)}.
	\label{eq:blending_local}
\end{align}
Here, $r_p\left(\bfx\right)$ ensures the partition-of-unity (PU), so that Equation~\ref{eq:blending_local} is essentially a weighted average. Compared to the B-spline model, the partition-of-unity method does not need explicit connections between neighboring patches, and allows for arbitrary overlapping of local models. This blending scheme has been previously used in the polyaffine model for diffeomorphic image registration~\cite{arsigny2005fast}, meshless image registration~\cite{makram510non}, and point-based computer graphics~\cite{ohtake2003multi}. Some previous works~\cite{makram510non,Melenk1996289,liu_gr_meshfree2009} directly define the partition functions $r_p(\bfx)$ without constructing them from  weighting functions. However, weighting functions provide two appealing features. 

First, weighting functions are intuitive to handle, they allow us to assign different weights and scales to patches, and can potentially be made adaptive according to prior information of shape deformation\footnote{In the partition-of-unity literature~\cite{Melenk1996289}, other a priori knowledge can also be included about local behavior of the solution in the finite element space.}.  It is desirable that the positions, scales, and weighting factors of patches can be optimized according to registration results in a data-driven fashion. However, data-driven approaches are often application dependent, so this topic is beyond the scope of our paper. Here, we exploit a simple heuristic to register shape contours. Broadly speaking, we choose the patches' scales according to the shapes' mutual distance, since edge points that are far apart are more likely to have large-scale deformations. We found that this simple heuristic works well in most cases. Figure~\ref{subfig:node_distribution} shows an example shape with patches distributed along its contour. Figure~\ref{subfig:sum_of_weights} shows the sum of their weighting functions (\ie the denominator of partition functions $r_p$ in Equation~\ref{eq:blending_local}). Thus, we limit the computation of the deformation model to take place along shape contours, and achieve a similar effect as the narrow-band function used in~\cite{huang2006shape,paragios2003non}.

\begin{figure}[t]
\begin{center}
\subfigure[]
{
	\includegraphics[width=.46\linewidth]
	{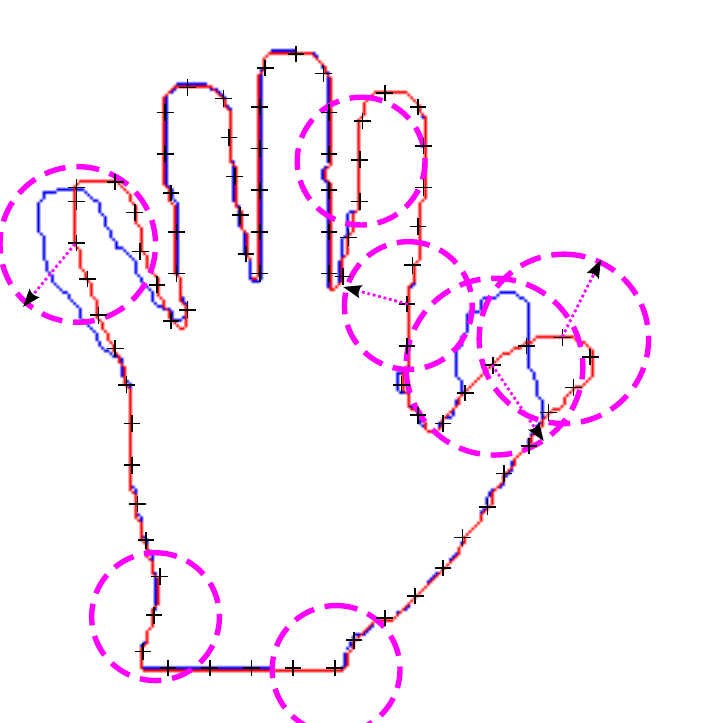}
	\label{subfig:node_distribution}
}
\subfigure[]
{
	\includegraphics[width=.46\linewidth]
	{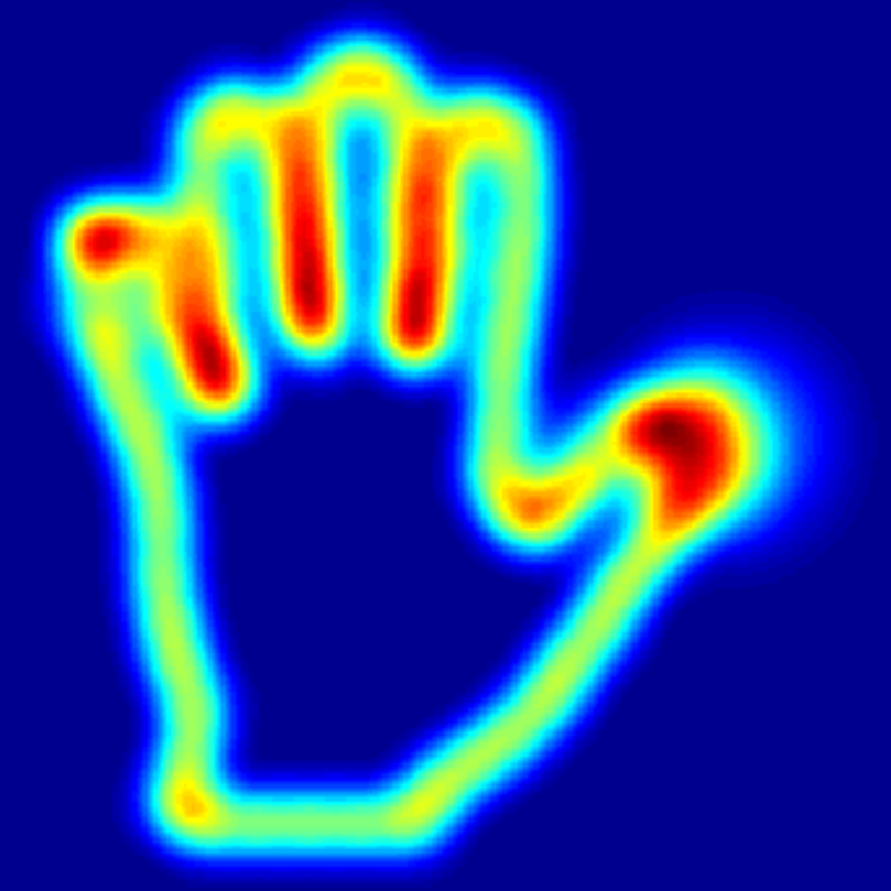}
	\label{subfig:sum_of_weights}
}
\end{center}
\caption{2-D meshless deformation model for shape registration. a) Circular patches are distributed along the shape contours with centers marked as black crosses, and their scales are adapted according to distance of the source (in red) to the target shape (in blue). b) The sum of weighting functions, whose support covers the interested image region.}
\label{fig:adaptive_nodes}
\end{figure}

Secondly, weighting functions also provide a way to specify  our confidence about the approximation accuracy of different local deformation models. This confidence measure will be used to regularize the global deformation field. Regularizers are important for preventing degenerating solutions, since the solutions of shape registration may not be unique. Regularizers remove ambiguous solutions by penalizing undesired fluctuations in the estimated deformation field, and prefer ``smooth'' solutions. Next, we study the problem of how to regularize a meshless shape-deformation model. 

\subsection{Smoothness regularizer}
\label{sec:consistency_enforcement}
We have shown that a global deformation field relating two shapes can be obtained by blending local deformation fields using Equation~\ref{eq:blending_local}. A classic regularizer used in B-spline models~\cite{huang2006shape} is to penalize the magnitude of derivatives of the resulting deformation. However, calculating derivatives of our meshless model is not straightforward. For example, according to the product rule, first-order derivatives of the global deformation are obtained as follows:
\begin{align}
	\nabla \bfu(\bfx) = \sum_{p\in N_\bfx} 
	\left(\nabla r_p(\bfx) \bfu_p(\bfx) +  
	r_p(\bfx) \nabla \bfu_p(\bfx)\right).
	\label{eq:first_order_derivative}
\end{align}
The above differentiation involves derivatives of local deformation models $\bfu_p(\bfx)$, and those of the partition functions $r_p(\bfx)$ given by:
\begin{align}
	\nabla r_p(\bfx) = \frac{\nabla w_p(\bfx) 
	\sum_{{p'}\in N_\bfx} w_{{p'}}(\bfx) - 
	w_p(\bfx) \sum_{{p'}\in N_\bfx} 
	\nabla w_{{p'}}(\bfx)}{{\left(\sum_{{p'}\in N_\bfx} 
	w_{{p'}}(\bfx)\right)}^2}.
	\label{eq:partition_function_derivative}
\end{align}
It is clear that calculation of $\nabla \bfu(\bfx)$ is complicated by the presence of blending functions. This problem was partially addressed in~\cite{makram510non}, by approximating  $\nabla \bfu(\bfx)$ using the differentiations of local models (\ie $\nabla \bfu_p(\bfx)$) as follows:
\begin{align}
	\widetilde{\nabla} \bfu(\bfx) = 
	\sum_{p \in N_\bfx} r_p(\bfx) \nabla \bfu_p(\bfx),
\end{align}
that is basically a weighted average of the derivatives of neighboring patches.
The approximated derivative $\widetilde{\nabla} \bfu(\bfx)$ is then penalized in a similar fashion as in the classic smoothness regularizer~\cite{huang2006shape}. However, this approximation is only accurate when neighboring patches have very similar derivatives (in conformity). This condition is achieved by penalizing a Sobolev conformity term between neighboring local deformation models as follows~\cite{makram510non}:
\begin{align}
	\mathbf{S}_k^{p,q} = \sum_{|\eta|\leq k} 
	\int w_p(\bfx) w_q (\bfx) \|\mathbf{D}^\eta \bfu_p(\bfx) -
	\mathbf{D}^\eta \bfu_p(\bfx)\|^2_2 d\bfx,
	\label{eq:sobolev_conformity}
\end{align}
where $\mathbf{D}^\eta \bfu_p(\bfx)=\partial^{\eta_x}_x \partial^{\eta_y}_y \bfu_p(\bfx)$, $\eta_x, \eta_y = 0,1,\ldots, k$, and $\eta=(\eta_x, \eta_y)$ is a multi-index vector with total order $|\eta|=\eta_x+\eta_y$. The usage of multi-index $\eta$ is to include various orders of derivatives.  In other words, the Sobolev conformity term is essentially a weighted squared sum of difference (SSD) between the derivatives of neighboring polynomial models. By enforcing $\mathbf{S}_k^{p,q}\rightarrow 0$ for each pair of patches $(p,q)$, the local deformation models are forced to have similar derivatives. 
Equation~\ref{eq:sobolev_conformity} can be simplified into a quadratic term of local polynomial coefficients, but this approach is still much more involved than regularizing mesh models, and it involves two components: the Sobolev conformity term and a classic gradient-based smoothness term.

We found that a much simpler regularizer can be derived for smoothing the deformation field. First, the conformity between local models can be simply enforced by shifting and comparing the coefficients of local polynomial representations, instead of using an integral SSD form as in Equation~\ref{eq:sobolev_conformity}. In this way, the conformity regularizer is greatly simplified, without having to compare the derivatives of local models. To differentiate with the conformity regularizer, we call our simplified functional as \emph{consistency regularizer}, as it penalizes the inconsistencies in the polynomial coefficients of neighboring patches, instead of measuring the conformity of their derivatives. Secondly, our consistency term also penalizes fluctuations in the deformation field, and is enough to ensure its smoothness, further simplifying the registration framework by replacing the classic derivative-based regularizer.

\begin{figure}[t]
\begin{center}
\subfigure[$\lambda=0$]
{
	\includegraphics[width=.45\linewidth]
	{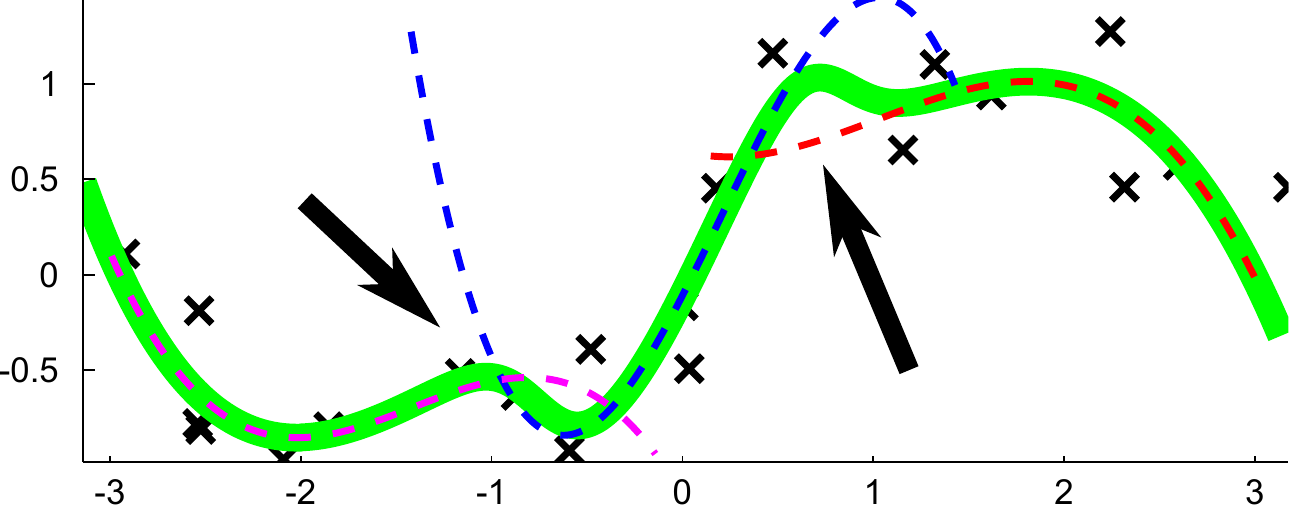}
	\label{subfig:inconsistencies}
}
\subfigure[$\lambda=0.001$]
{
	\includegraphics[width=.45\linewidth]
	{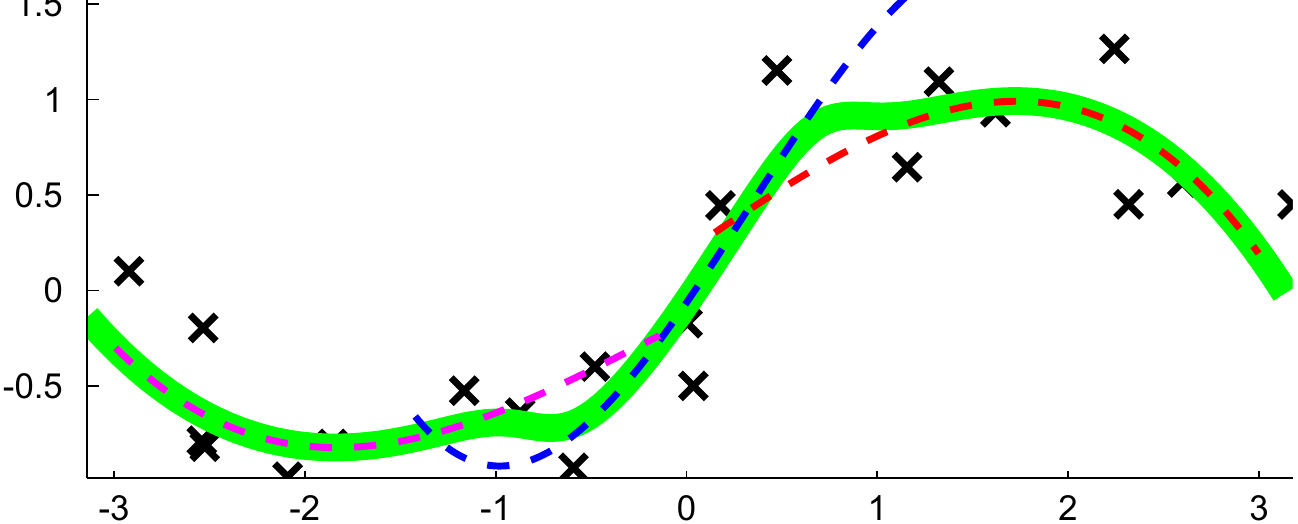}
}
\subfigure[$\lambda=0.05$]
{
	\includegraphics[width=.45\linewidth]
	{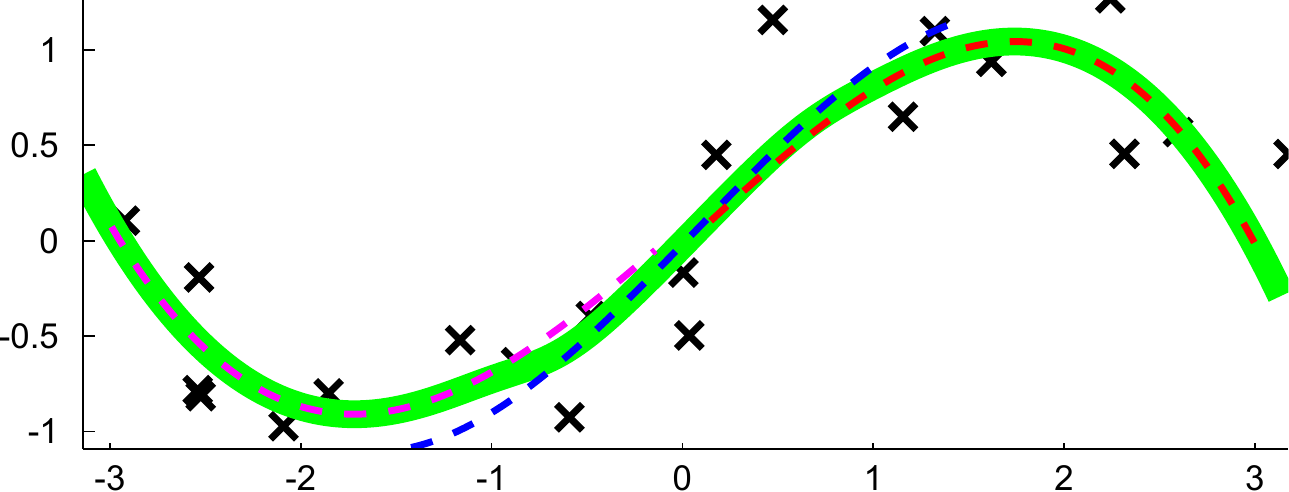}
}
\subfigure[$\lambda=1.0$]
{
	\includegraphics[width=.45\linewidth]
	{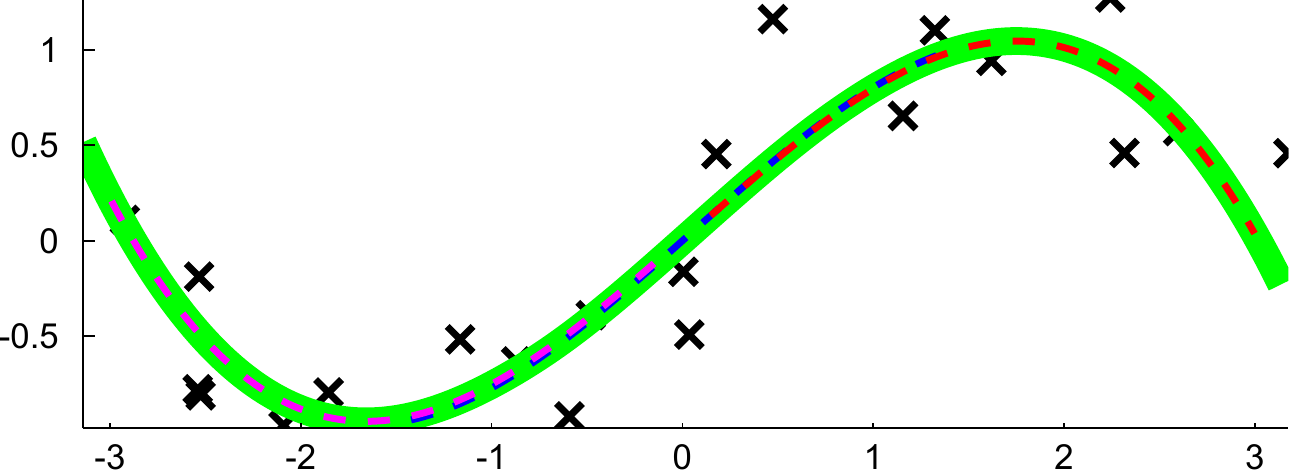}
}
\end{center}
\caption{Consistency regularizer for 1-D scatter-point approximation. Data points were generated from sinusoid function with added noise. Three patches of equal radius and weighting factors were used to partition the computational domain $\left[-\pi, \pi\right]$. The parameter $\lambda$ controls the relative contribution of the regularizer term. a) Fluctuations in the global model are related to inconsistencies between local polynomial models. b--d) By increasingly penalizing inconsistencies among local models, the global model becomes smoother.}
\label{fig:consistency_regularizer}
\end{figure}

In Figure~\ref{fig:consistency_regularizer}, we show the idea of our consistency regularizer using a simple 1-D example of approximating  scattered data points. Intuitively, the fluctuations in the global model is related to the inconsistencies among local deformation models (Figure~\ref{subfig:inconsistencies}). One may consider that consistency between two local deformation fields, $\LocalDeform_p$ and $\LocalDeform_q$, can be measured as the Euclidean distance between their parameters $\bfd_p$ and $\bfd_q$ as follows:
\begin{align}
	E_{p,q}=\left\|\tvec(\bfd_p-\bfd_q)\right\|^2.
	\label{eq:consistency_simple}
\end{align}
Here, we stretched the parameter matrices $\bfd_p$, $\bfd_q$ into vectors $\tvec(\bfd_p-\bfd_q)$.
In fact, a similar idea has been used for smoothing B-splines~\cite{eilers1996flexible} by penalizing differences of adjacent B-splines' parameters. However, for meshless models, the consistency regularizer based on (\ref{eq:consistency_simple}) incorrectly compares local deformation coefficients from different local coordinate systems. We demonstrate this effect using another 1-D example as shown in Figure~\ref{fig:coordinate_shifting}. In this example, three linear (\ie first-order polynomial) deformation models are blended into a global one. Even though the local models are themselves consistent, their coefficients are different due to different coordinate systems. The regularizer in Equation~\ref{eq:consistency_simple} failed to take the coordinate transform into consideration. To address this issue, we need to shift and align the local coordinate systems so that $\bfd_p$ and $\bfd_q$ is compared on a common ground.

\begin{figure}[t]
\begin{center}
	\includegraphics[width=.98\linewidth]
	{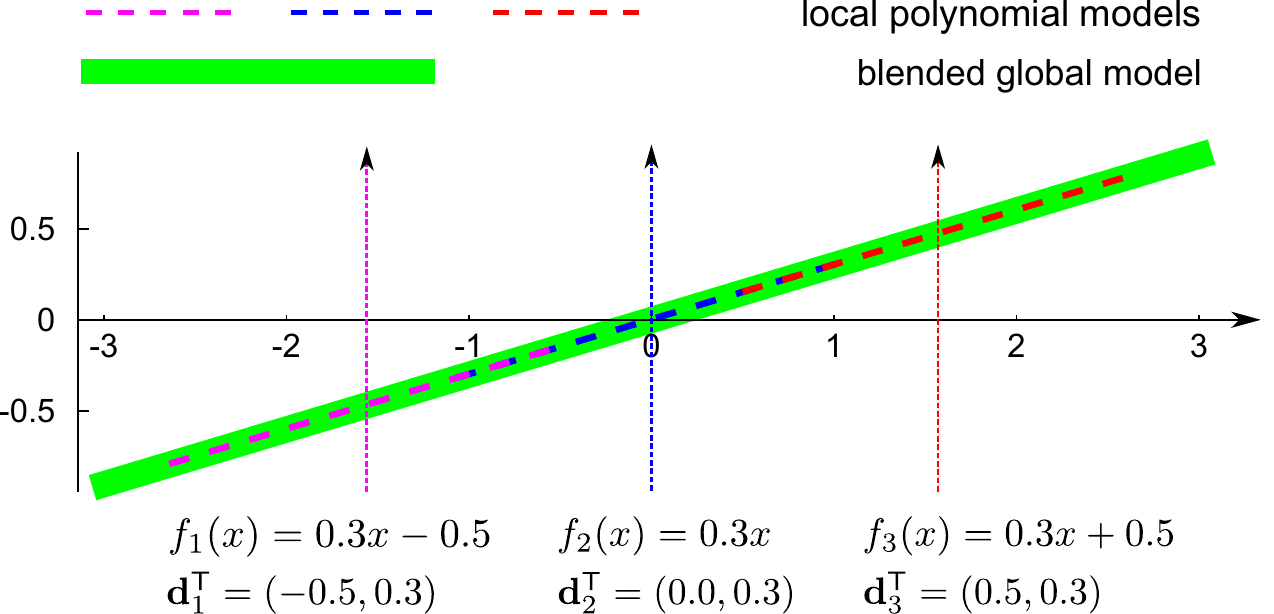}
\end{center}
\caption{A 1-D example of shifting local coordinate systems. Three local models are blended into a global one. The local models are represented using first-order polynomials, with basis functions $\BasisVector(\bfx)=\left(1,x\right){\Transpose}$. Shifting the coordinate system leads to a linear transform of local model coefficients.}
\label{fig:coordinate_shifting}
\end{figure}

Fortunately, aligning the local coordinate systems can be achieved using a linear operator. To begin with, the representation of $\bfu_q(\bfx)$ in the coordinate system of another patch $p$ can be obtained by shifting $\bfu_q(\bfx)$ as follows: 
\begin{align}
	\bfu_q(\bfx) &=\bfd_q{\Transpose} \BasisVector_q(\bfx) \notag \\
	&=\bfd_q{\Transpose} \BasisVector(\bfx-\bfq) \notag\\
	&=\bfd_q{\Transpose} \BasisVector(\bfx-\bfp+\bfp-\bfq) \notag\\
	&=\bfd_q{\Transpose} \BasisVector_p(\bfx+\underbrace{\bfp-
	\bfq}_{\boldsymbol{\Delta} {\bf x}})
	\label{eq:shifting_local}
\end{align}
As a result, the alignment of local deformation models reduces itself to  the shifting of the basis functions $\BasisVector_p(\bfx)$. Our choice of monomial basis functions makes this computation particularly simple. In general, shifting a polynomial basis $\BasisVector(\bfx)$ by $\boldsymbol{\Delta} {\bf x}=(\delta x, \delta y)$ leads to: 
\begin{align} 
	\BasisVector(\PosVector+ \boldsymbol{\Delta} {\bf x})&=
	       \left(1,x+\delta x, y+\delta y, (x+\delta x)(y+\delta y),
	       \ldots,(y+\delta y)^m\right){\Transpose}\notag\\
			&=\mathbf{S}{\Transpose}(\boldsymbol{\Delta} {\bf x}) 
			\BasisVector(\PosVector),
	\label{eq:basis_shifted}
\end{align}
where $\mathbf{S}{\Transpose}(\boldsymbol{\Delta} {\bf x})$ is the linear \emph{basis-shifting-operator}. For example, for second-order basis vector $\BasisVector(\bfx)=\left(1,x,y,xy,x^2,y^2\right){\Transpose}$, its shifting-operator is:
\begin{align}
	\mathbf{S}{\Transpose}(\boldsymbol{\Delta} {\bf x}) = \left(
	\begin{array}{cccccc}	
		1    &    \delta x   &   \delta y   &   \delta x \delta y   &    \delta^2 x    & \delta^2 y\\
		0    &    1          &   0          &   \delta y            &    2\delta x     & 0\\
		0    &    0          &   1          &   \delta x            &    0             & 2\delta y\\
		0    &    0          &   0          &   1                   &    0             & 0\\
		0    &    0          &   0          &   0                   &    1             & 0\\
		0    &    0          &   0          &   0                   &    0             & 1
	\end{array}\right).
	\label{eq:linear_shift_matrix}
\end{align}
By substituting (\ref{eq:basis_shifted}) into (\ref{eq:shifting_local}), we obtain the shifted local deformation field:
\begin{align}
\bfu_{q\rightarrow p}(\bfx)&=\bfd_q{\Transpose} \mathbf{S}{\Transpose}(\boldsymbol{\Delta} \bfx)\BasisVector_p(\PosVector) \notag \\
&=\left(\mathbf{S}(\boldsymbol{\Delta} \bfx)\bfd_q \right){\Transpose} \BasisVector_p(\PosVector).
\end{align}
Here, we used $\bfu_{q\rightarrow p}(\bfx)$ to distinguish $\bfu_{q}$ from its shifted representation.
Therefore, shifting the local coordinate system leads to linearly transformed polynomial coefficients $\bfd_{q \rightarrow p}=\mathbf{S}(\boldsymbol{\Delta} \bfx)\bfd_q$. Now, we can modify the consistency term in Equation~\ref{eq:consistency_simple} to include the shifting operator as:
\begin{align}
	\Ecpq=\left\|\tvec(\bfd_p-\mathbf{S(p-q)}\bfd_q)\right\|^2.
	\label{eq:consistency}
\end{align}
Finally, we can sum up the consistency term for every pair of neighboring patches as follows:
\begin{align}
	E^c&=\frac{1}{N}\sum_{p} \left(\sum_{q} \wqp \Ecpq \right),
	\label{eq:total_consistency}
\end{align}
where $N$ is the number of patches, and the terms $E_{p,q}^c$ are weighted by $w_q$ based on the distance between patches. The inclusion of weighting functions has two purposes. First, it effectively limits the interaction range of a patch $p$ to its neighboring patches $q$ that contain the center of $p$, otherwise $\wqp =0$ . Secondly, the weighting function gives more emphasis to the consistency of neighboring patches that are close to each other.

\subsection{Verification of the regularizer}
We have introduced a simple consistency regularizer based on the linear shifting operator in (\ref{eq:total_consistency}). Before adopting it for shape registration, we would like to briefly verify its ability of smoothing the deformation fields. We have shown in Figure~\ref{fig:consistency_regularizer} an example of 1-D approximation of scattered data points, combining the consistency regularizer with a SSD data-fitting cost function. Formally, the combined data-fitting functional is given by:
\begin{align}
E_{\text{approx}}&=\sum_{i} \left(\bfu(\bfx)-f_i(\bfx_i)\right)^2 + \lambda E^{c},
\end{align}
where $(\bfx_i, f_i(\bfx_i))$ are scattered data points, and $\lambda$ is the weight of the consistency regularizer. Figure~\ref{fig:consistency_regularizer} illustrates how the increase of  the regularizer's weight leads to smoother global approximation, and the global model approximates a polynomial fitting asymptotically. The effect of the consistency regularizer can also be analyzed theoretically. Here, we aim at providing an intuitive understanding of the regularizer rather than a rigorous proof. 

On one hand, we can show that minimizing the consistency term in (\ref{eq:consistency_simple}) can also enforce the conformity between local deformation models, \ie the continuity in their derivatives. To verify, note that the  differential operator on polynomials can again be represented as a linear operator on the coefficients $\bfd_p$. Let us denote a general differential operator and its matrix representation using the same symbol $\mathbf{D}^\eta$. Combining it with our shifting operator into the conformity term in Equation~\ref{eq:sobolev_conformity}, we obtain:
\begin{align}
\mathbf{S}^{p,q}_{k}&=\sum_{|\eta|\leq k}\int w_p(\bfx)w_q(\bfx)\left(\bfD^\eta \bfu_p(\bfx)-\bfD^\eta \bfu_{q\rightarrow p}(\bfx)\right)^2 d \bfx \notag \\
&=\sum_{|\eta|\leq k}\int w_p(\bfx)w_q(\bfx)\left(\bfD^\eta \bfd_p\BasisVector_p(\bfx)-\bfD^\eta \mathbf{S(p-q)} \bfd_q \BasisVector_p(\bfx)\right)^2 d \bfx \notag \\
&\leq  \left\|\tvec\left({\bfd_p-  \mathbf{S(p-q)} \bfd_q}\right)\right\|^2 \,\,\sum_{|\eta|\leq k} \left\|\bfD^\eta\right\|^2 \int w_p(\bfx)w_q(\bfx) \BasisVector_p^2(\bfx) d\bfx.
\end{align}
Consequently, the conformity term is upper-bounded by our consistency regularizer, and minimizing the consistency regularizer will similarly enforce the conformity condition among local models.

On the other hand, when $E^c \rightarrow 0$, the blended global model is coerced into a polynomial one, and the order of this polynomial is the same as the basis functions. For instance, if we choose the linear basis vector
 $\BasisVector=(1,x,y){\Transpose}$, then $E^c=0$ implies that $\bfu(\bfx)$ is a piecewise linear function and $\|\nabla \bfu_x\|^2+\|\nabla \bfu_y\|^2=0$. As a result, our consistency regularizer has the same asymptotic property as the classic gradient-based regularizers. 

We have shown the relationship of our consistency regularizer with both the conformity term used in a recent partition-of-unity method~\cite{makram510non} and the gradient-based regularizers for mesh models~\cite{huang2006shape}. Next, we formulate and solve the nonrigid shape registration problem, by combining the meshless deformation model with our variational chamfer-matching energy.

\section{Registration using gradient-descent method}
\label{sec:gradient}
With the meshless deformation model, the registration problem is now converted to the one of seeking the ``best'' deformation parameters $\bfd_p$, $p \in \BallSet $, by combining both the variational chamfer-matching and the consistency regularizer. The chamfer-matching energy measures the closeness of registered shape contours (data term), and the consistency regularizer penalizes fluctuations in the deformation field (smoothness term). Formally, we formulate nonrigid shape registration as the following functional minimization problem:
\begin{align}
 	\mathbf{d}_p &=\argmin_{\mathbf{d}_p} 
	E^v=\argmin_{\mathbf{d}_p} 
	\left(E^{d}(\bfu)+\lambda E^{c}\right),
	\label{eq:functional}
\end{align}
where parameter $\lambda$ defines the relative importance of the consistency term. As in previous works~\cite{huang2006shape,borgefors1988hierarchical}, minimizing (\ref{eq:functional}) can be achieved using gradient descent. The gradients can be derived as follows given by:
\begin{align}
\frac{\partial E^v}{\partial \bfd_p}&=\frac{\partial E^d(\bfu)}{\partial \bfd_p} +\lambda \frac{\partial E^c}{\partial \bfd_p}.
\end{align}
Gradients of the consistency energy $\frac{\partial E^c}{\partial \bfd_p}$ can be calculated from (\ref{eq:total_consistency}) as:
\begin{align}
\frac{\partial E^c}{\partial \bfd_p}&=\frac{1}{N}\sum_q w_q(\|\bfp-\bfq\|) \frac{\partial E_{p,q}^c}{\partial \bfd_p} \notag \\
&=\frac{1}{N}\sum_q w_q(\|\bfp-\bfq\|) \left(\bfd_p-\mathbf{S}(\bfp-\bfq)\bfd_q\right).
\label{eq:jacob_consistency_term}
\end{align} 
Gradients of the variational chamfer-matching energy can be calculated from chain-rule of variational calculus, to obtain:
\begin{align}
\frac{\partial E^d(\bfu)}{\partial \bfd_p}&=\int \mathbf{J(\bfx)} \cdot \frac{\partial \bfu(\bfx)}{\partial \bfd_p} d\bfx \notag \\
&\approx \int \mathbf{\widetilde{J}(\bfx)} \cdot r_p(\bfx) \BasisVector(\bfx) d\bfx.
\label{eq:jacob_data_term}
\end{align}
Here, $\mathbf{J(\bfx)}$ is left-hand-side of the Euler-Lagrange equation of $E^d(\bfu)$, and we use its approximation $\ApproxJ(\bfx)$ defined in Equation~\ref{eq:vchamfer_gradient}.

Using the derived gradients, we implemented an optimization algorithm based on the Broyden-Fletcher-Goldfarb-Shanno (BFGS) method~\cite{bonnans2006numerical}. At each iteration of the algorithm, the source shape is first warped using the deformation field reconstructed from local field parameters (Equation~\ref{eq:blending_local}), and then its distance transform $\Pi_{\ShapeS(\bfx+\bfu)}$ is updated. Both the destination shape $\ShapeD$ and its distance transform $\Pi_{\ShapeD}$ remain constant. Conceptually, the gradient-descent process is described in Algorithm~\ref{algo:gradient_descent}.

\begin{algorithm}[t]
\label{algo:gradient_descent}
\KwIn{Binary maps for source and target shapes ($\ShapeS$ and $\ShapeD$), partition-of-unity patches $\BallSet$, consistency weighting factor $\lambda$.}
\KwOut{Deformation field $\bfu(\bfx)$.}
Initialize local deformation coefficients $\bfd_p^0 \leftarrow 0$, for $p \in \BallSet$\; 
Calculate distance map of the target shape $\Pi_\ShapeD$\;
\While{Not converge}
{
Calculate global deformation $\bfu^k(\bfx)$ from local deformation coefficients $\bfd_p^k$ (Equation~\ref{eq:blending_local})\;
Warp source shape as $\ShapeS^k(\bfx)=\ShapeS(\bfx+\bfu^k)$ \;
Calculate distance map $\Pi_{\ShapeS^k(\bfx)}$\;
Update Jacobian matrix (Equation~\ref{eq:jacob_consistency_term} and Equation~\ref{eq:jacob_data_term})\;
Update $\bfd_p^{k+1}$ using gradient-descent rules\;
}
\caption{Shape registration based on gradient-descent.}
\end{algorithm}

Direct warping of a shape contour might cause deterioration of image quality since the deformation $\ShapeS(\bfx+\bfu)$ is often implemented using interpolated backward mapping. Evidently, a warped 0--1 edge map may produce decimal pixel values, \ie no longer a binary image. 
In practice, the deterioration of edge maps is rarely an issue for our algorithm, and can be naturally handled by our variational chamfer-matching gradient (Equation~\ref{eq:vchamfer_gradient}), by relaxing the binary edge map to a gray-scale one. Furthermore, we can use an extended version of the distance transform defined on gray-scale images that takes the gray-scale edge map value into account. Nevertheless, in this paper, we want to keep the edge map binary, for better evaluation of the registration results, since the compared methods in~\cite{huang2006shape,chen_hui_2006} are based on binary representations.

%

In previous implementation of distance-based registration~\cite{huang2006shape}~\footnote{Available for download at \url{http://www.cse.lehigh.edu/~huang/downloads.htm}}, the shape contours are represented as connected line segments, that is basically a vectorized representation. Warping of the shape contours is achieved by deforming the end points of these line segments. Thus, interpolation and resulting  deterioration of edge maps can be avoided. This approach may be adopted for our method, but is not straightforward to implement. Fortunately, in many applications, the input shape contours are detected from grayscale source images. In these cases, we simply warp the source images, and then apply an edge-detection step to recover the deformed binary contour. 
To summarize, we modify the registration Algorithm~\ref{algo:gradient_descent} to obtain an edge-preserving one given by Algorithm~\ref{algo:edge_preserve} in which the main differences are highlighted in italic. Both the distance transform and edge detection are highly efficient  operations, with minimal influence on the algorithm's computational cost.

\begin{algorithm}[t]
\label{algo:edge_preserve}
\KwIn{{\it Source and target images ($I_s$ and $I_t$)}, partition-of-unity patches $\BallSet$, consistency weighting factor $\lambda$.}
\KwOut{Deformation field $\bfu(\bfx)$.}
Initialize local deformation coefficients $\bfd_p^0 \leftarrow 0$, for $p \in \BallSet$\; 
{\it Detect target shape contour from the image $\ShapeD\leftarrow I_t$\;}
Calculate distance map of the target shape $\Pi_\ShapeD$\;
\While{Not converge}
{
Calculate global deformation $\bfu^k(\bfx)$ from local deformation coefficients $\bfd_p^k$ (Equation~\ref{eq:blending_local})\;
{\it Warp source image as $I_s^k(\bfx)=I_s(\bfx+\bfu^k)$ \;}
{\it Detect source shape contour $\ShapeS^k \leftarrow I_s^k$ \;}
Calculate distance map $\Pi_{\ShapeS^k(\bfx)}$\;
Update Jacobian matrix (Equation~\ref{eq:jacob_consistency_term} and Equation~\ref{eq:jacob_data_term})\;
Update $\bfd_p^{k+1}$ using gradient-descent rules\;
}
\caption{Gradient-descent with edge-preserving.}
\end{algorithm}

Finally, we adopt the hierarchical multi-scale registration strategy used in previous nonrigid registration methods~\cite{huang2006shape,taron2009registration,makram510non}  in order to avoid local minima (\ie a coarse-to-fine approach). Briefly, multi-scale registration starts with sub-sampled images (at the coarser level). The deformation obtained at coarse scale is then used to initialize the algorithm at finer scales.

\section{Evaluation}
\label{sec:experiments}
In this section, we evaluate our method on a number of experiments, and compare it with the methods by Huang~\etal~\cite{huang2006shape} and Chen and Bhanu~\cite{chen_hui_2006}. As described in Section~\ref{sec:distance_functions}, Huang's method is based on distance-map representation and a B-spline deformation model, while Chen and Bhanu's method is based on shape-context and a thin-plate deformation model. Our goal in this comparison is to demonstrate two aspects of our method. First, by removing the narrow-band function, our chamfer-matching functional can help improving the registration accuracy of distance-based methods at high-curvature regions. Secondly, by replacing spline-based mesh models with our partition-of-unity model, we can handle challenging shape deformations that may be difficult for existing methods. We tested our method on the Brown University shape dataset~\cite{Sharvit:etal:JVIS:1998}, that was also used by many previous methods~\cite{huang2006shape, chen_hui_2006, paragios2003non}. Additionally, we demonstrate an application of our method to image sequences of cell mitosis.

\begin{figure}[h!!!!]
	\begin{center}
			{\includegraphics[width=.12\linewidth]
		  		{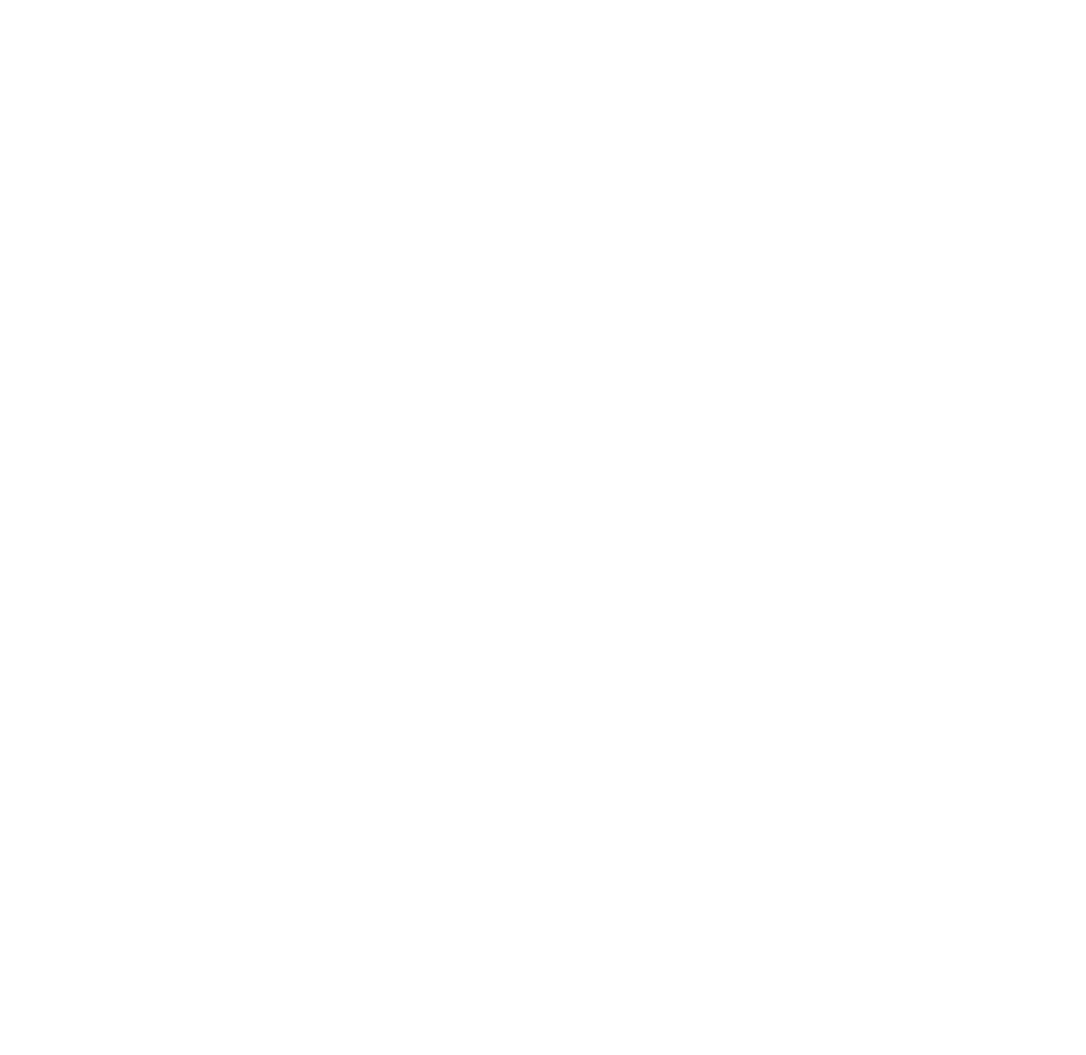}}
			{\includegraphics[width=.12\linewidth]
		  		{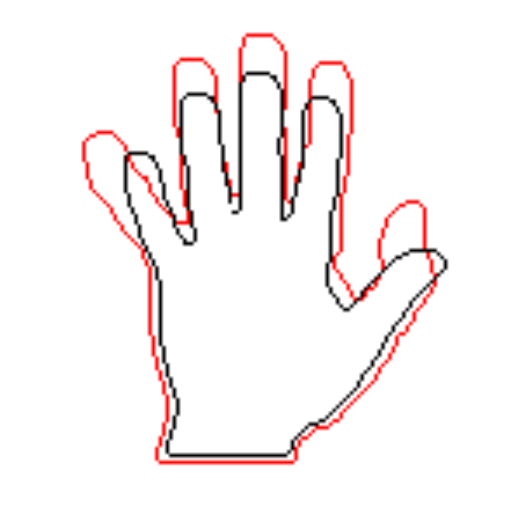}}
		  {\includegraphics[width=.12\linewidth]
		  		{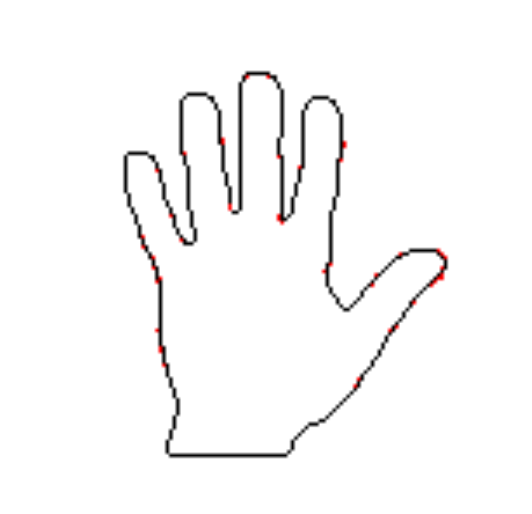}}
		  {\includegraphics[width=.12\linewidth]
		  		{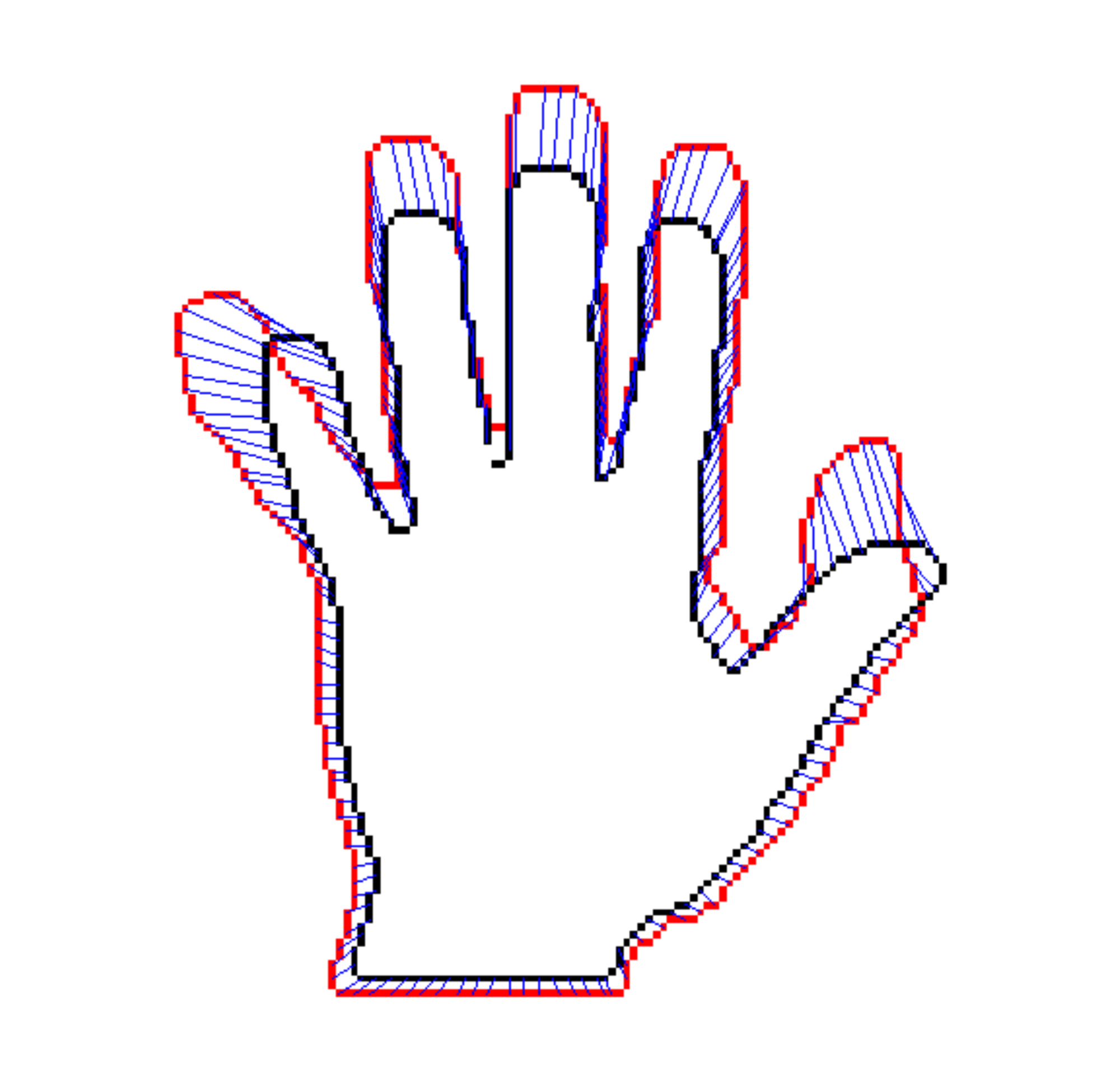}}
		  {\includegraphics[width=.12\linewidth]
		  		{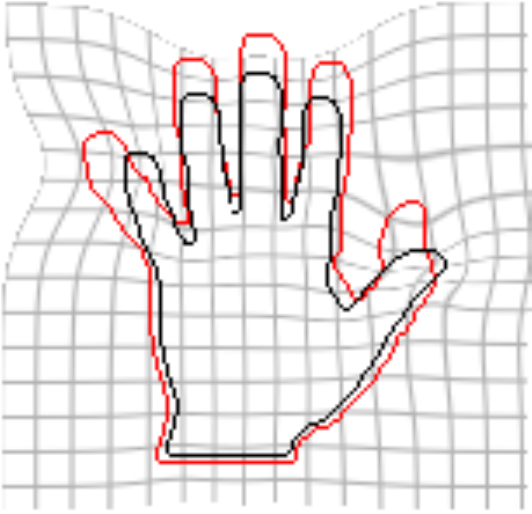}}\\
			{\includegraphics[width=.12\linewidth]
		  		{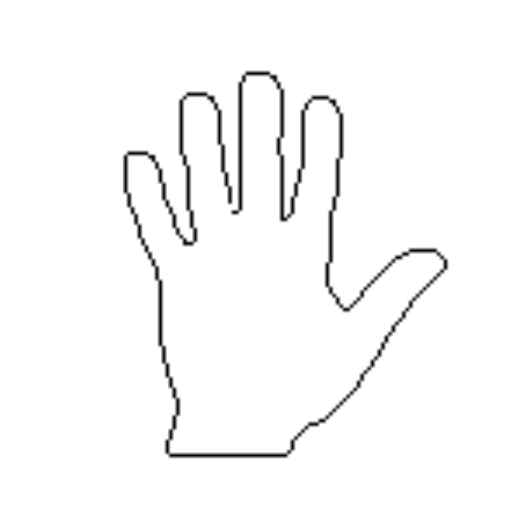}}
			{\includegraphics[width=.12\linewidth]
		  		{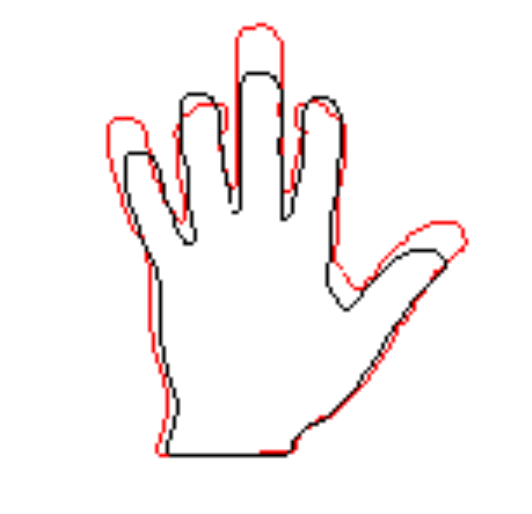}}
		  {\includegraphics[width=.12\linewidth]
		  		{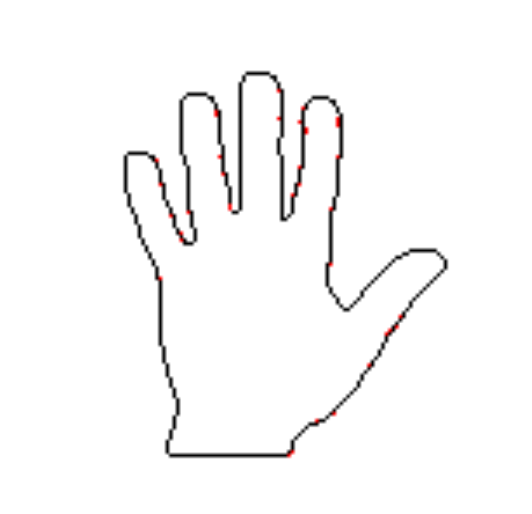}}
		  {\includegraphics[width=.12\linewidth]
		  		{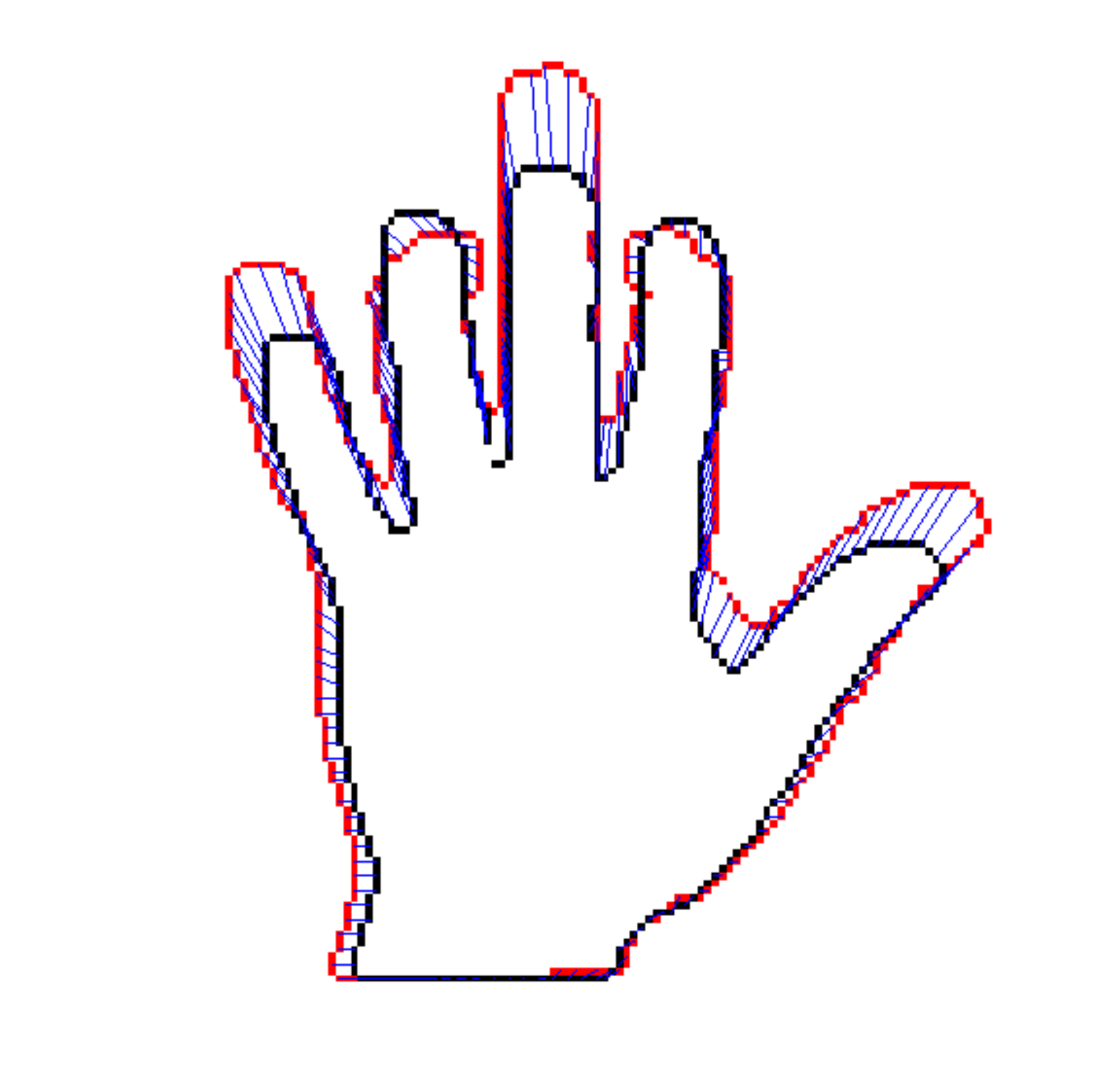}}
		  {\includegraphics[width=.12\linewidth]
		  		{./figs/sihouette/handbent1_to_hand/grid_warp.pdf}}\\
			{\includegraphics[width=.12\linewidth]
		  		{./figs/sihouette/blank.pdf}}
			{\includegraphics[width=.12\linewidth]
		  		{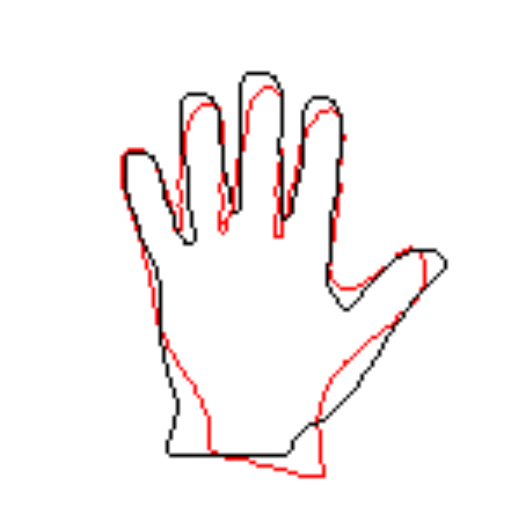}}
		  {\includegraphics[width=.12\linewidth]
		  		{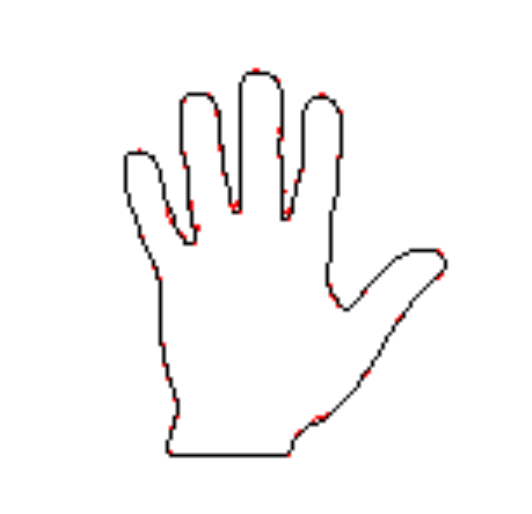}}
		  {\includegraphics[width=.12\linewidth]
		  		{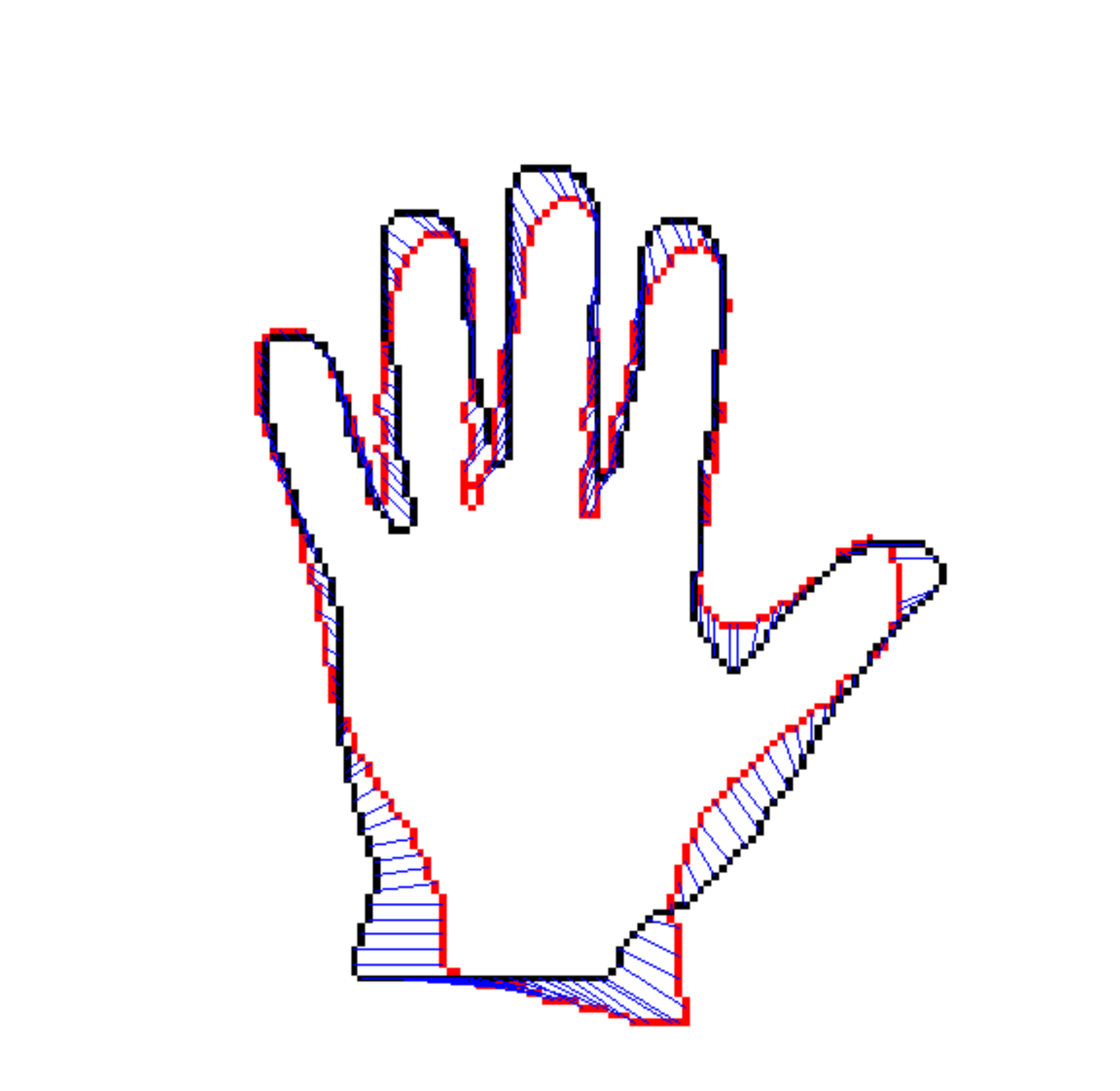}}
		  {\includegraphics[width=.12\linewidth]
		  		{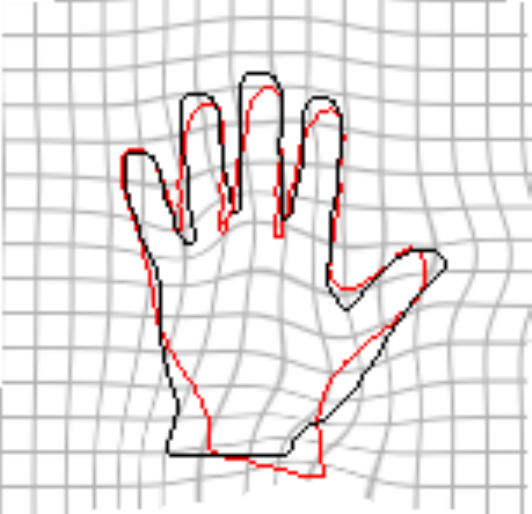}}\\
		  		{\includegraphics[width=.12\linewidth]
		  		{./figs/sihouette/blank.pdf}}
			{\includegraphics[width=.12\linewidth]
		  		{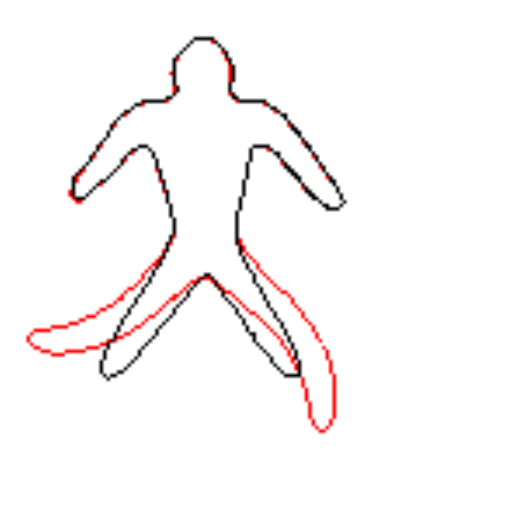}}
		  {\includegraphics[width=.12\linewidth]
		  		{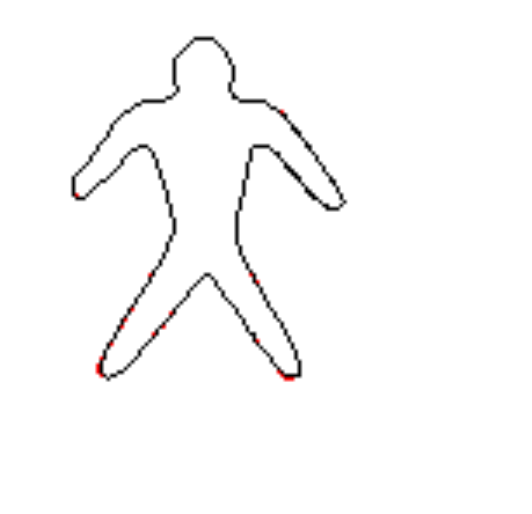}}
		  {\includegraphics[width=.12\linewidth]
		  		{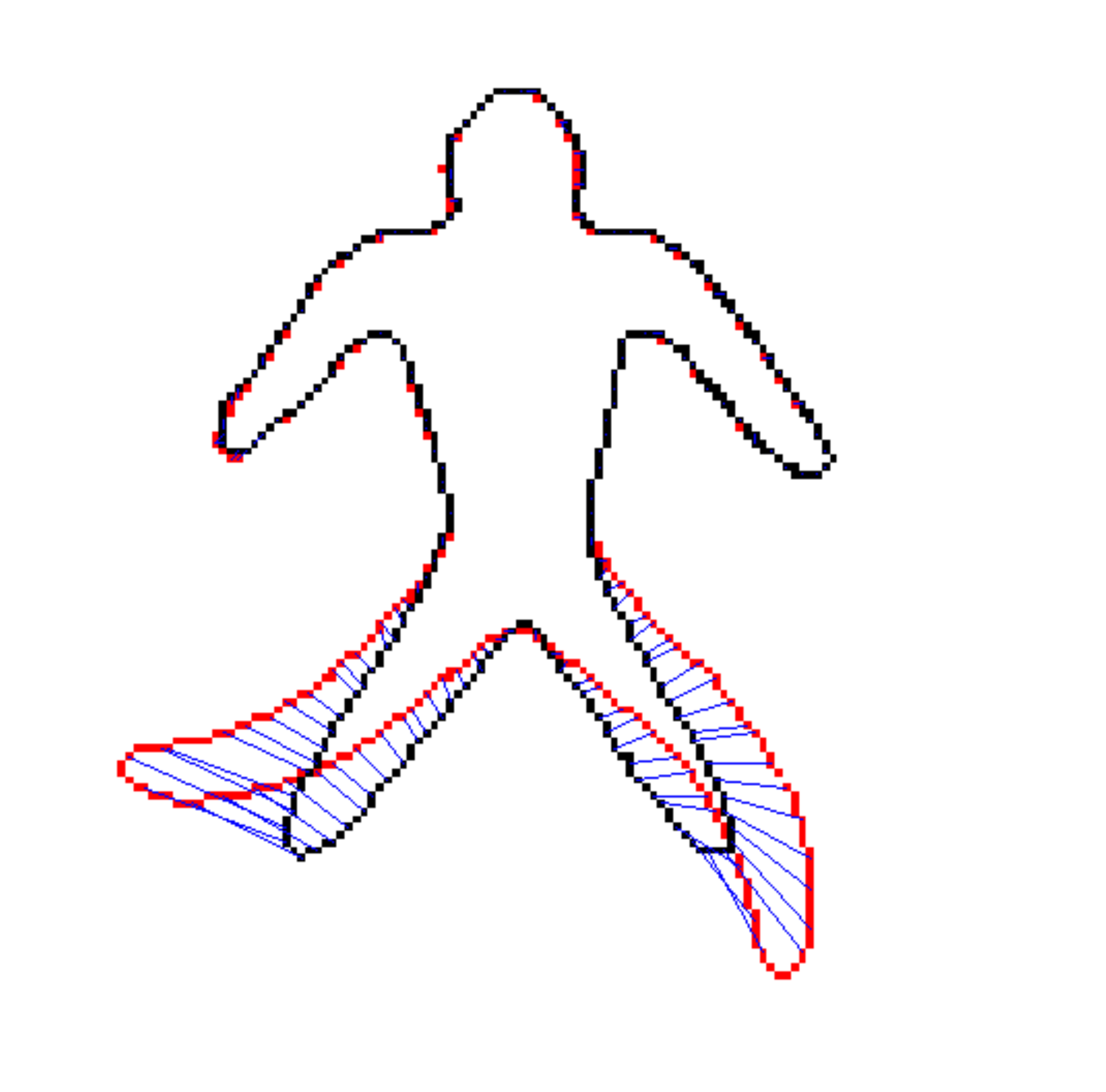}}
		  {\includegraphics[width=.12\linewidth]
		  		{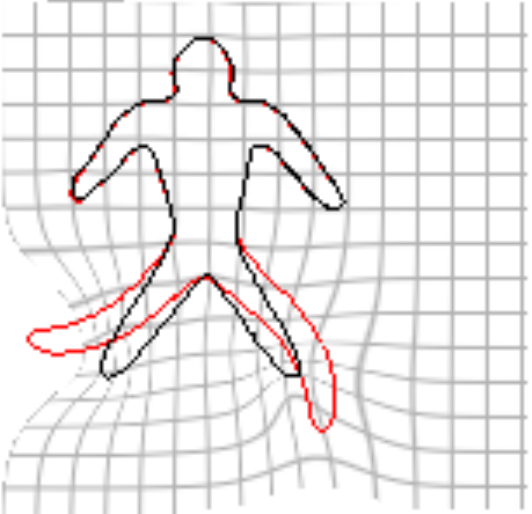}}\\
			{\includegraphics[width=.12\linewidth]
		  		{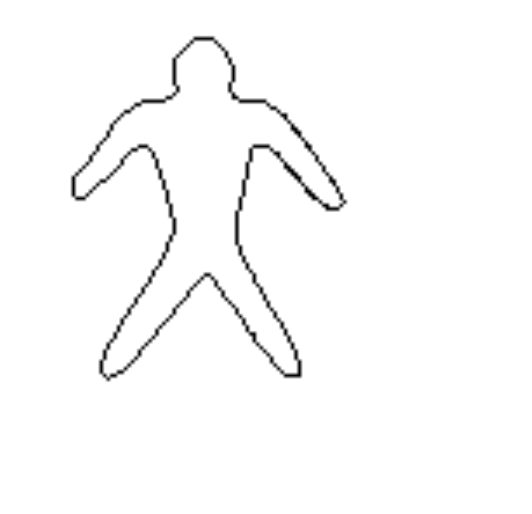}}
			{\includegraphics[width=.12\linewidth]
		  		{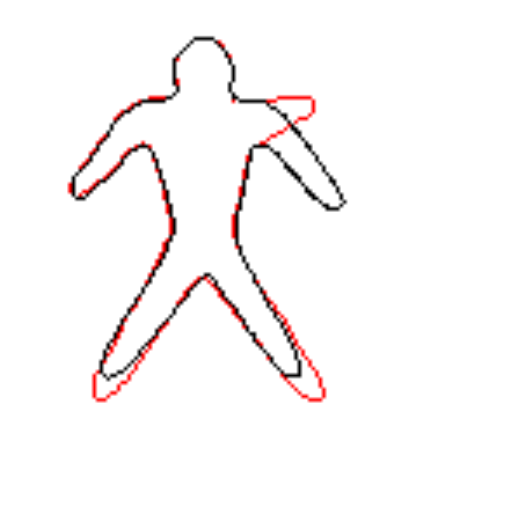}}
		  {\includegraphics[width=.12\linewidth]
		  		{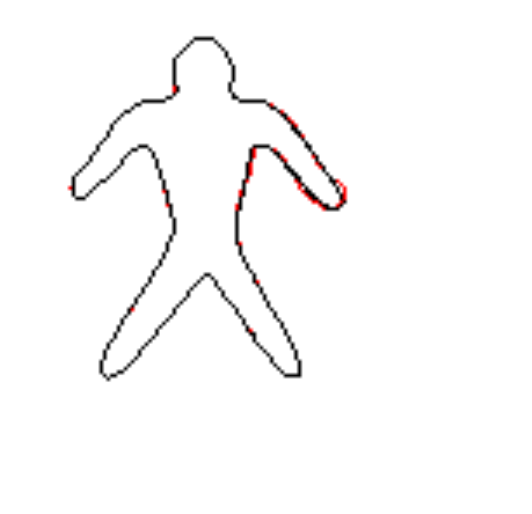}}
		  {\includegraphics[width=.12\linewidth]
		  		{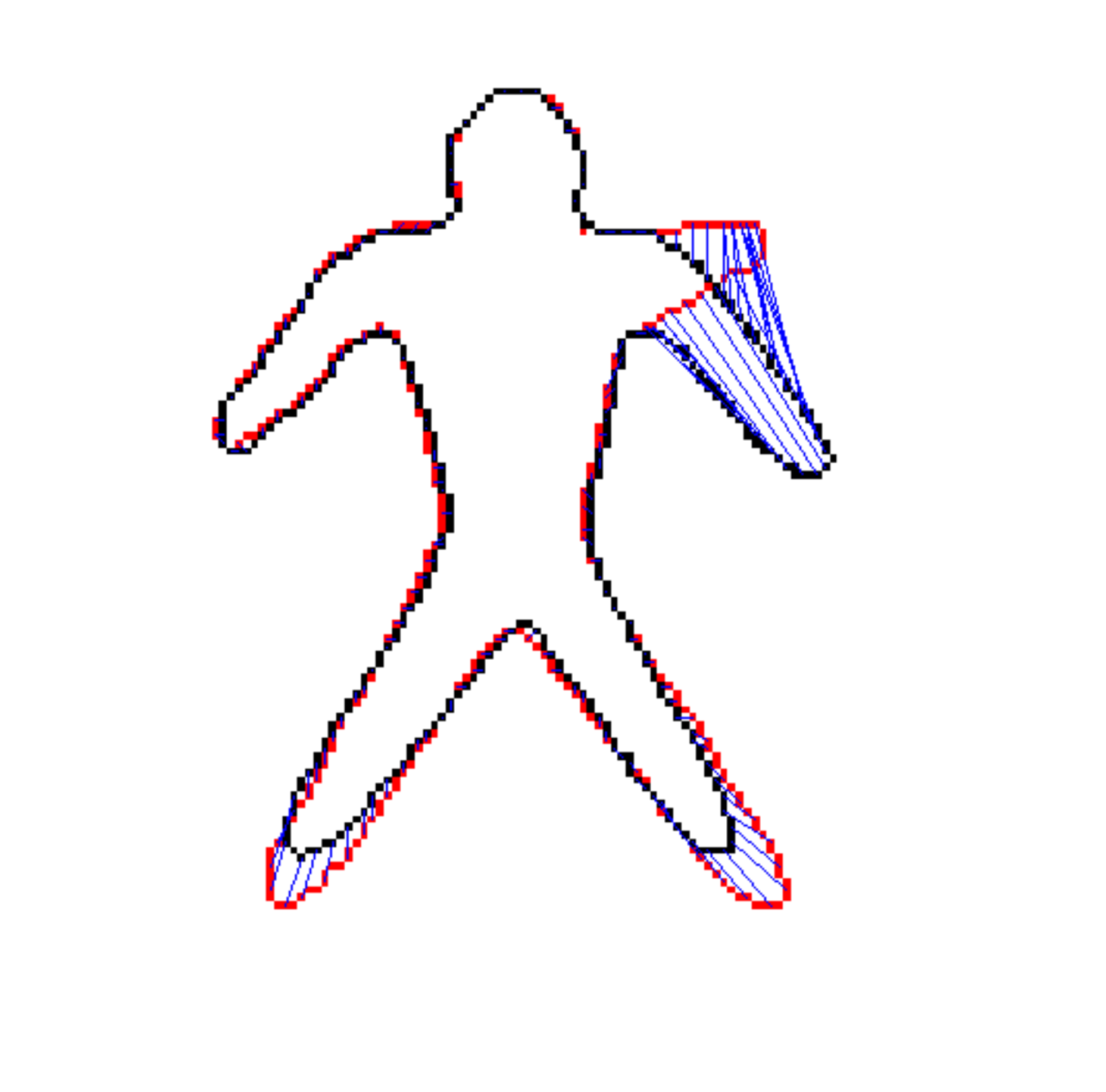}}
		  {\includegraphics[width=.12\linewidth]
		  		{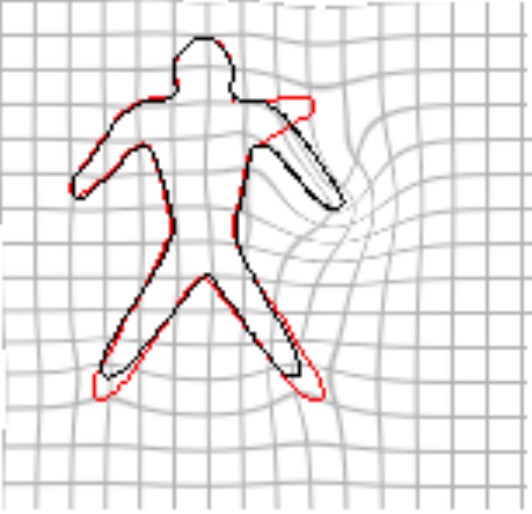}}\\
			{\includegraphics[width=.12\linewidth]
		  		{./figs/sihouette/blank.pdf}}
			{\includegraphics[width=.12\linewidth]
		  		{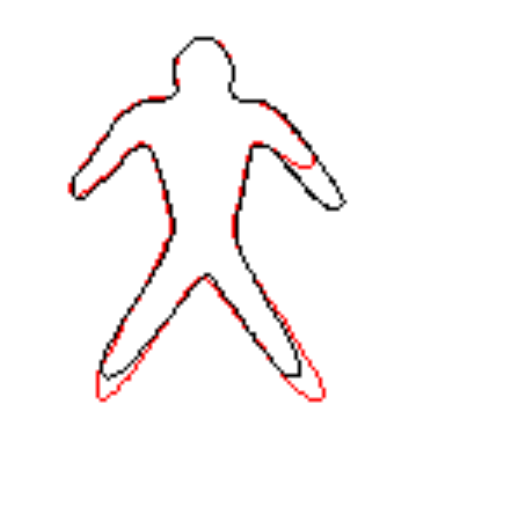}}
		  {\includegraphics[width=.12\linewidth]
		  		{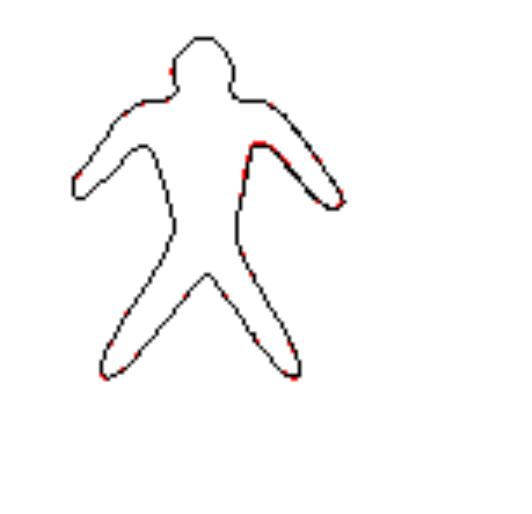}}
		  {\includegraphics[width=.12\linewidth]
		  		{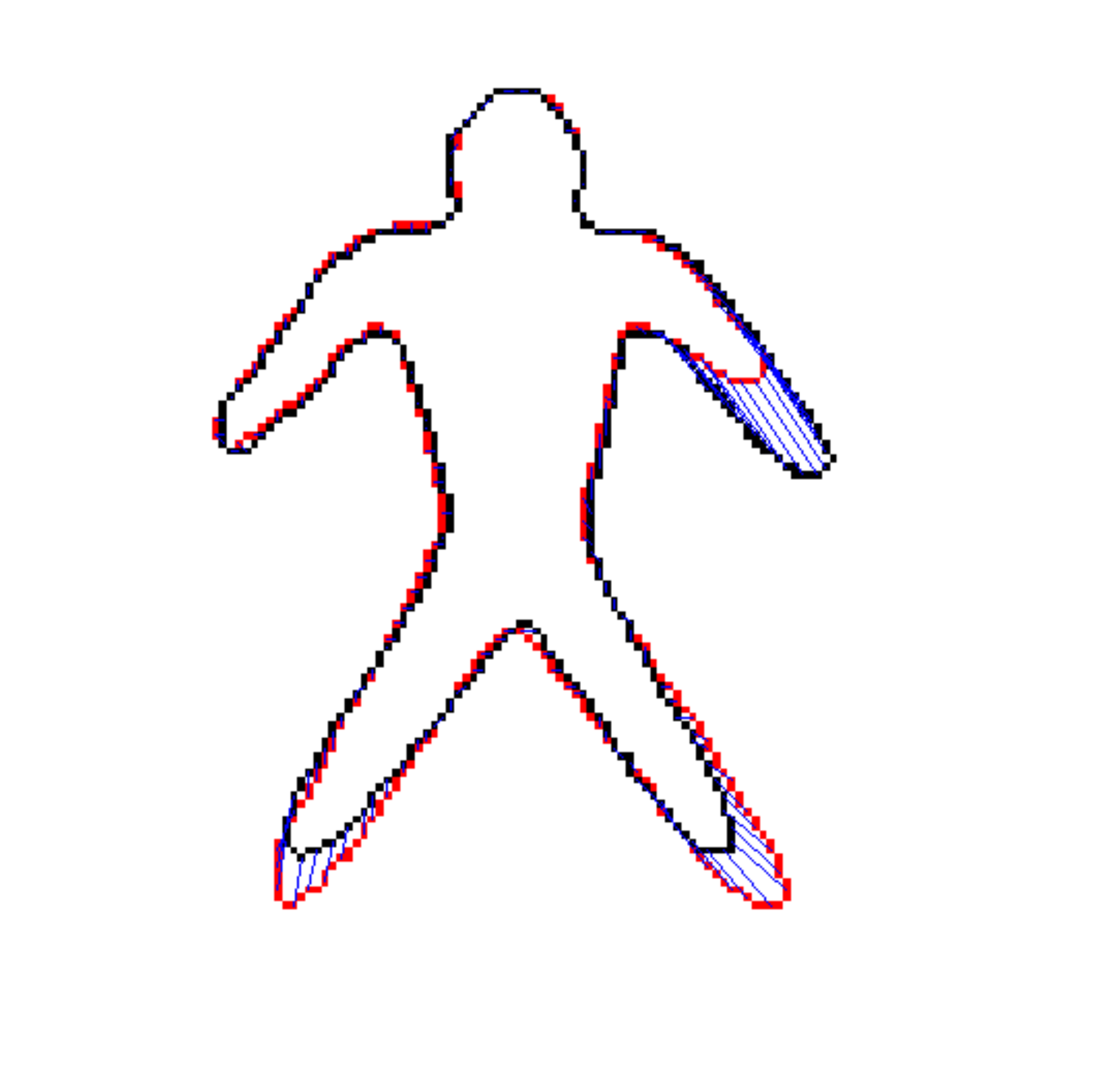}}
		  {\includegraphics[width=.12\linewidth]
		  		{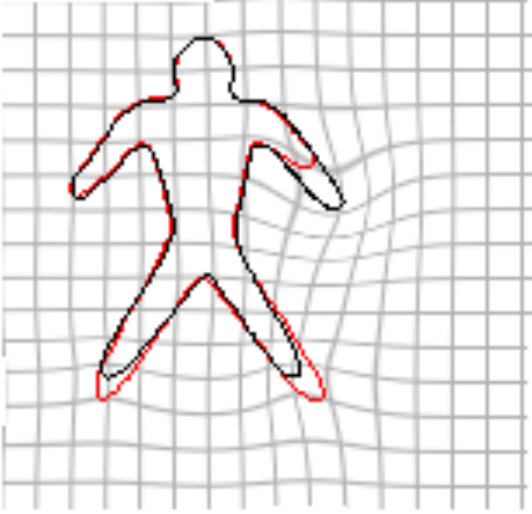}}\\
		  		{\includegraphics[width=.12\linewidth]
		  		{./figs/sihouette/blank.pdf}}
			{\includegraphics[width=.12\linewidth]
		  		{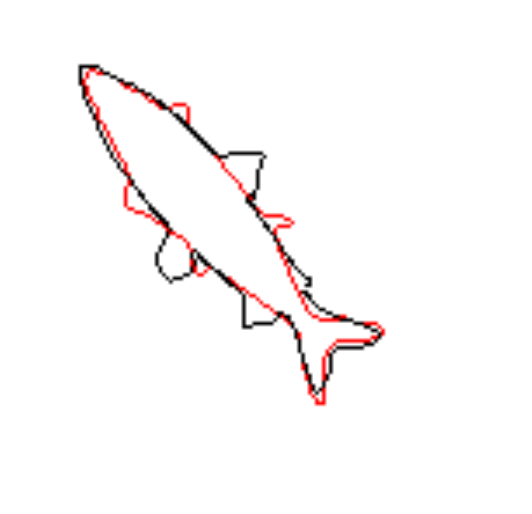}}
		  {\includegraphics[width=.12\linewidth]
		  		{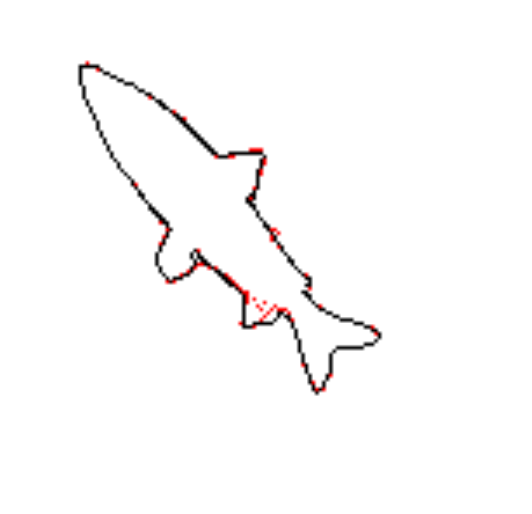}}
		  {\includegraphics[width=.12\linewidth]
		  		{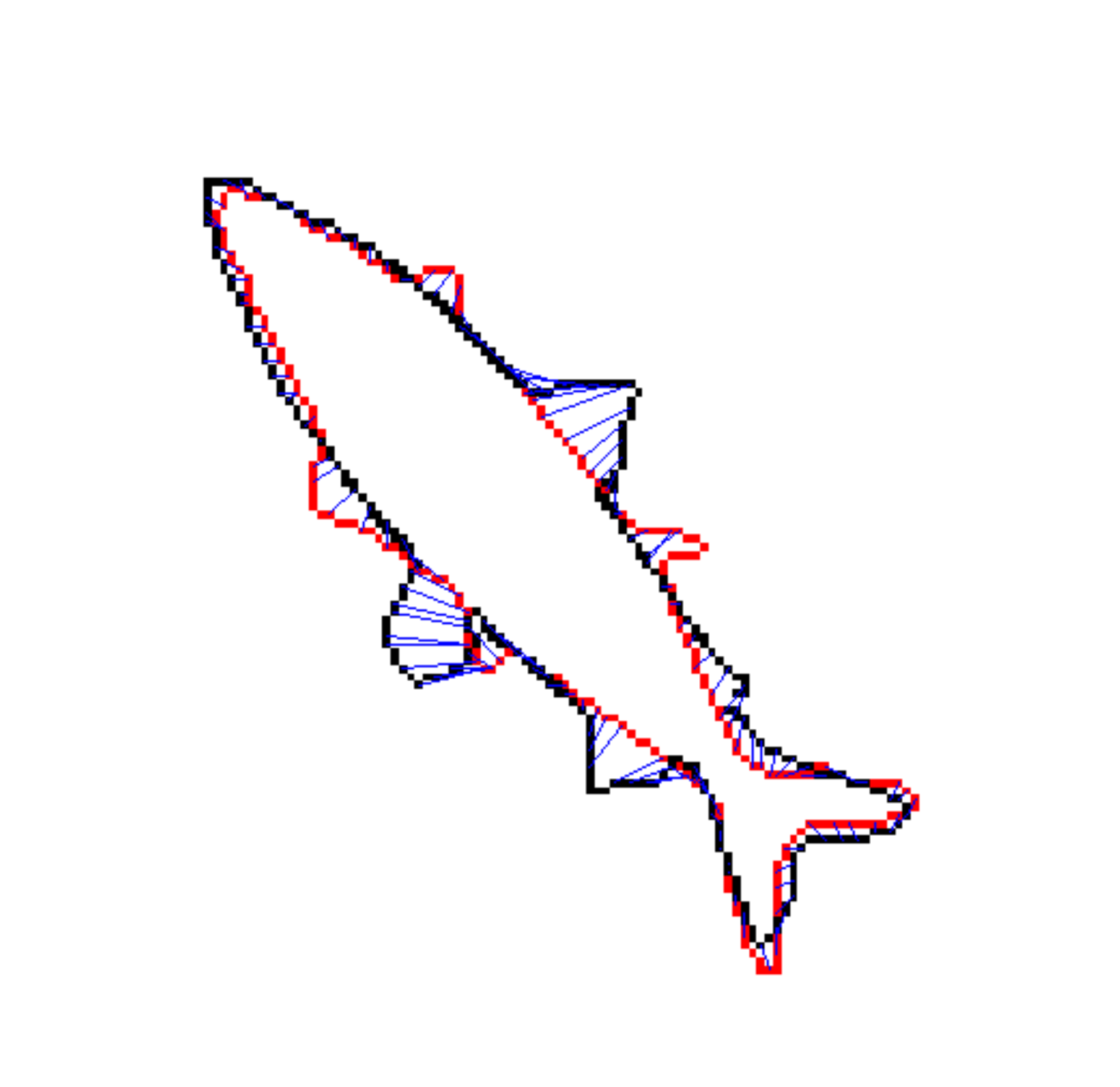}}
		  {\includegraphics[width=.12\linewidth]
		  		{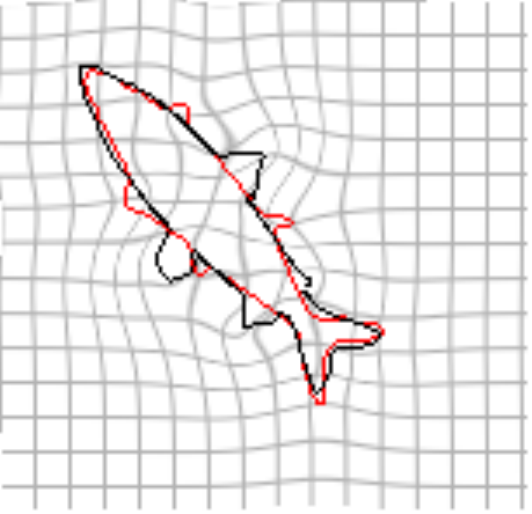}}\\
			{\includegraphics[width=.12\linewidth]
		  		{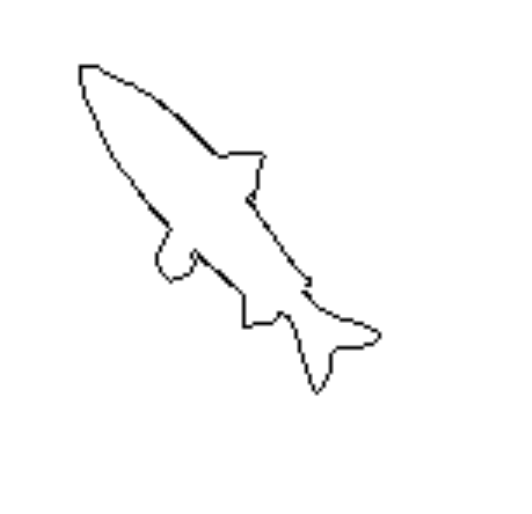}}
			{\includegraphics[width=.12\linewidth]
		  		{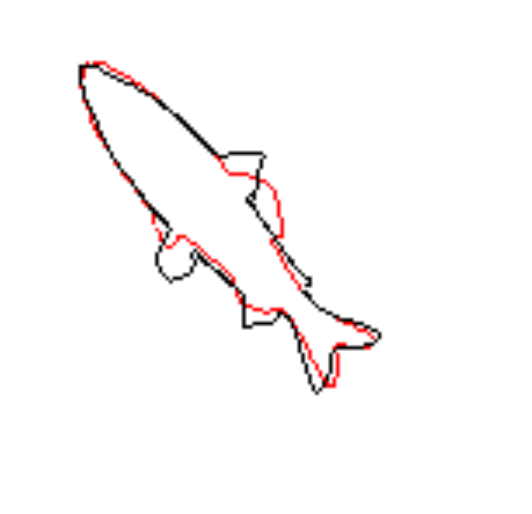}}
		  {\includegraphics[width=.12\linewidth]
		  		{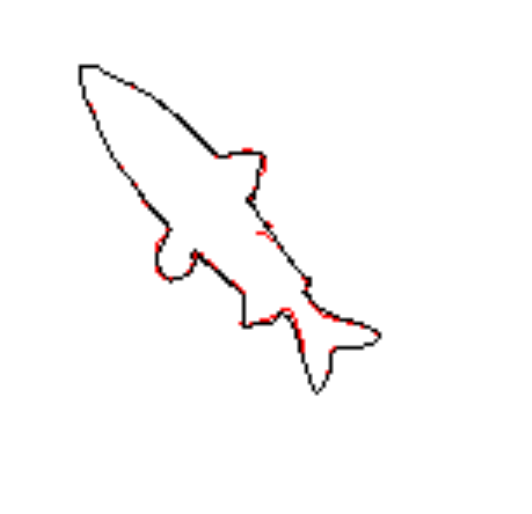}}
		  {\includegraphics[width=.12\linewidth]
		  		{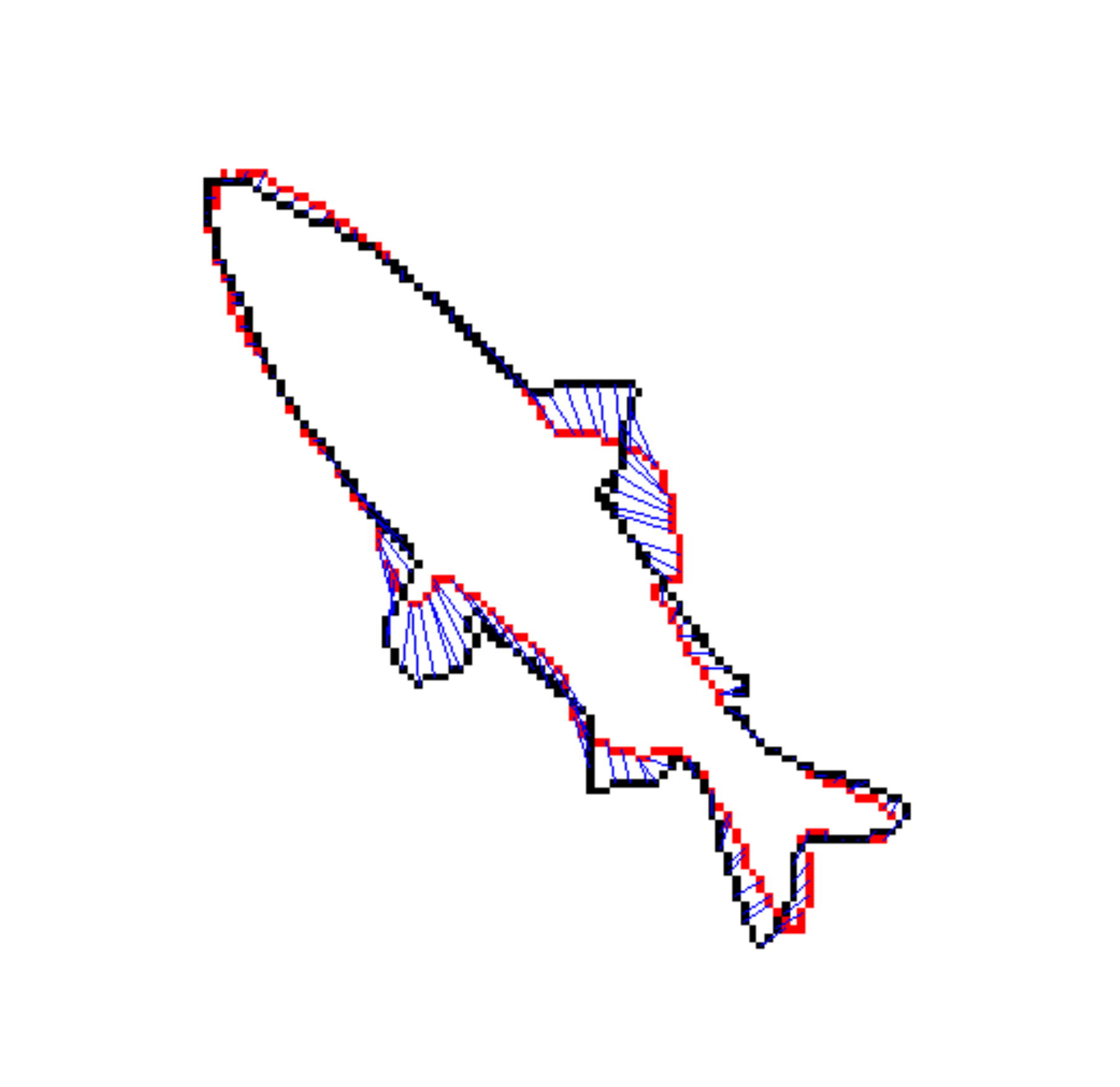}}
		  {\includegraphics[width=.12\linewidth]
		  		{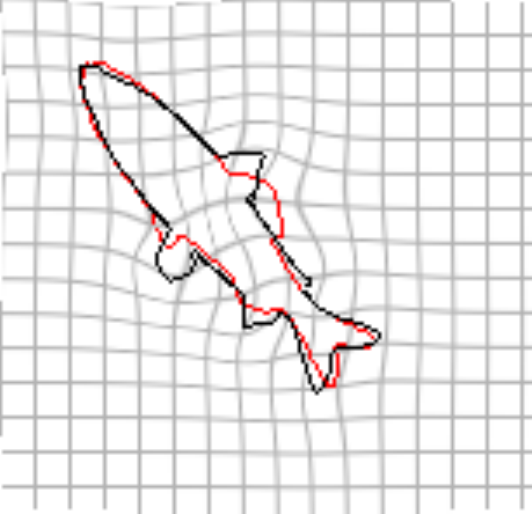}}\\
		  		\subfigure[]
			{\includegraphics[width=.12\linewidth]
		  		{./figs/sihouette/blank.pdf}}
		  		\subfigure[]
			{\includegraphics[width=.12\linewidth]
		  		{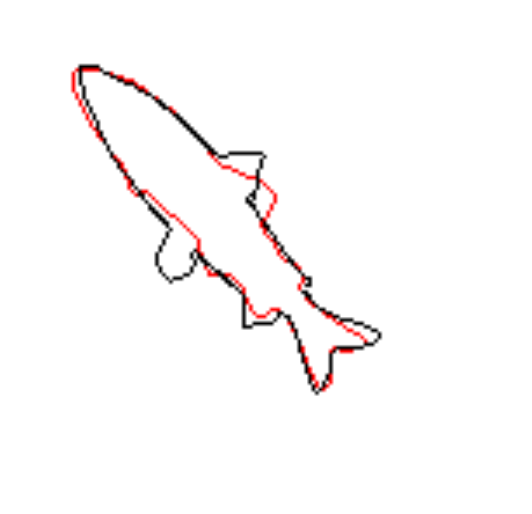}}
		  		\subfigure[]
		  {\includegraphics[width=.12\linewidth]
		  		{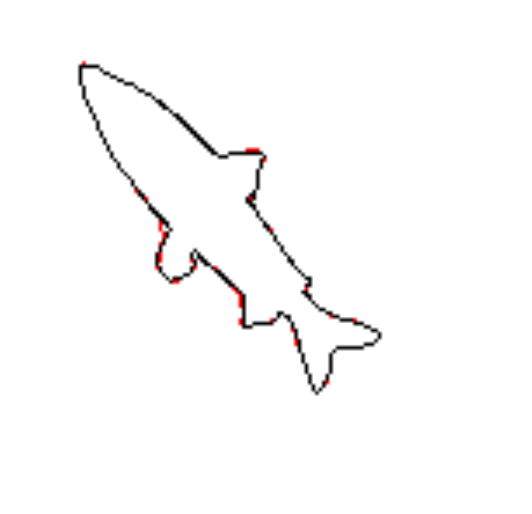}}
		  		\subfigure[]
		  {\includegraphics[width=.12\linewidth]
		  		{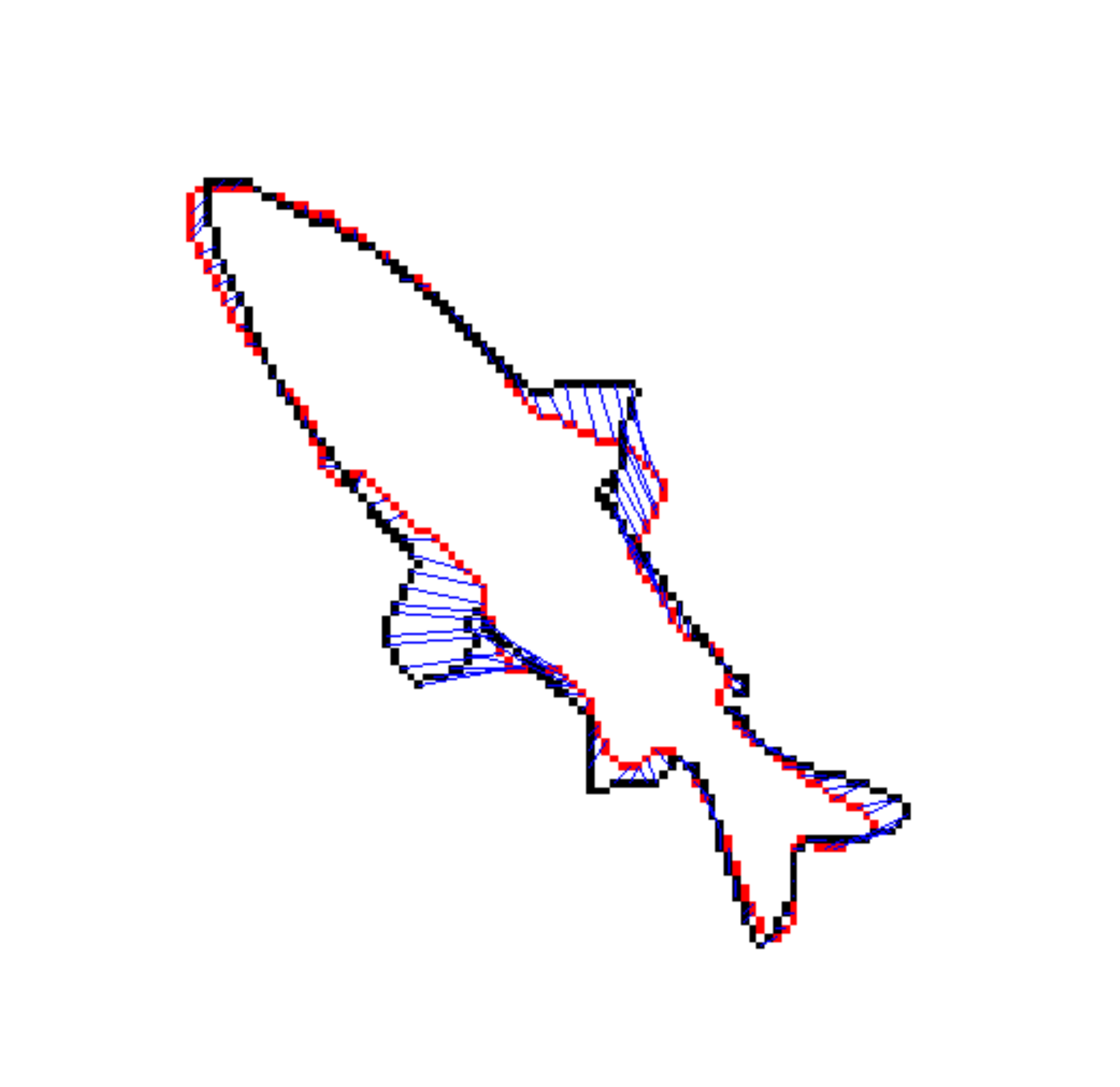}}
		  		\subfigure[]
		  {\includegraphics[width=.12\linewidth]
		  		{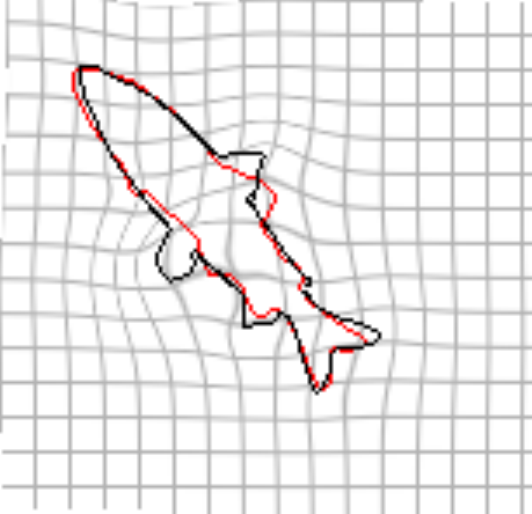}}
	\label{fig:brown_wei}
	\end{center}
	\caption{Brown university shape dataset. (a) Target images. (b) Overlaid target (in black) and source images (in red) before registration. (c) After registration. (d) Correspondence between target and source images. (e) Deformation fields as distorted grids.}
	\label{fig:brown_dataset}
\end{figure}

\subsection{Brown University shape dataset}
Our method requires some parameters to be set. First, we empirically set the consistency regularizer's weighting factor $\lambda=0.001$. Secondly, the representation power of our meshless model depends on both the density and distribution the image domain partitions. To simplify, we began with regularly distributed patches with radius $r=20$ (pixels) and the spaces between patches $d=6$.  We normalized the images of Brown University shape dataset to sizes $150\times 150$ pixels. Finally, we globally aligned the shapes using the rigid registration method implemented in~\cite{kroon_isbi2009}.  Figure~\ref{fig:brown_dataset} shows registration results produced using our method. For a quantitative evaluation, we used the average mutual distance between the registered shape contours~\cite{chen_hui_2006}. 
For qualitative comparison, we also visualize the results in a similar manner as in~\cite{huang2006shape,paragios2003non,chen_hui_2006}. As in~\cite{chen_hui_2006}, we selected three different shapes (person\footnote{Named dude in the original dataset.}, hand, and fish). The average pixel distances after local registration for person, fish, and hand were 0.14, 0.24, and 0.08, respectively. This result was better than the one reported in~\cite{chen_hui_2006} (0.59, 0.45 and 0.61, respectively). Additionally, for most cases, the maximum pixel distance was less than 2 pixels showing that registration quality was consistent along contours. We believe that the use of both the chamfer-matching functional and meshless model contributed to the improved accuracy. On one hand, chamfer-matching functional is a more direct measurement of shape contour's alignment than shape-context descriptors. On the other hand, the partition-of-unity model is numerically more accurate than thin-plate splines that use radial-basis functions.

\begin{figure}[t!!!]
	\begin{center}
			\subfigure[Huang~\etal]
			{\includegraphics[width=.42\linewidth]
		  		{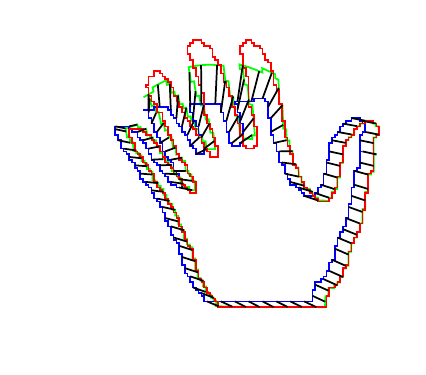}}
		  \subfigure[Ours]
		  {\includegraphics[width=.42\linewidth]
		  		{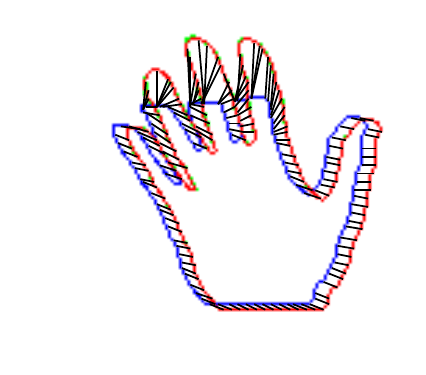}}
			\\
			\subfigure[Before]
			  {\includegraphics[width=.23\linewidth]
		  		{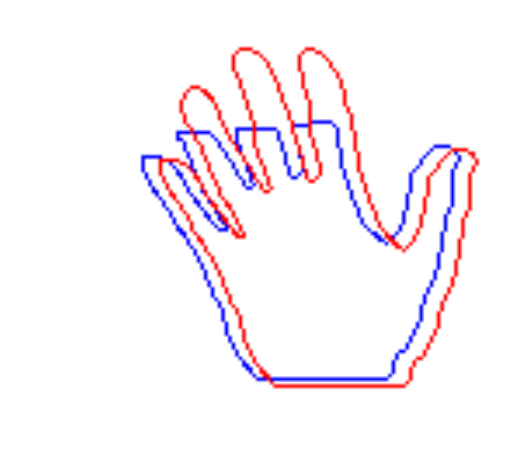}}
			\subfigure[After]
			  {\includegraphics[width=.23\linewidth]
		  		{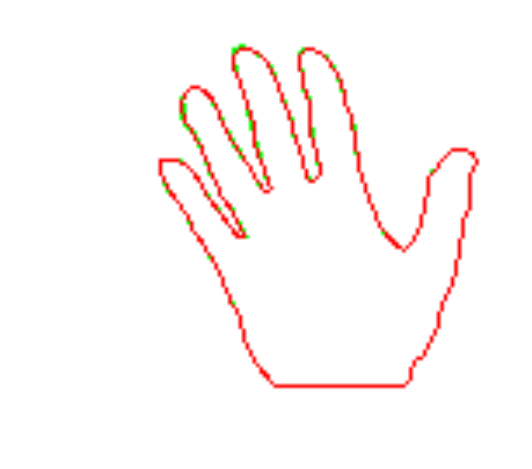}}
			\subfigure[Huang~\etal]
			  {
			  	\label{subfig:huang_hand_grid}
			  	\includegraphics[width=.23\linewidth]
		  		{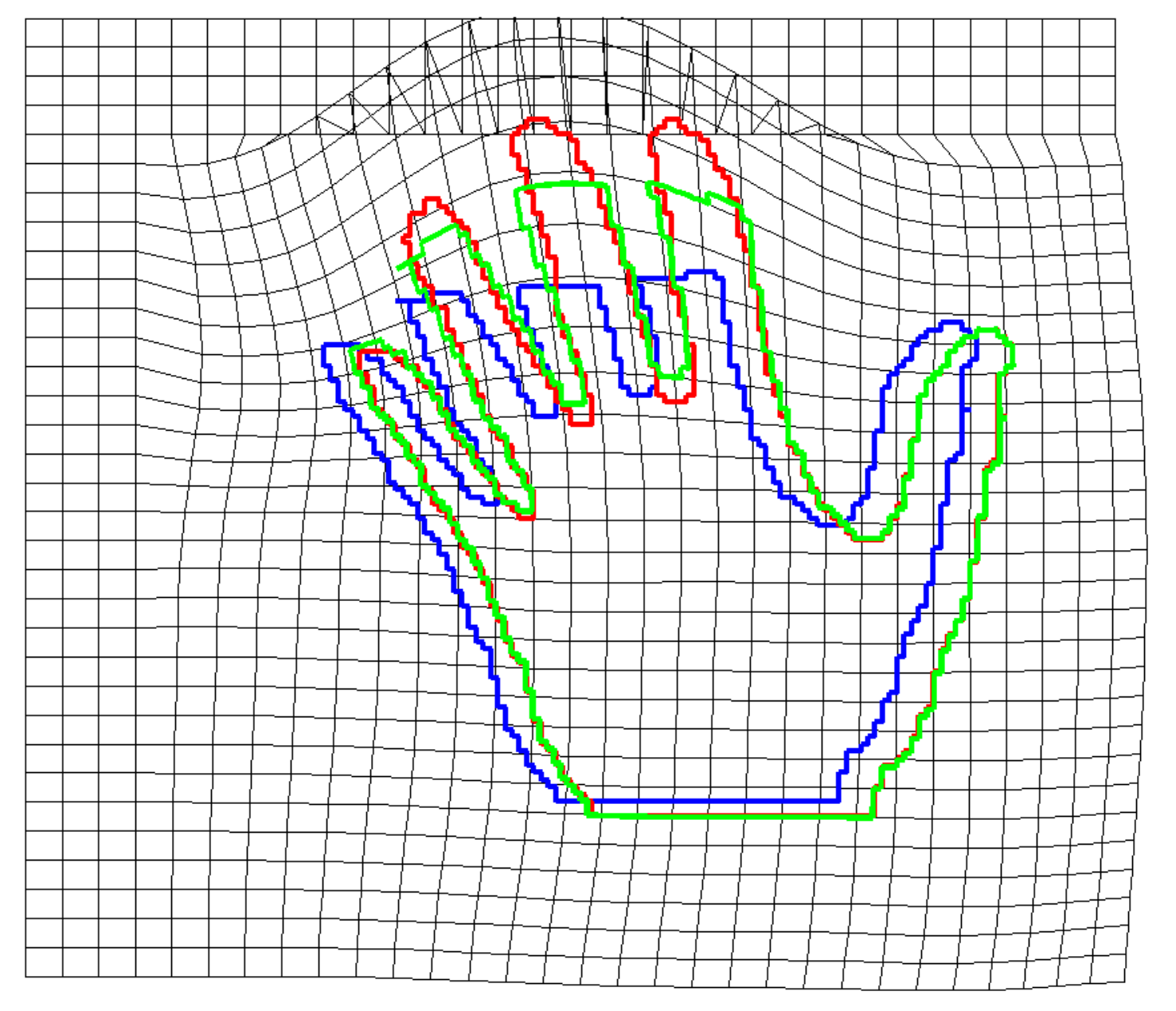}}
			\subfigure[Our method]
			  {
			  	\label{subfig:wei_hand_grid}
			  	\includegraphics[width=.23\linewidth]
		  		{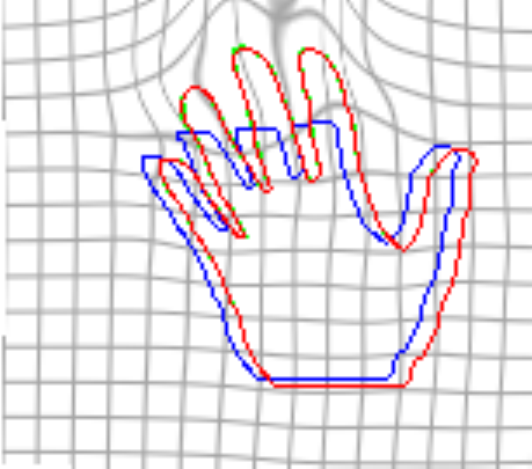}}
	\end{center}		
	\caption{Occluded shapes (hand). The contours are source shape (blue), target shape (red), and deformed source shape (green). a) Huang~\etal~\cite{huang2006shape}, average distance: 0.64, max: 9.2, variance: 1.99 . b) Our method, average distance: 0.15, max: 1.41, variance: 0.13. c) Deformed source (green) and target (red) shapes. d) The deformation field represented using a deformed grid image.}
	\label{fig:huang_hand}
\end{figure}

\begin{figure}[htb]
	\begin{center}
		 \subfigure[$\delta=10$]
			{
					\label{subfig:huang_bunny_narrow}
					\includegraphics[width=.30\linewidth]
		  		{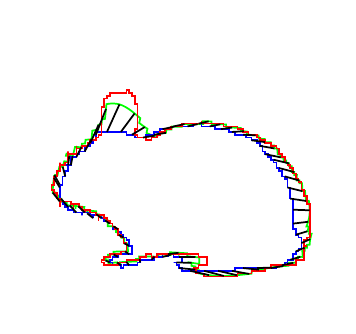}
		  }
		 \subfigure[$\delta=50$]
			{
					\label{subfig:huang_bunny_wide}
					\includegraphics[width=.30\linewidth]
		  		{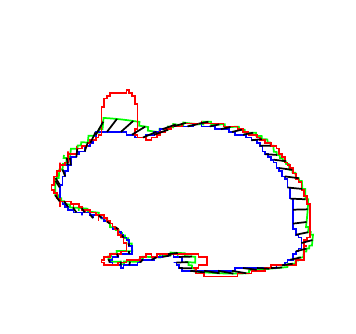}
		  }
		  \subfigure[Ours]
		  {		
		  		\label{subfig:wei_bunny}
		  		\includegraphics[width=.30\linewidth]
		  		{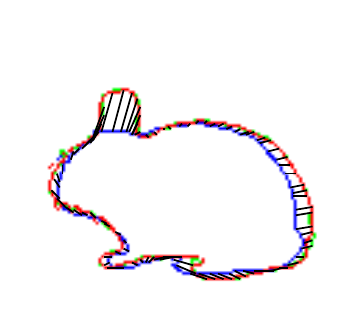}}
	\end{center}
	\caption{Occluded shapes (rabbit). a) Huang~\etal~\cite{huang2006shape}, average distance: 0.62, maximum:5.38, variance: 0.89. b) Increasing the narrow-band width futher reduces the registration accuracy, with average distance: 1.00, maximum: 10.44, variance: 2.63. c) Our method, average distance: 0.41, maximum: 2.23, variance: 0.28.}
	\label{fig:huang_bunny}
\end{figure}

\subsubsection{Shapes with occlusion and large deformation}
Our method was able to register shapes with occlusion, large deformation, and high curvature. Here, we compared our method with Huang's method~\cite{huang2006shape}. Figure~\ref{fig:huang_hand} shows the registration of occluded hand shapes. In~\cite{huang2006shape}, landmarks were placed at occluded parts to guide the registration process. Without landmarks, it is difficult for splines to handle large deformation from the source (blue) to the target (red) shape due to foldings in the mesh model (Figure~\ref{subfig:huang_hand_grid}). In contrast, our meshless model alleviated the folding effect (Figure~\ref{subfig:wei_hand_grid}). Here, we plotted the deformation field by deforming an image of grid, but it does not mean that our model is restricted to such a  topology. Quantitatively, our method still achieved remarkable accuracy with the mean, maximum, and variance of edge distances as 0.15, 1.41 and 0.13, respectively, in comparison with Huang's result (mean: 0.64, max: 9.2, variance: 1.99). 

The mesh model's folding effect may be further complicated by the narrow-band function $N_\delta$. Figure~\ref{fig:huang_bunny} shows registration of occluded shapes (bunny) using increasing narrow-band $\delta$. On one hand, we observed that increasing $\delta$ will generally reduce the ability of Huang's method to handle high-curvature regions (Figure~\ref{subfig:huang_bunny_narrow} and~\ref{subfig:huang_bunny_wide}) as using large $\delta$ may include embedding space that undergoes high-curvature bending (see Figure~\ref{fig:bending_curve} for demonstration), while our method was not affected at all by such a parameter (Figure~\ref{subfig:wei_bunny}). On the other hand, with decreasing $\delta$, Huang's method became prone to local minima. Figure~\ref{fig:huang_dude} shows an example of shapes with both high curvature and large deformations. Here, we tested Huang's method using different narrow-band width $\delta$. When $\delta$ was small (Figure~\ref{subfig:huang_dude_narrow}), the algorithm converged prematurely, unable to extend the deformation large enough. When $\delta$ was large (Figure~\ref{subfig:huang_dude_wide}), the method was unable to cope with high-curvature regions around the bending arm. An optimal result was obtained around $\delta=20$. However, even this optimal result was still less accurate when compared to our method.

\begin{figure}[t!!!]
	\begin{center}
		 \subfigure[$\delta=10$]
		  { 		\label{subfig:huang_dude_narrow}
		  		\includegraphics[width=.23\linewidth]
		  		{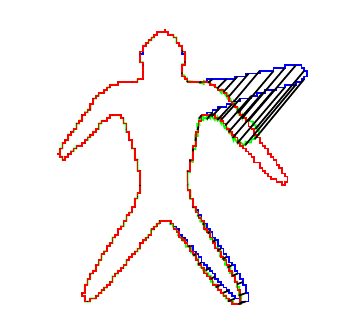}		  }
		 \subfigure[$\delta=50$]
		  { 		\label{subfig:huang_dude_wide}
		  		\includegraphics[width=.23\linewidth]
		  		{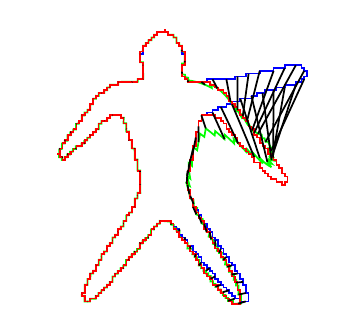}		  }
		  \subfigure[$\delta=20$]
		  {\includegraphics[width=.23\linewidth]
		  		{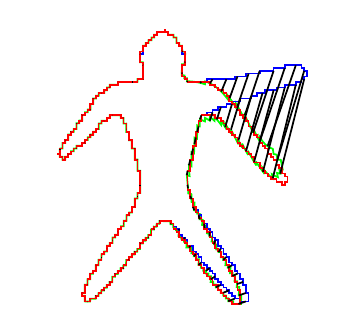}}
		\subfigure[Our method]
		  {\includegraphics[width=.23\linewidth]
		  		{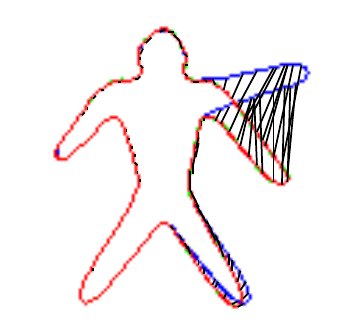}}
		\subfigure[Huang~\etal]
		  {\includegraphics[width=.23\linewidth]
		  		{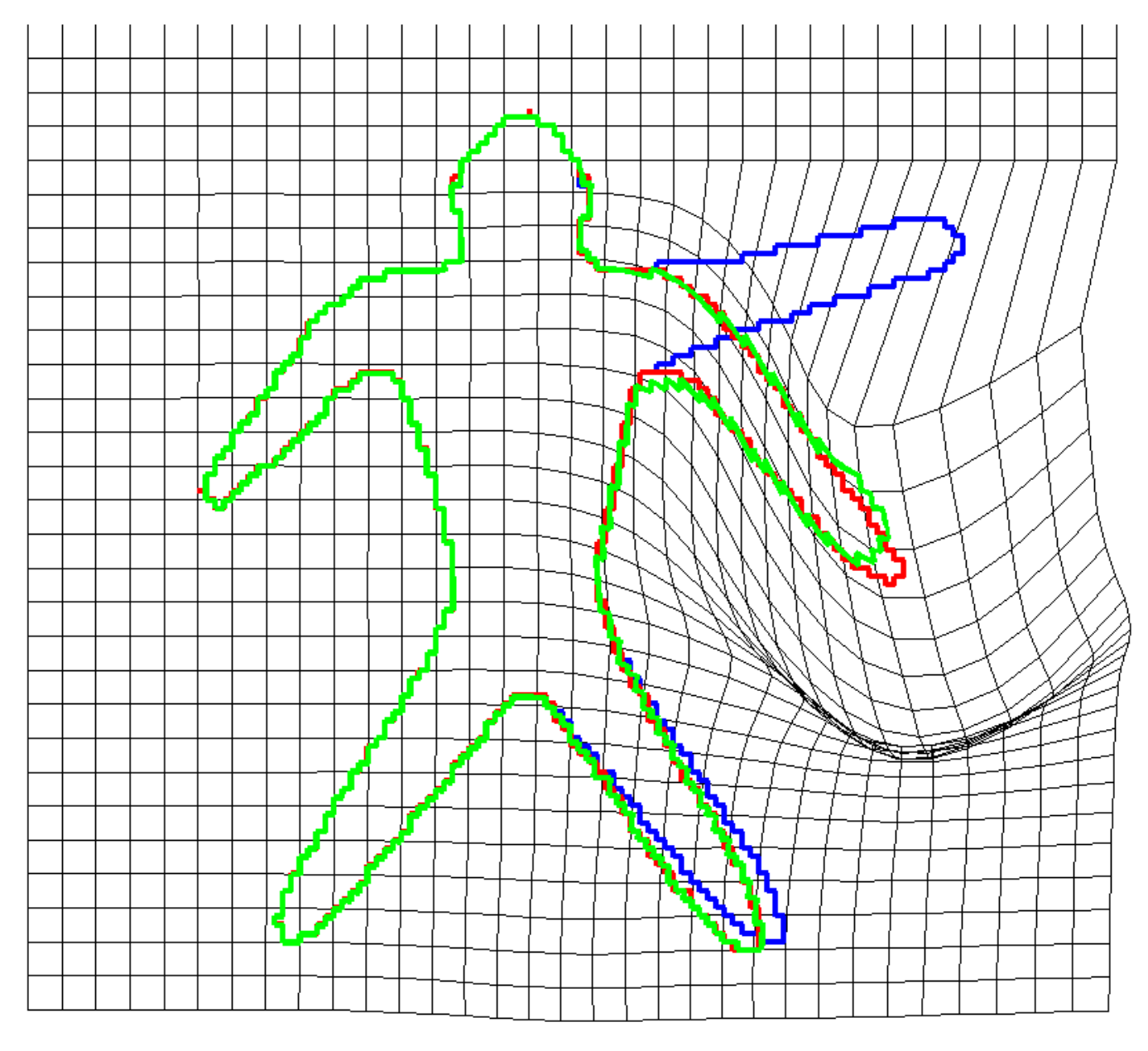}}
		\subfigure[Our method]
		  {\includegraphics[width=.23\linewidth]
		  		{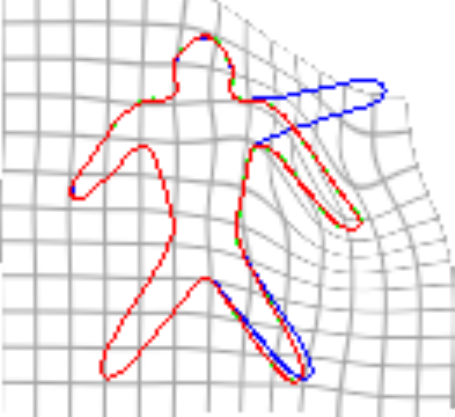}}
		  \subfigure[Huang~\etal]
		  {\includegraphics[width=.23\linewidth]
		  		{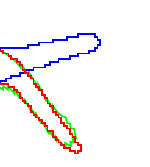}}
		 \subfigure[Our method]
		  {\includegraphics[width=.23\linewidth]
		  		{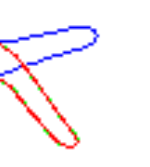}}
	\end{center}
	\caption{Large deformation (dude). a-c) Method of Huang~\etal, with varying narrow-band width. Best result was obtained when $\delta=20$, with average distance: 0.17, max: 3.61, variance: 0.20. d) Ours, average: 0.14, max: 1.00, variance: 0.12.}
	\label{fig:huang_dude}
\end{figure}

\subsection{Adapted partition-of-unity}

We now extend our comparison with Huang's method using heuristically adapted patches introduced in Section~\ref{sec:global_deformation} in order to reduce the computation domain to be around the shape contours. To be specific, we sampled along the shape contour at an equal geodesic distance $d$ (usually $5\leq d\leq20$), and placed a patch centered at each sample point with its scale set to be $r_p=\max\left(r_m,\kappa*\Pi_{\ShapeD(\bfx)}(\bfp)\right)$, where $r_m=\rho\, d$ was the minimum radius so that the patches overlapped with one another ($1.5 \leq \rho \leq 3$, and $\kappa=2$). In other words, the patches' scale increased proportionally if the two shapes were far from each other. Figure~\ref{fig:adaptive_hand}, Figure~\ref{fig:adaptive_bunny} and Figure~\ref{fig:adaptive_callosum} show examples of shape registration using adapted patches. In all these examples, the adapted-patch method achieved  results comparable to the ones using regular partitions and better results than Huang's method, while reducing the computational time by at least 50 percent. More reductions in computational cost are possible if we use finer partitions based on prior knowledge provided, for example, by statistical models of specific shapes.

\begin{figure}[t!!!!!]
	\begin{center}
			\begin{tabular}{ccccc}
				\parbox[b]{0.02\linewidth}{(a)\vspace{1.5\baselineskip}}
				&
		  	{\includegraphics[width=.21\linewidth]
		  		{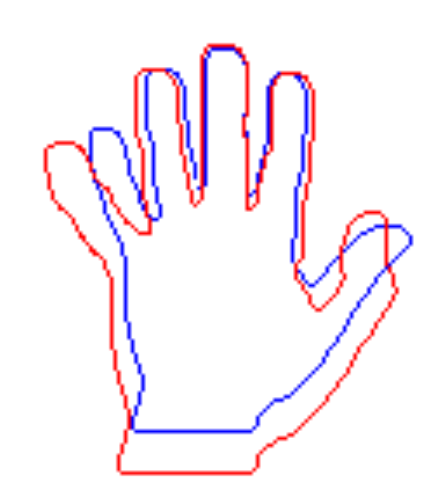}}&
		  	{\includegraphics[width=.21\linewidth]
		  		{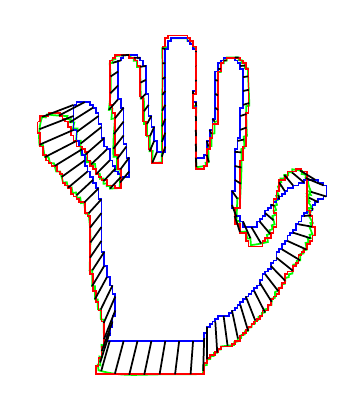}}&
		  	{\includegraphics[width=.21\linewidth]
		  		{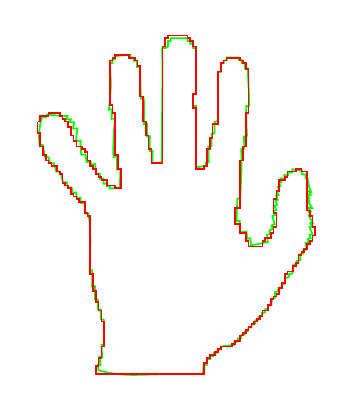}}&
		  	{\includegraphics[width=.21\linewidth]
		  		{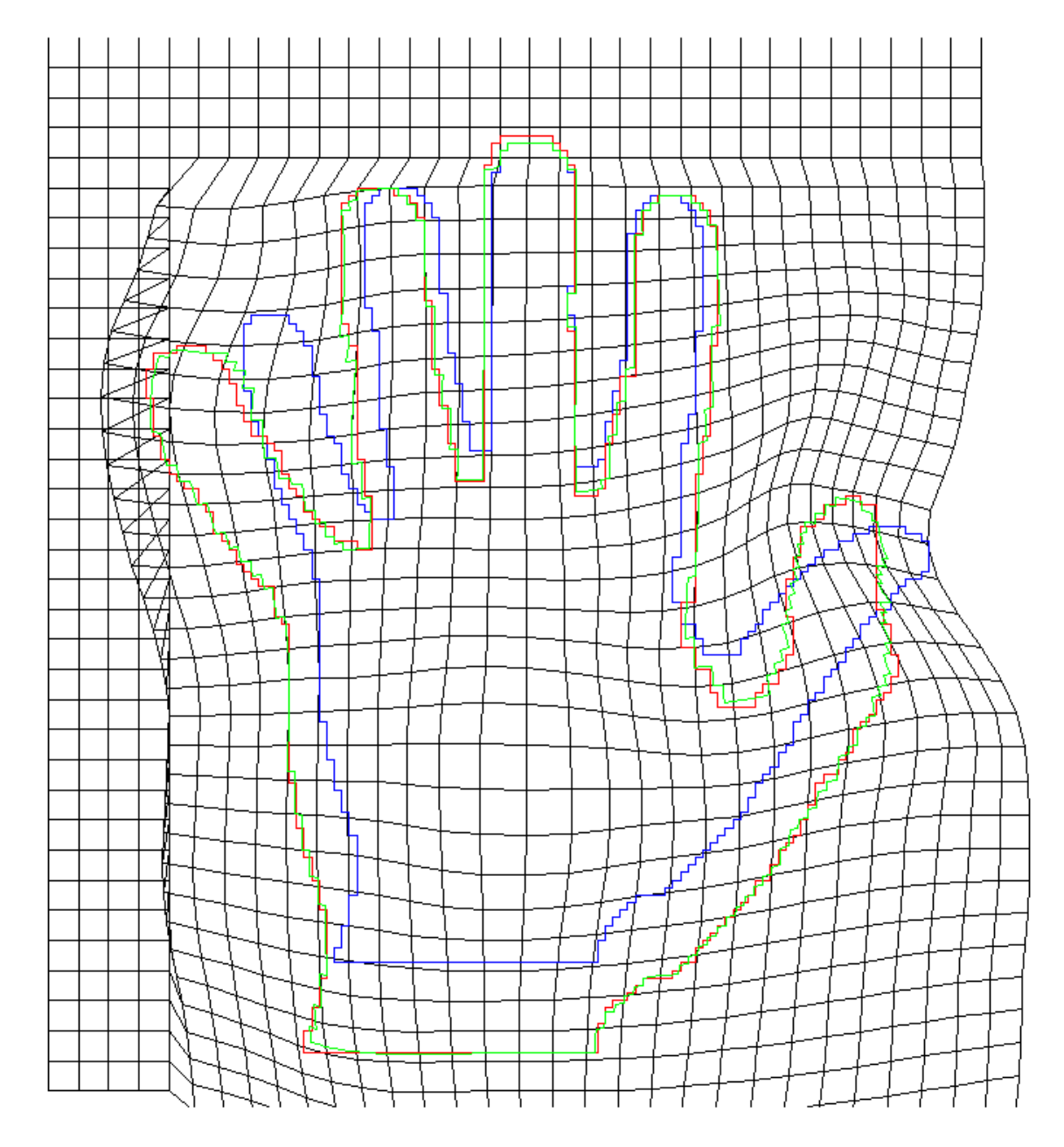}}
		  	\\
		  	\parbox[b]{0.02\linewidth}{(b)\vspace{1.5\baselineskip}}
				&
				{\includegraphics[width=.21\linewidth]
		  		{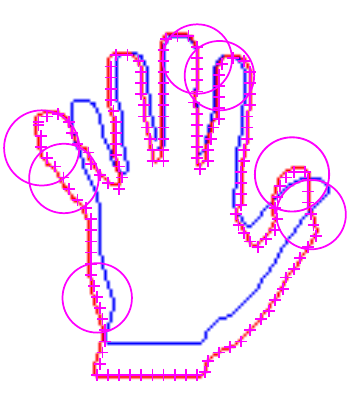}}&
				{\includegraphics[width=.21\linewidth]
		  		{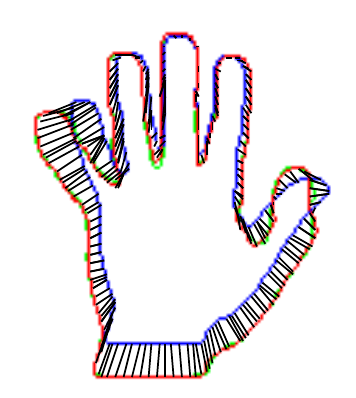}}&
				{\includegraphics[width=.21\linewidth]
		  		{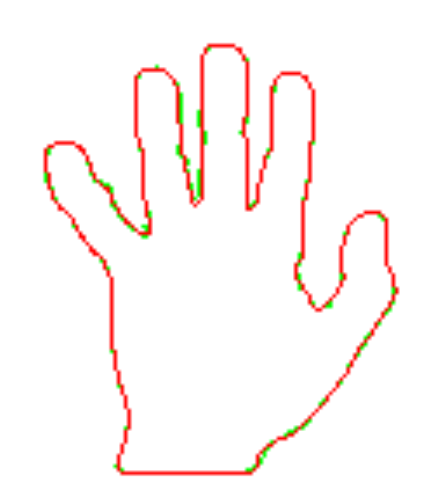}}&
		  	{\includegraphics[width=.21\linewidth]
		  		{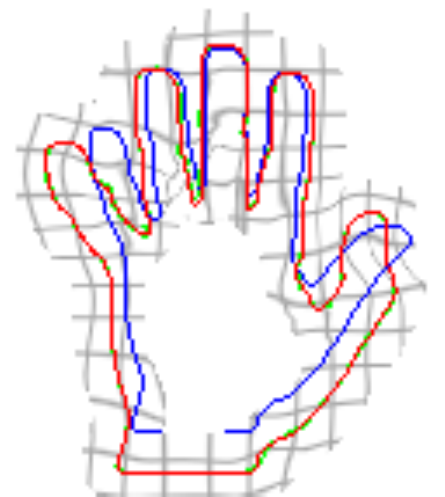}}
		  	\\
		  	\parbox[b]{0.02\linewidth}{(c)\vspace{1.5\baselineskip}}
				&
				{\includegraphics[width=.21\linewidth]
		  		{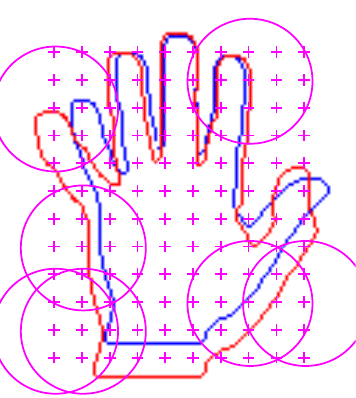}}&
				{\includegraphics[width=.21\linewidth]
		  		{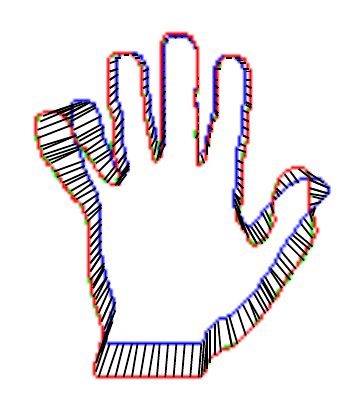}}&
				{\includegraphics[width=.21\linewidth]
		  		{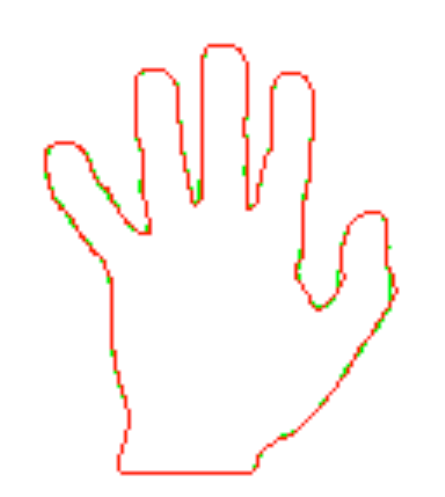}}&
		  	{\includegraphics[width=.21\linewidth]
		  		{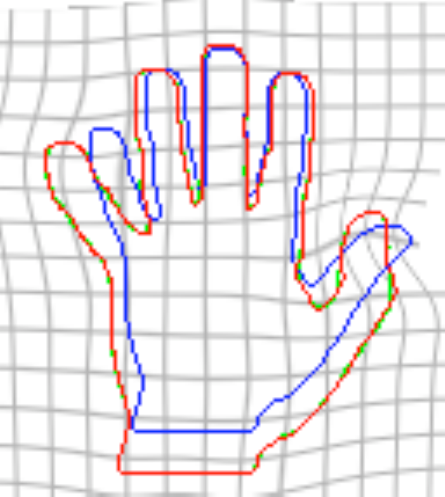}}
			\end{tabular}
	\end{center}
	\caption{Registration using adaptive patches (hand). a) Results of Huang's method (mean: 0.29, max: 2.00, variance: 0.22).  b) Adapted partitions and results (mean: 0.22, max: 2.00, variance: 0.19). g) Regularly distributed patches. c) Regular partitions and results (mean: 0.22, max: 2.00, variance: 0.18).}
	\label{fig:adaptive_hand}
\end{figure}

\begin{figure}[h!!!!!]
	\begin{center}
		\begin{tabular}{ccccc}	
			\parbox[b]{0.02\linewidth}{(a)\vspace{1.5\baselineskip}}&
		  {\includegraphics[width=.21\linewidth]
		  		{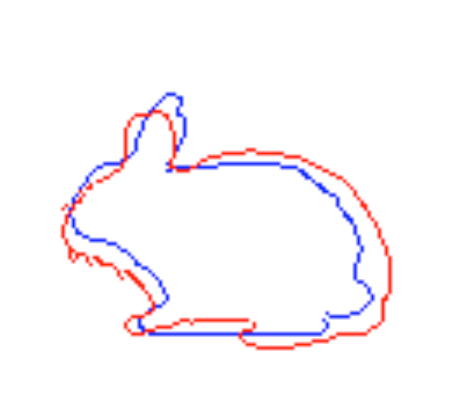}}&
		  {\includegraphics[width=.21\linewidth]
		  		{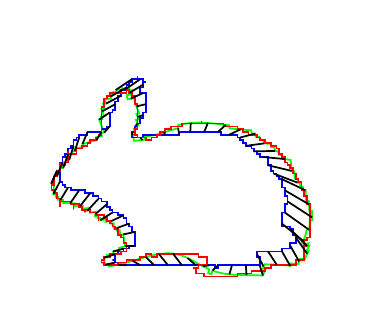}}&
		  {\includegraphics[width=.21\linewidth]
		  		{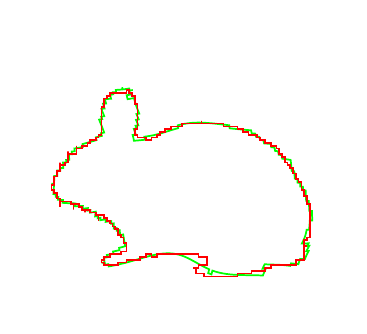}}&
		  {\includegraphics[width=.21\linewidth]
		  		{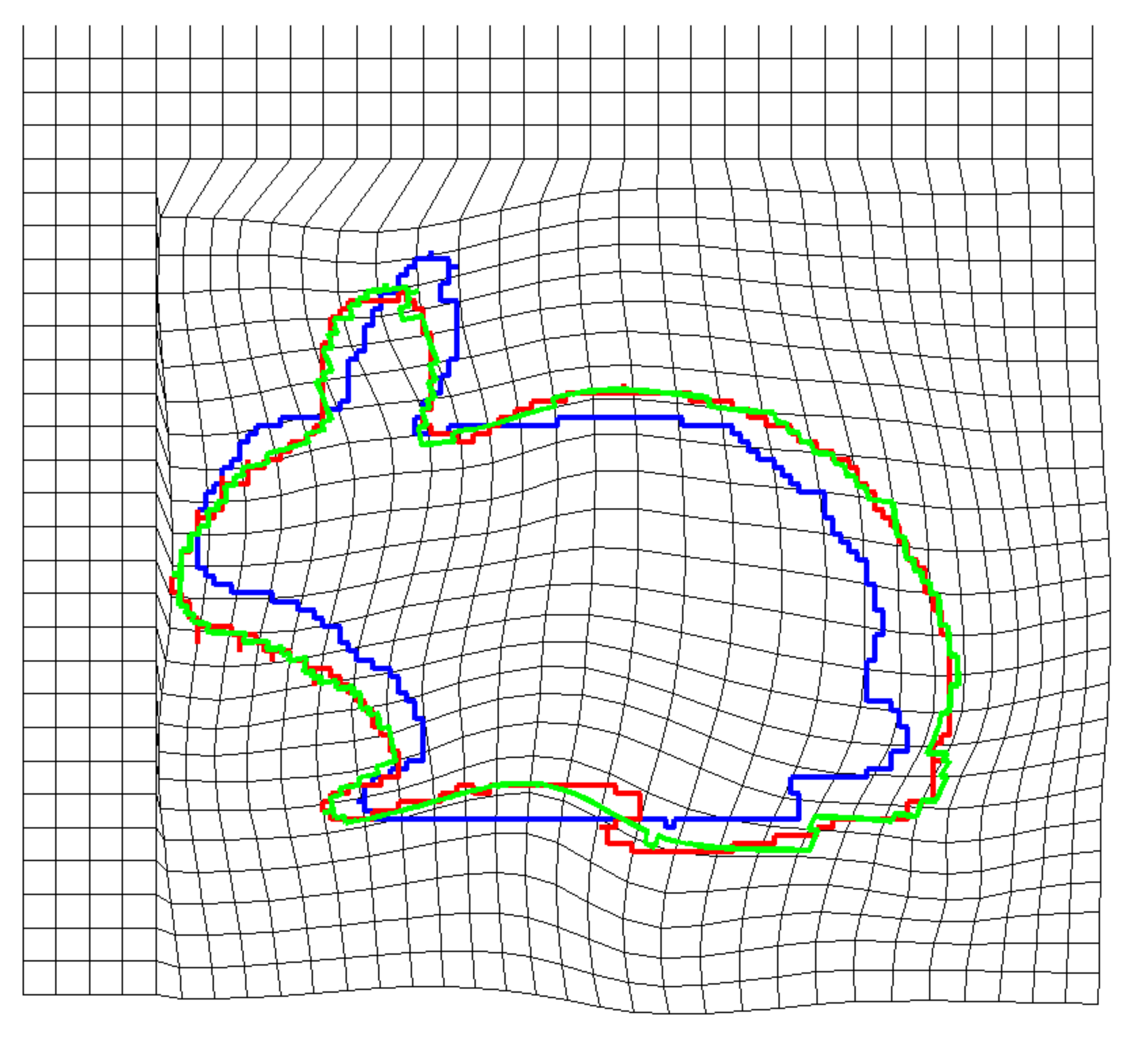}}
		  \\
		  \parbox[b]{0.02\linewidth}{(b)\vspace{1.5\baselineskip}}&
			{\includegraphics[width=.21\linewidth]
		  		{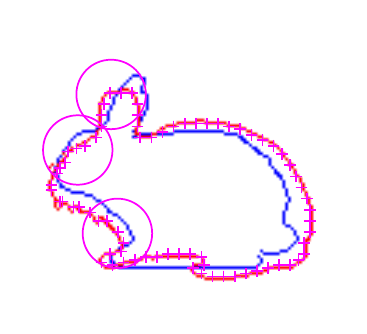}}&
			{\includegraphics[width=.21\linewidth]
		  		{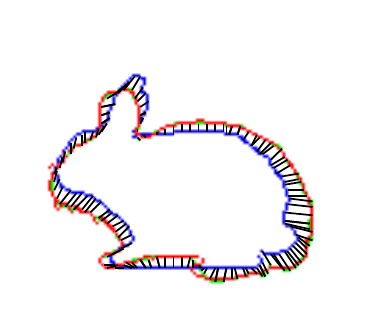}}&
			{\includegraphics[width=.21\linewidth]
		  		{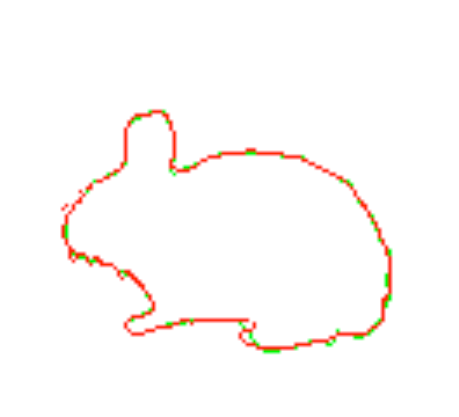}}&
		  {\includegraphics[width=.21\linewidth]
		  		{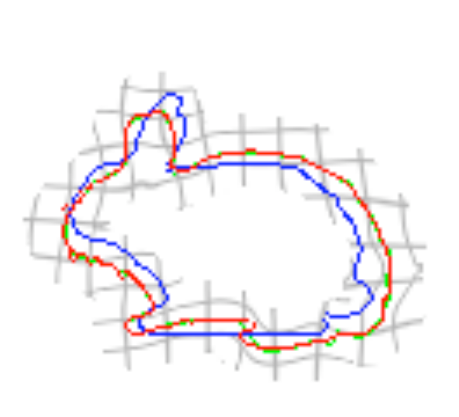}}
		  \\
		  \parbox[b]{0.02\linewidth}{(c)\vspace{1.5\baselineskip}}&
			{\includegraphics[width=.21\linewidth]
		  		{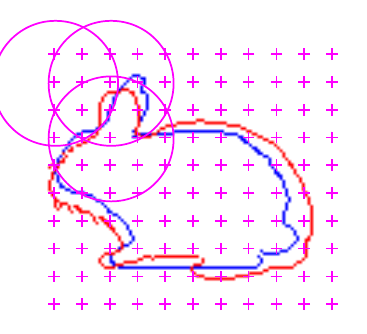}}&
			{\includegraphics[width=.21\linewidth]
		  		{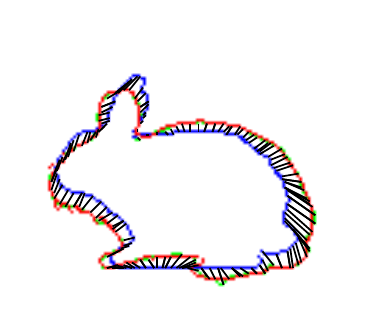}}&
			{\includegraphics[width=.21\linewidth]
		  		{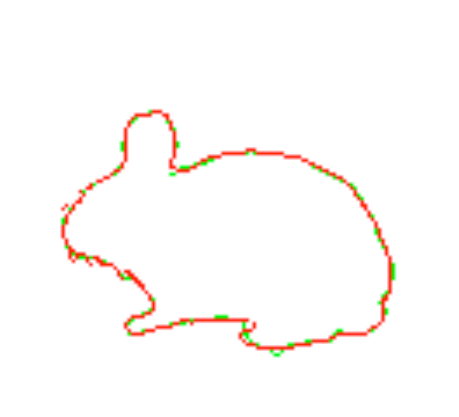}}&
		  {\includegraphics[width=.21\linewidth]
		  		{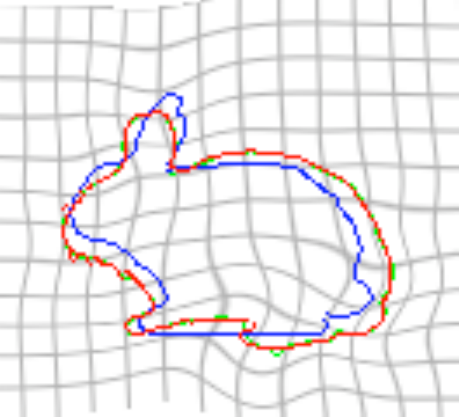}}
		\end{tabular}
	\end{center}
	\caption{Adapted partitions (bunny). a) Results of Huang's method (mean: 0.48, max: 3.16, variance: 0.39).  b) Adapted partitions and results (mean: 0.40, max: 2.83, variance: 0.30). g) Regular partitions and results (mean: 0.40, max: 2.83, variance: 0.29).}
	\label{fig:adaptive_bunny}
\end{figure}

\begin{figure}[h!!!!!]
	\begin{center}
			\begin{tabular}{ccccc}	
				\parbox[b]{0.02\linewidth}{(a)\vspace{1.5\baselineskip}}&
		  	{\includegraphics[width=.21\linewidth]
		  		{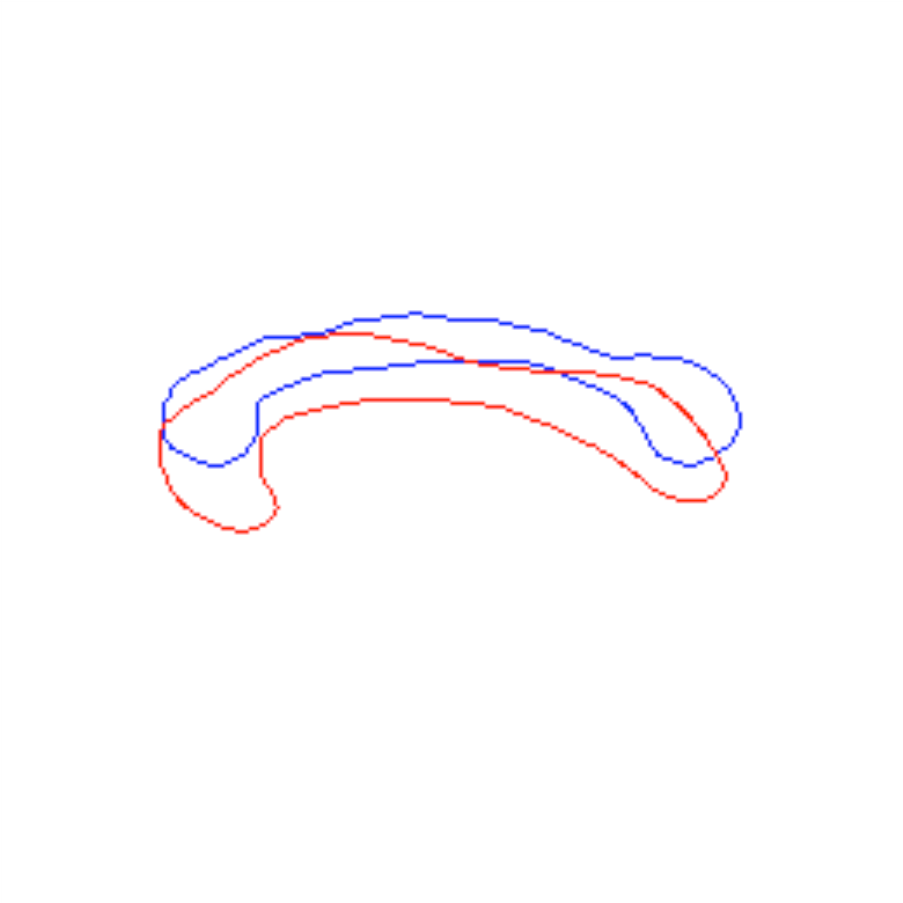}}&
		  	{\includegraphics[width=.21\linewidth]
		  		{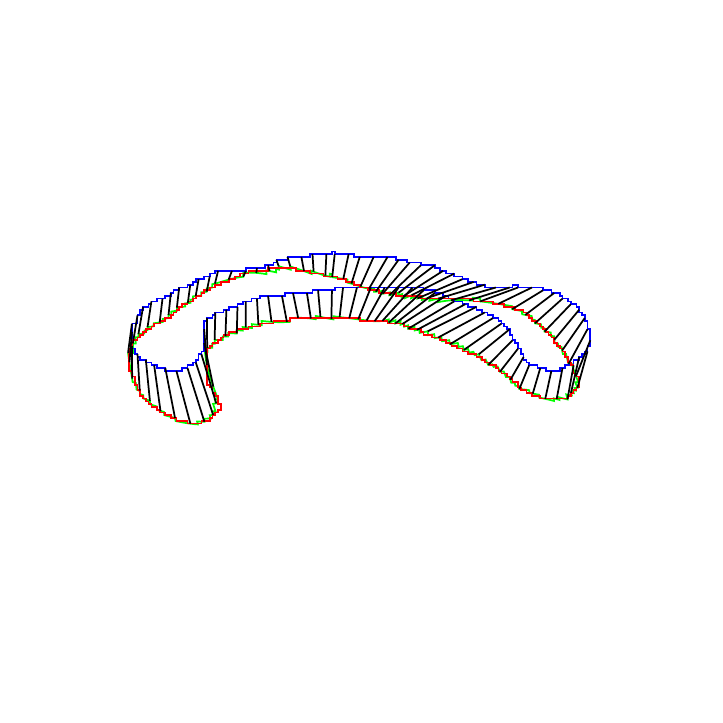}}	&
		  	{\includegraphics[width=.21\linewidth]
		  		{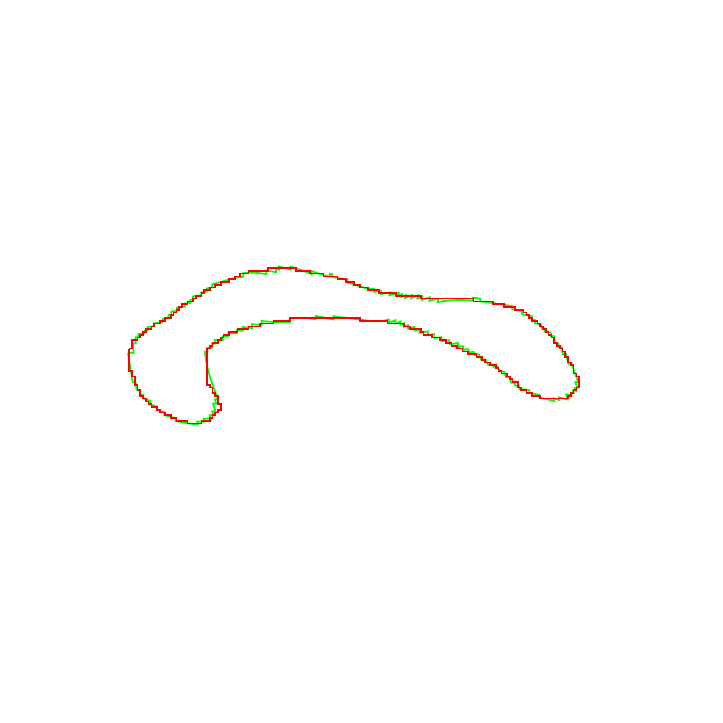}}&
		  	{\includegraphics[width=.21\linewidth]
		  		{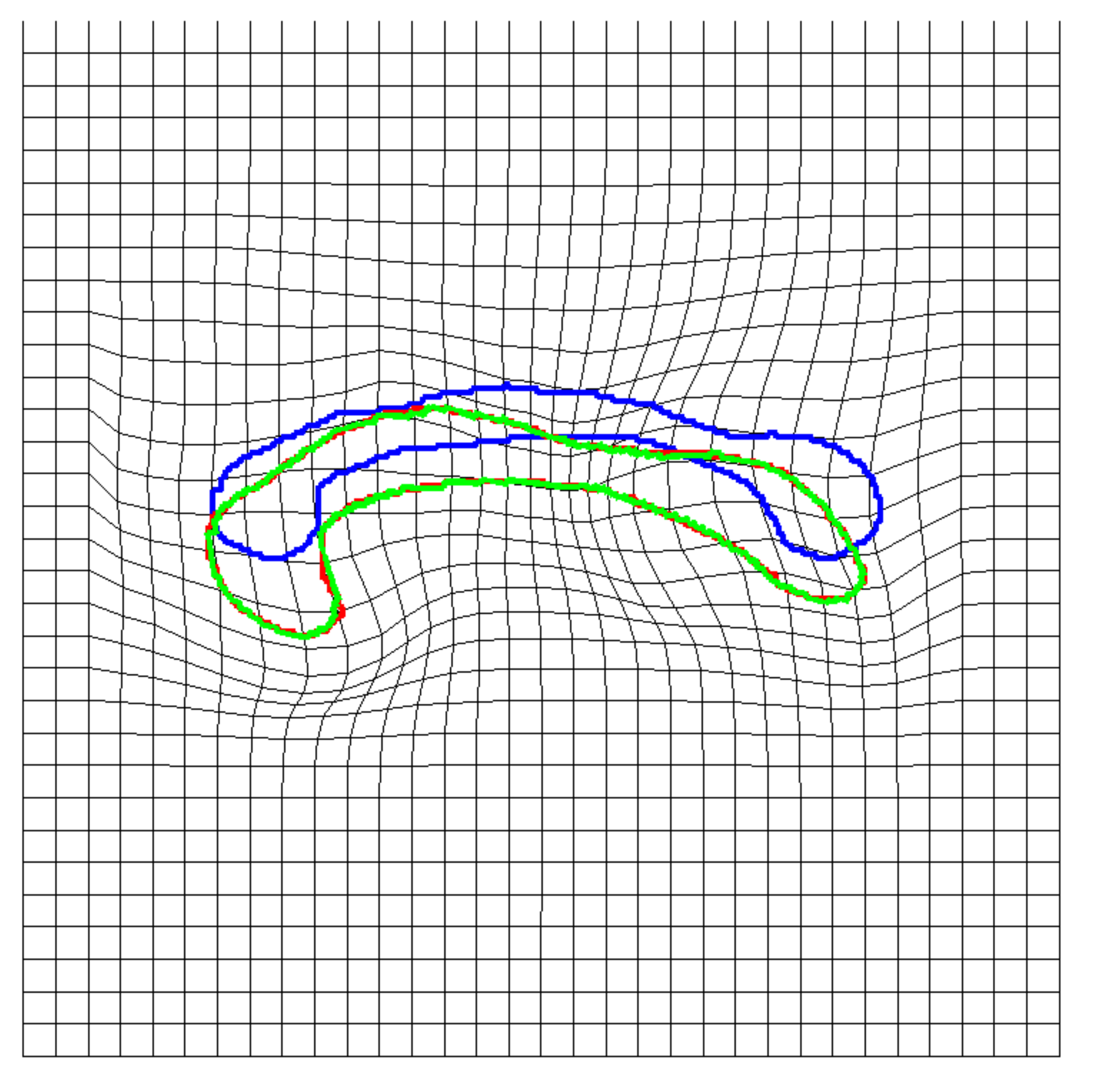}}
		  \\
		  	\parbox[b]{0.02\linewidth}{(b)\vspace{1.5\baselineskip}}&
				{\includegraphics[width=.21\linewidth]
		  		{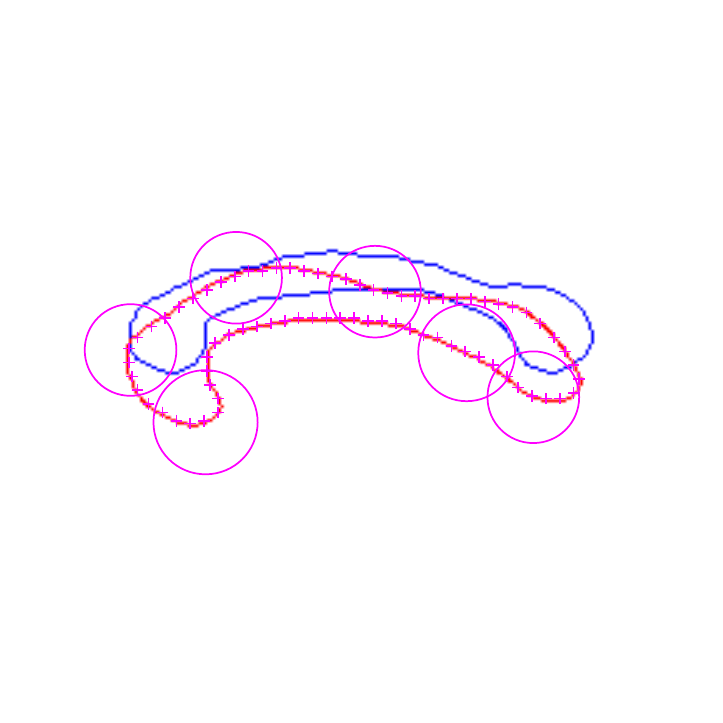}}&
				{\includegraphics[width=.21\linewidth]
		  		{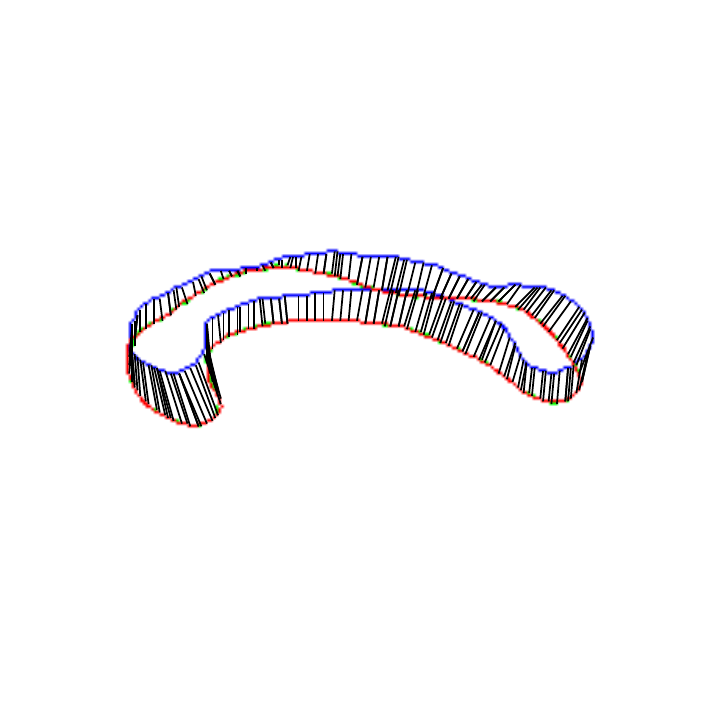}}&
				{\includegraphics[width=.21\linewidth]
		  		{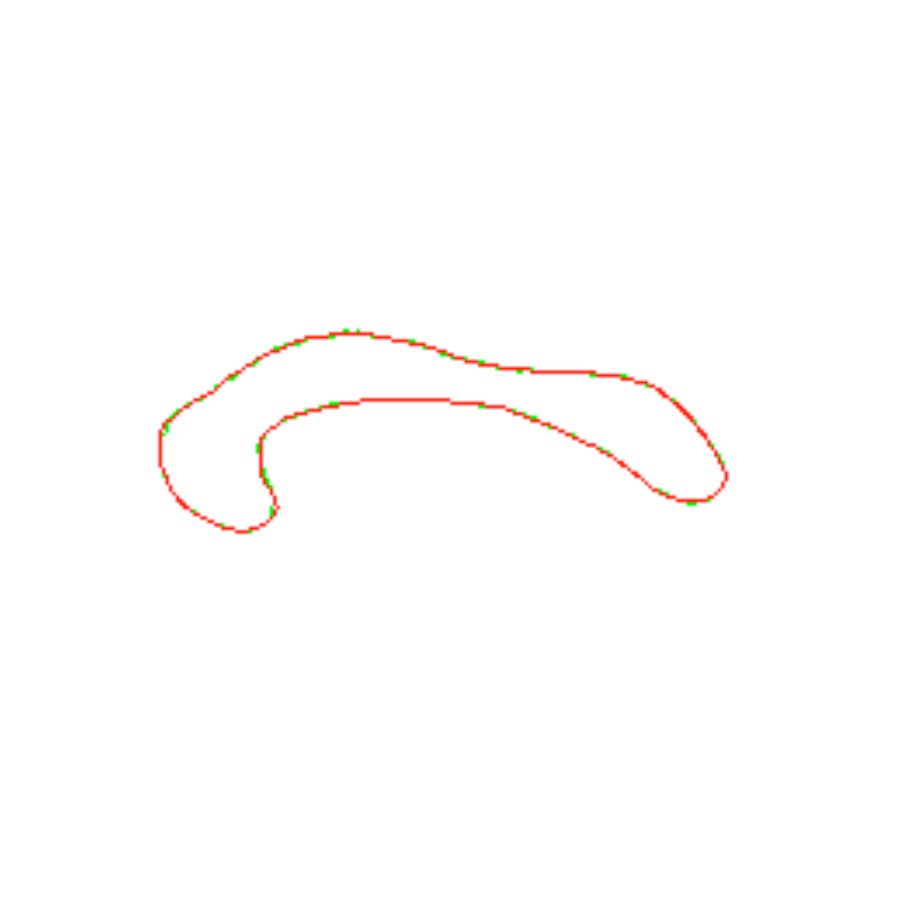}}&
		  	{\includegraphics[width=.21\linewidth]
		  		{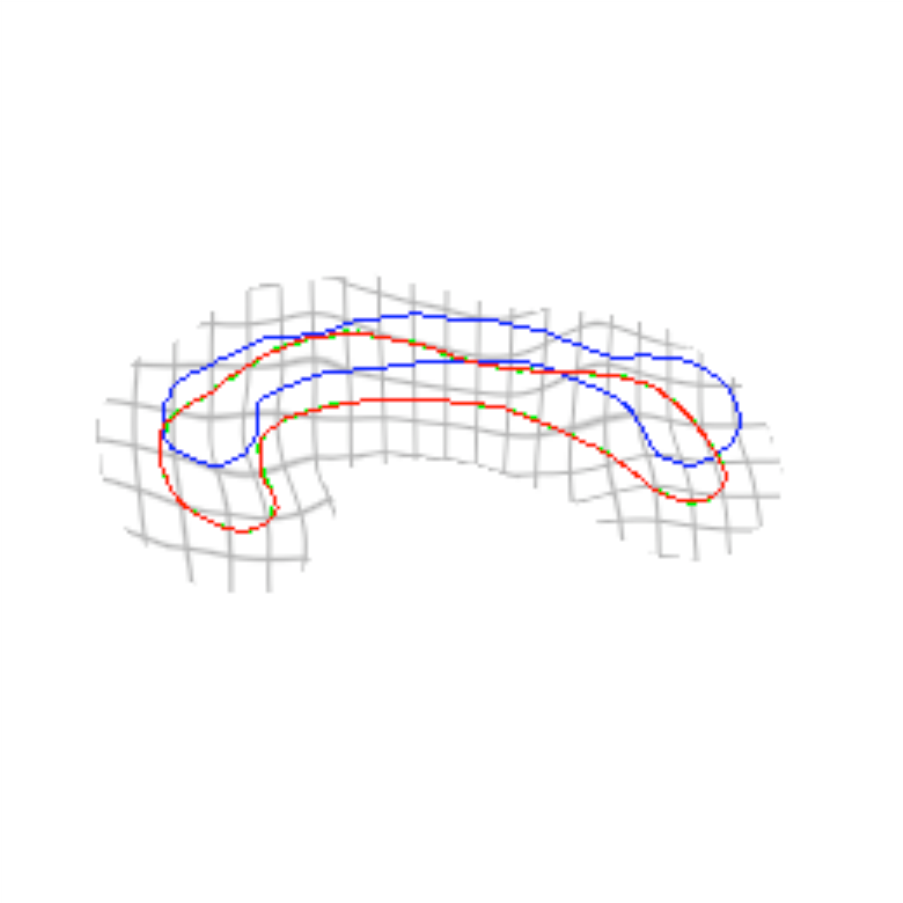}}
		  \\
		  	\parbox[b]{0.02\linewidth}{(c)\vspace{1.5\baselineskip}}&
				{\includegraphics[width=.21\linewidth]
		  		{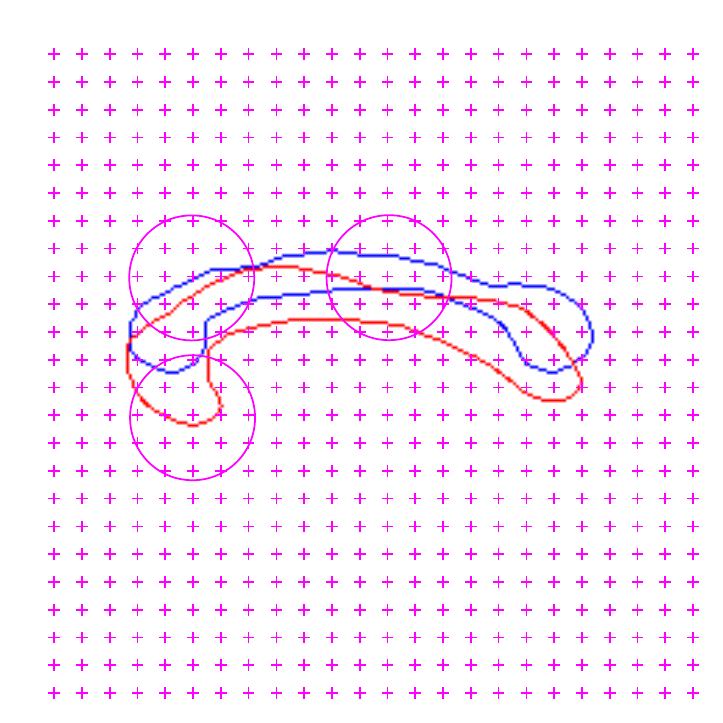}}&
				{\includegraphics[width=.21\linewidth]
		  		{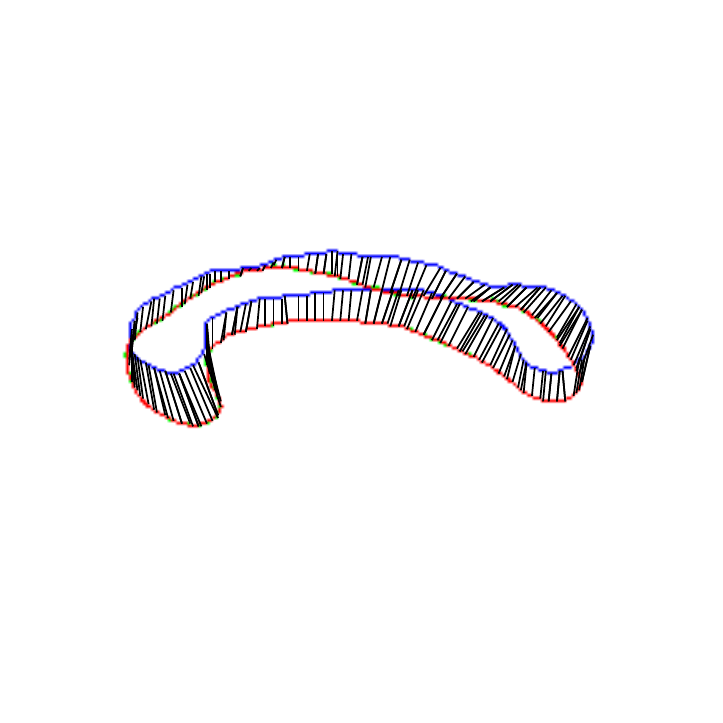}}&
				{\includegraphics[width=.21\linewidth]
		  		{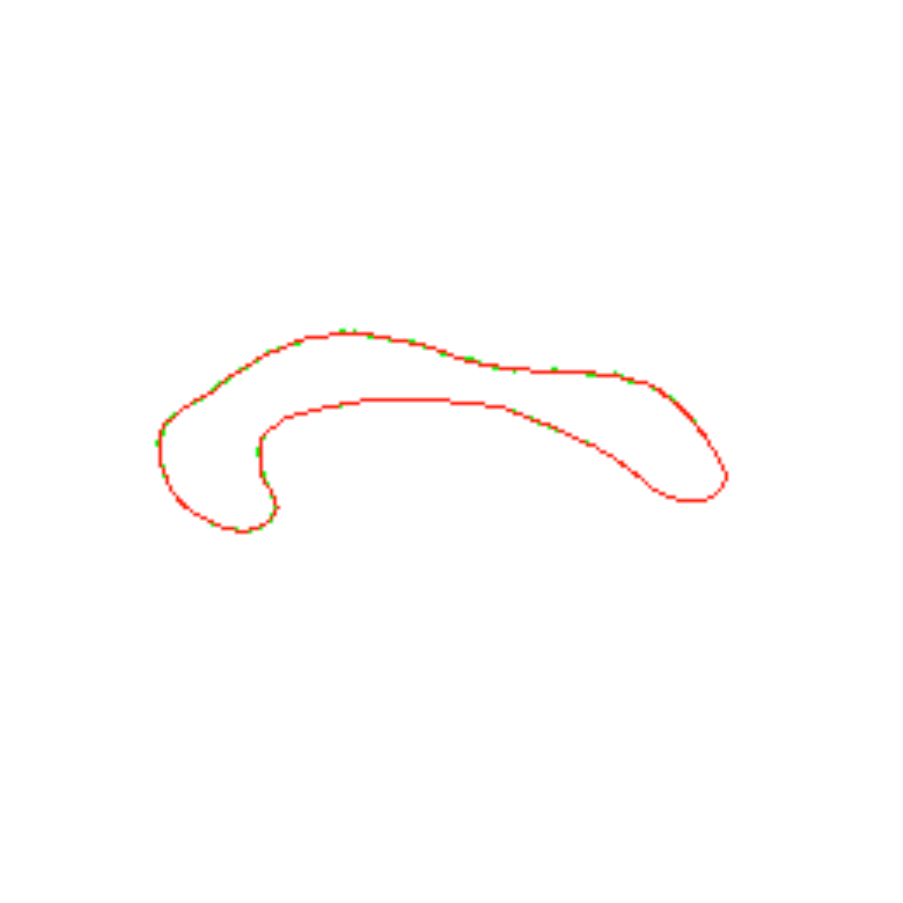}}&
		  	{\includegraphics[width=.21\linewidth]
		  		{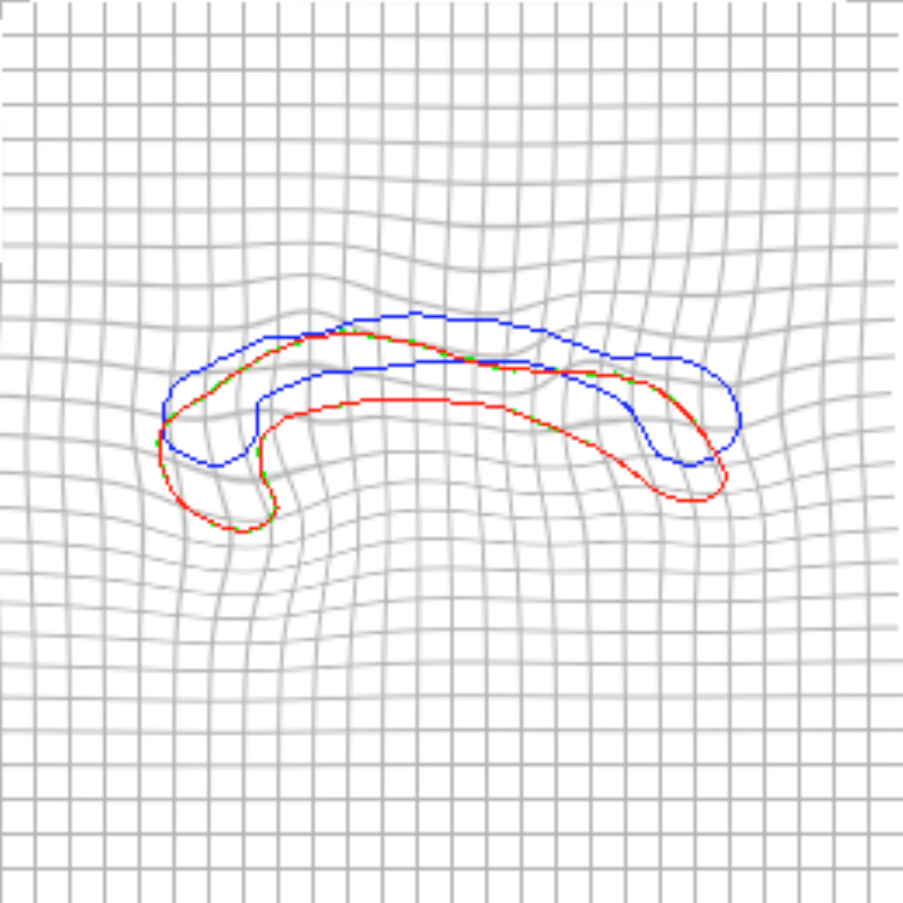}}
		  \end{tabular}
		 \end{center}
	\caption{Adapted partitions (corpus callosum). a) Results of Huang's method (mean: 0.36, max: 2.0, variance: 0.25).  b) Adapted partitions and results (mean: 0.25, max: 1.41, variance: 0.19). g) Regular partitions and results (mean: 0.21, max: 1.00, variance: 0.17).}
	\label{fig:adaptive_callosum}
\end{figure}


\subsection{Cell division}
Finally, we demonstrate an application of our method on analyzing cell deformation. Tracking cell contour is of great interest in biomedical image analysis. Many existing methods are based on active contours and level-set representations. Although these methods produce good results in segmenting cell contours, few of them provide correspondences between deforming contours. Our method can be combined with level-set segmentation algorithms to further establish these correspondences. Figure~\ref{fig:mitosis} shows a sequence of cell mitosis. The top row shows the segmentation results obtained using the level-set segmentation method by Chunming Li~\etal~\cite{Li_TIP08}. The second row shows overlaid shape contours from adjacent images. The third row shows the correspondence established using our method with adapted partition-of-unity. The deformed source contour (green) is well aligned with the target (blue). Interestingly, at the moment of cell mitosis, shape topology changed but our method still established reasonable correspondence between shapes of different topologies, considering that no continuous deformation actually exists to change a shape's topology.

\begin{figure}[h!!!]
	\begin{center}
			{\includegraphics[width=.15\linewidth]
		  		{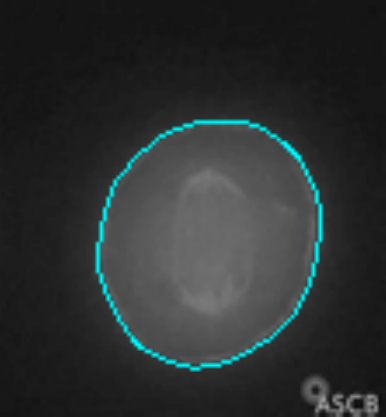}}
		  {\includegraphics[width=.15\linewidth]
		  		{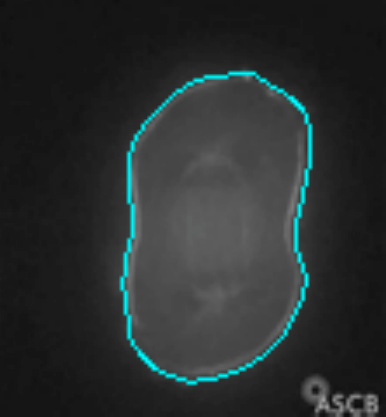}}
		  {\includegraphics[width=.15\linewidth]
		  		{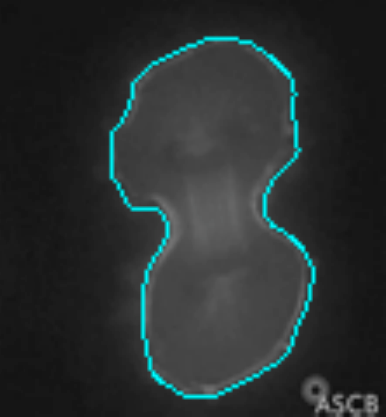}}
		  {\includegraphics[width=.15\linewidth]
		  		{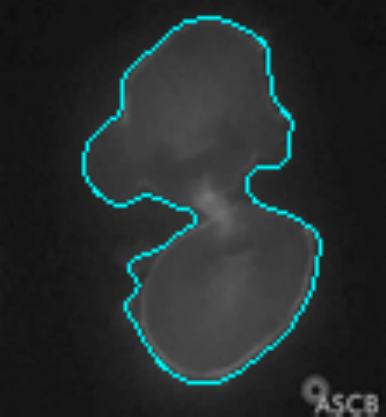}}
		  {\includegraphics[width=.15\linewidth]
		  		{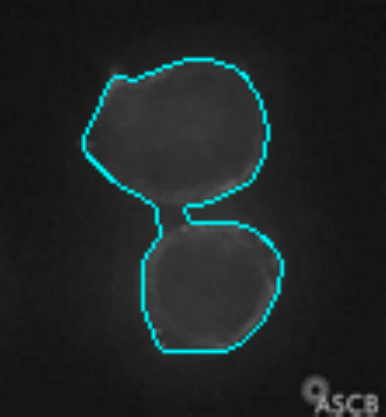}}
		  {\includegraphics[width=.15\linewidth]
		  		{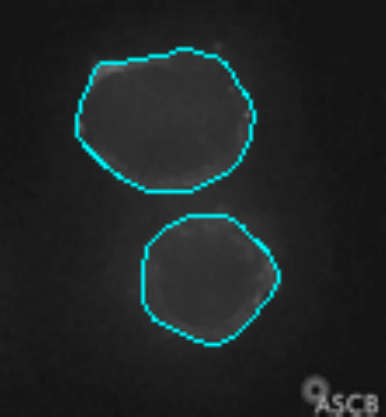}}
		  		\\
		  {\includegraphics[width=.18\linewidth]
		  		{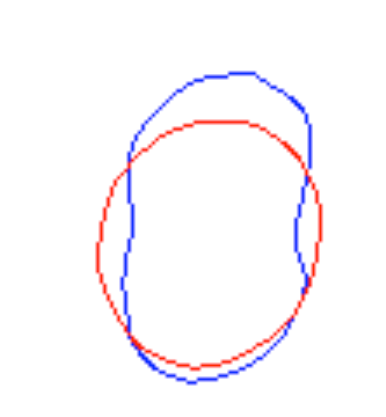}}
		  {\includegraphics[width=.18\linewidth]
		  		{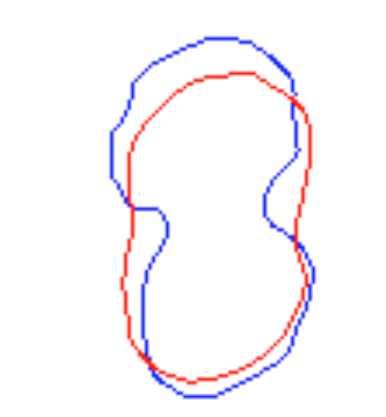}}
		  {\includegraphics[width=.18\linewidth]
		  		{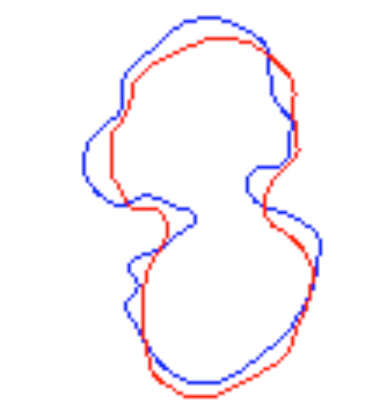}}
		  {\includegraphics[width=.18\linewidth]
		  		{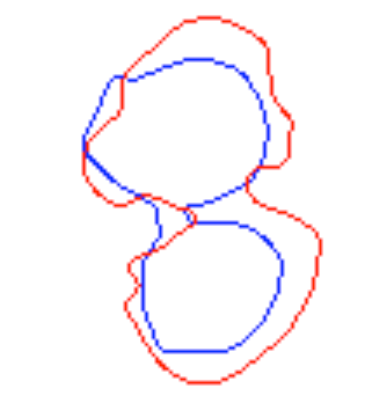}}
		  {\includegraphics[width=.18\linewidth]
		  		{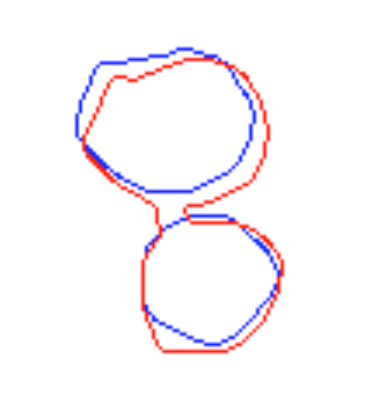}}
		  \\
		  {\includegraphics[width=.18\linewidth]
		  		{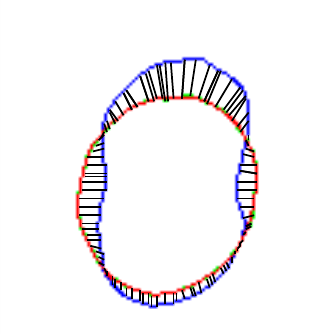}}
		  {\includegraphics[width=.18\linewidth]
		  		{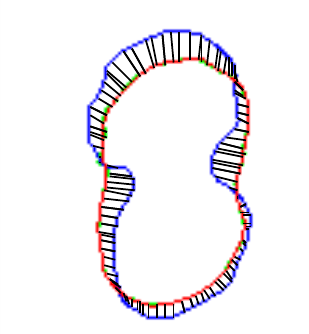}}
		  {\includegraphics[width=.18\linewidth]
		  		{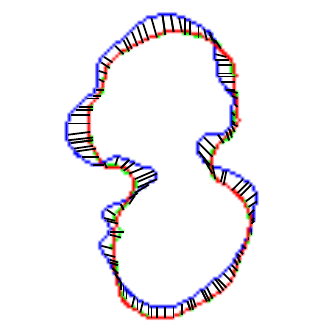}}
		  {\includegraphics[width=.18\linewidth]
		  		{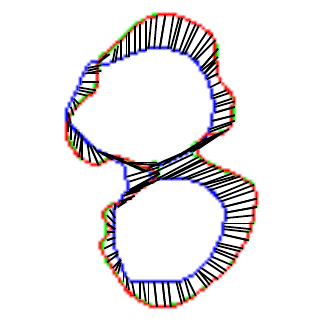}}
		  {\includegraphics[width=.18\linewidth]
		  		{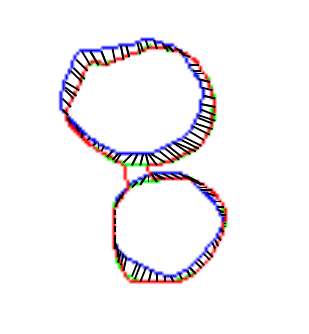}}
		  		\\
		  {\includegraphics[width=.18\linewidth]
		  		{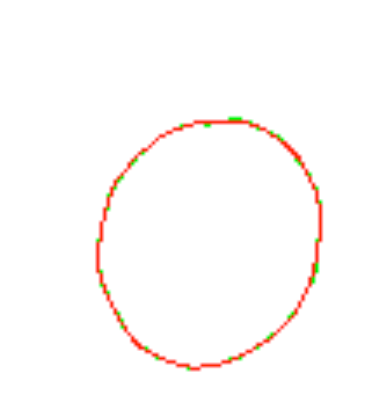}}
		  {\includegraphics[width=.18\linewidth]
		  		{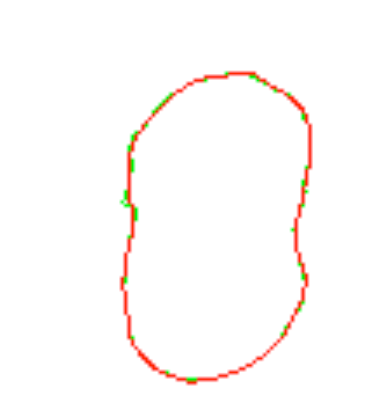}}
		  {\includegraphics[width=.18\linewidth]
		  		{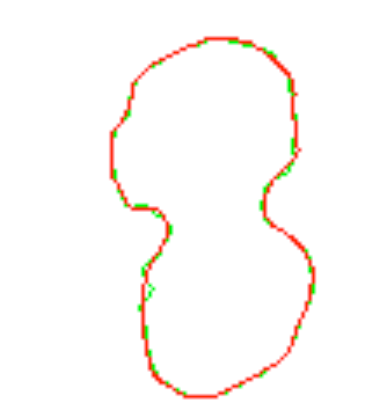}}
		  {\includegraphics[width=.18\linewidth]
		  		{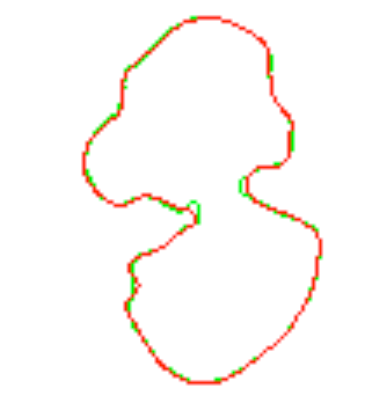}}
		  {\includegraphics[width=.18\linewidth]
		  		{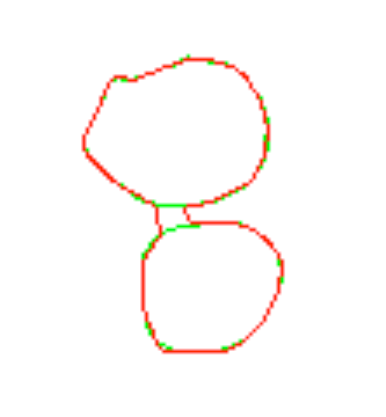}}
	\end{center}
	\caption{Cell-mitosis sequence. Row 1: Cell contours using level-set segmentation in~\cite{Li_TIP08}. Row 2: Overlaid source (blue) and target shapes (red). Row 3: Correspondence established using our method. Row 4: Target shape overlaid over the deformed source shape (green). }
	\label{fig:mitosis}
\end{figure}

\section{Conclusions and future work}
A meshless nonrigid shape-registration algorithm was presented. The registration functional is a variational extension of the classic chamfer-matching energy in which distance transforms provide registration-error gradients, facilitating efficient registration. Also, we modeled shape deformation using a meshless parametric representation. This model does not rely on a regular control-point grid, and can be adapted to arbitrary shapes. Thus, registration can be focused around the shape contours, greatly improving computational efficiency. We tested the proposed method by registering a number of synthetic shapes, and a deforming cell sequence. Future work includes a 3-D extension of the method, and the handling of topological changes.

Despite promising results, our method still encounters problems in registering shapes that have large curvatures, and undergo high-degree deformation, causing local minima in the registration error. We believe that this problem can be addressed by adopting global-optimization algorithms such as simulated annealing~\cite{bonnans2006numerical}, or by the inclusion of statistical priors~\cite{rousson2008prior}.

{\small 
\bibliographystyle{unsrt}
\bibliography{wei}
}

\end{document}